\documentclass[11pt]{article}
\usepackage{latexsym}
\usepackage{amsmath}
\usepackage{amssymb}
\usepackage{amsthm}
\usepackage{natbib}
\setcitestyle{authoryear,open={(},close={)}}
\usepackage{hhline}
\usepackage{epsfig}
\usepackage{xcolor}
\usepackage{graphicx}
\usepackage{bm}
\usepackage{enumitem}
\usepackage{mathtools}
\usepackage{graphicx}
\usepackage{tikz}
\usetikzlibrary{matrix,backgrounds}
\usepackage{ mathrsfs }
\usepackage[toc,page]{appendix}
\usepackage{graphicx}
\usepackage{bbm}
\usepackage{array}
\usepackage{amsmath}
\usepackage{algorithm}
\usepackage{algpseudocode}
\usepackage{multirow} % needed for \multirow
\usepackage{titletoc}
\usepackage{pifont} % For check and cross symbols
\definecolor{darkgreen}{RGB}{0,128,0}
\newcommand{\cmark}{\textcolor{darkgreen}{\ding{51}}} % Green checkmark
\newcommand{\xmark}{\textcolor{red}{\ding{55}}}
\algrenewcommand{\alglinenumber}[1]{\textbf{#1}}
\algrenewcommand\algorithmicrequire{\textbf{Input:}}
\algrenewcommand\algorithmicensure{\textbf{Output:}}
\usepackage{lipsum} 
\usepackage{mdframed}

\newcommand{\removed}[1]{}
\usepackage[top=1in, bottom=1in, left=1in, right=1in]{geometry}
\usepackage{hyperref}
\hypersetup{
    colorlinks,
    linkcolor={blue!80!black},
    citecolor={green!50!black},
}
\usepackage{cleveref}
\usepackage[T1]{fontenc} 
\usepackage{lmodern}

\colorlet{linkequation}{blue}
\definecolor{fg}{rgb}{0.0, 0.5, 0.0}
\usepackage{tcolorbox}
\renewcommand{\P}{\mathbb{P}}
\newcommand{\E}{\mathbb{E}}

\newcommand{\Ndim}{\textnormal{Ndim}}

\usepackage[nice]{nicefrac}
\newcommand{\Z}{\mathbb{Z}}
\newcommand{\R}{\mathbb{R}}

\newcommand{\eps}{\varepsilon}

\newcommand{\<}{\langle}
\renewcommand{\>}{\rangle}
\newcommand{\sign}{\text{sign}}

\newcommand{\circuit}{TC}

\def\bzero{{\boldsymbol 0}}

\def\zo{0\textnormal{-}1}
\def\prefix{\textnormal{pfx}}

\def\move{{\mathsf{move}}}

\def\state{{\mathsf{state}}}
\def\symbol{{\mathsf{symb}}}

\def\pos{{\mathsf{pos}}}
\def\npos{{\mathsf{npos}}}
\def\nread{{\mathsf{read}}}

\DeclareMathOperator*{\argmax}{arg\,max}

\newcommand{\norm}[1]{\lVert #1 \rVert}
\DeclareMathOperator{\ldim}{Ldim}

\DeclareMathOperator{\sfat}{sfat}

\newtheorem{theorem}{Theorem}[section]
\newtheorem{hassumption}{Hardness Assumption}
\Crefname{hassumption}{Assumption}{Assumptions}
\newtheorem*{theorem*}{Theorem}
\newtheorem{lemma}[theorem]{Lemma}

\newtheorem{definition}{Definition}
\newtheorem{proposition}{Proposition}

\newtheorem{claim}[theorem]{Claim}

\newtheorem{corollary}[theorem]{Corollary}

\theoremstyle{definition}

\newtheoremstyle{myremark} % name
    {\topsep}                    % Space above
    {\topsep}                    % Space below
    {\rm}                        % Body font
    {}                           % Indent amount
    {\bf}                        % Theorem head font
    {.}                          % Punctuation after theorem head
    {.5em}                       % Space after theorem head
    {}  % Theorem head spec (can be left empty, meaning normal)

\theoremstyle{myremark}
\newtheorem{remark}{Remark}[section]

\DeclareSymbolFont{rsfs}{U}{rsfs}{m}{n}
\DeclareSymbolFontAlphabet{\mathscrsfs}{rsfs}

\def\bbx{{\mathbf x}}
\def\bby{{\mathbf y}}

\def\bomega{{\boldsymbol \omega}}
\def\bb{{\boldsymbol b}}

\def\bu{{\boldsymbol u}}
\def\bv{{\boldsymbol v}}
\def\bp{{\boldsymbol p}}
\def\bw{{\boldsymbol w}}
\def\bx{{\boldsymbol x}}
\def\by{{\boldsymbol y}}
\def\bz{{\boldsymbol z}}
\def\VCdim{\mathrm{VCdim}}
\def\Pdim{\mathrm{Pdim}}

\def\btheta{{\boldsymbol \theta}}

\def\bTheta{{\boldsymbol \Theta}}

\def\post{{\rm post}}
\def\pre{{\rm pre}}

\def\lin{{\rm lin}}

\def\lin{{\rm lin}}

\def\post{{\rm post}}
\def\pre{{\rm pre}}

\def\cL{{\mathcal L}}
\def\cF{{\mathcal F}}

\def\cI{{\mathcal I}}

\def\cH{{\mathcal H}}
\def\cA{{\mathcal A}}

\newcommand{\TMf}{g_{\langle \S,\Time,\tau\rangle}}

\def\S{\mathtt{S}}
\def\M{\mathtt{M}}
\def\Time{\mathtt{T}}
\def\TIME{\mathsf{TIME}}

\def\naturals{{\mathbb N}}

\def\poly{{\mathcal{P}\hspace{-1pt}oly}}
\def\Poly{\poly}

\def\naturals{{\mathbb N}}

\def\th{\tilde{h}}

\def\cD{{\mathcal D}}
\def\cX{{\mathcal X}}
\def\cF{{\mathcal F}}

\def\cI{{\mathcal I}}

\def\lin{{\rm lin}}

\def\cX{{\mathcal X}}

\def\btheta{{\boldsymbol \theta}}
\def\bTheta{{\boldsymbol \Theta}}

\def\cM{{\mathcal M}}

\def\ind{\mathbbm{1}}

\def\cY{\mathcal{Y}}
\def\cZ{\mathcal{Z}}

\def\val{{\rm val}}

\def\bs{{\boldsymbol s}}

\def\vbb{{{\Vec{\bb}}}}
\def\th{\textnormal{thr}}

\newcommand{\zhiyuan}[1]{{\color{red}[ZL:#1]}}

\def\of{\overline{f}}
\def\sCoT{{\sf CoT}}
\def\sete{{\sf e2e}}
\def\prog{{\sf PROG}}
\def\TM{{\sf TM}}
\def\TMAR{\cF_{\scriptscriptstyle{\mathsf{TM},\S}}}

\newcommand{\ete}[2]{#1^{\sete\textnormal{-}{#2}}}
\newcommand{\mete}[1]{m^{\sete\textnormal{-}{#1}}}
\newcommand{\CoT}[2]{#1^{\sCoT\textnormal{-}{#2}}}
\newcommand{\NTICoT}[2]{#1^{\textnormal{TD-}\sCoT\textnormal{-}{#2}}}
\newcommand{\NTIete}[2]{#1^{\textnormal{TD-}\sete\textnormal{-}{#2}}}
\newcommand{\mcot}[1]{m^{\sCoT\textnormal{-}{#1}}}
\newcommand{\Cons}{\textsc{Cons}}

\newcommand{\cotCons}{\Cons_{\CoT{\cF}{T}}}
\newcommand{\Conscot}{\Cons_{\sCoT}}
\newcommand{\Consete}{\Cons_{\sete}}
\newcommand{\nil}{     
     \vcenter{\hbox{\scalebox{0.45}{$\;\mathbin{ \square }\;$}}}    
}

\def\zo{\textnormal{0-1}}
\def\iid{\sim_{iid}}

\def\cFlin{\mathcal{F}_{\scriptscriptstyle d,\lin}}
\def\cFlinsp{\mathcal{F}_{\scriptscriptstyle d,k,\lin}}

\newcommand{\abs}[1]{\left\lvert{#1}\right\rvert}
\newcommand{\runtime}{\mathsf{time}}

\def\preprocess{{\mathsf{Pre}}}
\def\postprocess{{\mathsf{Post}}}
\newcommand{\isfirst}{{\mathsf{is}\textsf{-}\mathsf{first}}}

\newcommand{\invpos}{{\mathsf{idx}\textsf{-}\mathsf{inv}}}
\newcommand{\aha}{\mathsf{AHA}}

\newcommand{\zeros}{{\mathbf{0}}}
\newcommand{\indct}{{\mathbf{1}}}
\newcommand{\sfq}{{\mathsf{q}}}
\newcommand{\sfk}{{\mathsf{k}}}
\newcommand{\sfv}{{\mathsf{v}}}
\newcommand{\readtape}{\textnormal{read-tape}}
\newcommand{\scaledpos}{\mathsf{scaled}\textsf{-}\mathsf{npos}}
\newcommand{\att}{\mathsf{Att}}
\newcommand{\score}{\mathsf{score}}
\def\ERM{{\textnormal{ERM}}}

\usepackage{authblk}

\title{A Theory of Learning with Autoregressive Chain of Thought}
\author[1]{Nirmit Joshi\footnote{Corresponding authors' emails: \texttt{\{nirmit,nsrebro\}@ttic.edu}}}
\author[2]{ \quad Gal Vardi}
 \author[3]{\quad Adam Block}
\author[4]{\quad Surbhi Goel}
\author[1]{\\ Zhiyuan Li}
\author[5]{\quad Theodor Misiakiewicz}
\author[1]{\quad Nathan Srebro$^*$}

\affil[1]{\small 
Toyota Technological Institute at Chicago}
\affil[2]{\small 
Weizmann Institute of Science}
\affil[3]{\small Microsoft Research, NYC}
\affil[4]{\small University of Pennsylvania}
\affil[5]{\small Yale University}

\date{}
\newcounter{SubAlgLine}

\begin{document}\maketitle
\vspace{-1cm}
\begin{abstract}
  For a given base class of sequence-to-next-token generators, we consider learning prompt-to-answer mappings obtained by iterating a fixed, time-invariant generator for multiple steps, thus generating a chain-of-thought, and then taking the final token as the answer. We formalize the learning problems both when the chain-of-thought is observed and when training only on prompt-answer pairs, with the chain-of-thought latent. We analyze the sample and computational complexity both in terms of general properties of the base class (e.g. its VC dimension) and for specific base classes such as linear thresholds. We present a simple base class that allows for universal representability and computationally tractable chain-of-thought learning. Central to our development is that time invariance allows for sample complexity that is independent of the length of the chain-of-thought. Attention arises naturally in our construction.
    \end{abstract}
%%%%%%%%%%%% Introduction %%%%%%%%%%%%%%
\section{Introduction}

Autoregressive generation and learning, particularly with attention-based models such as transformers, is driving remarkable advances in Artificial Intelligence and increasingly becoming synonymous with AI itself.  To solve complex tasks, especially those requiring multi-step or compositional reasoning and computation, autoregressive generation produces a Chain-of-Thought, consisting of multiple intermediate tokens, that ultimately leads to the desired answer. In this paper, we propose a formal framework for studying this emerging paradigm:  learning complex functions through autoregressive Chain-of-Thought generation using a simple next-token generator.  We analyze the statistical and computational benefits and pitfalls of this approach, and see how attention naturally arises as a key ingredient for ``universal'' learning with autoregressive Chain-of-Thought generation.

In our view, a central component of autoregressive generation is that it is \textbf{time-invariant}, i.e., the same next-token-generator, with the identical parameters (e.g.~a transformer with the same weights) is used at each step of generation, to generate each token as a function of the prefix thus far. Throughout the paper we emphasize how such time-invariance is crucial for allowing learning with sample complexity independent (perhaps up to a log-factor) of the generation length, and later on, of the compute time (number of steps) of a learned process. In this crucial regard, we deviate significantly from another recent (and inspiring) attempt to formalize autoregressive learning by \citet{malach2023auto}, who considered time-\emph{dependent} autoregressive learning. We directly contrast with \citeauthor{malach2023auto}'s approach as we progress through our presentation.
%\removed{
% Our framework is not meant to be a comprehensive framework for studying learning via next-token prediction (see Related Work in \Cref{app:related}), but rather focuses specifically on learning a complex target function (from a ``prompt'' to an ``answer'') through time-invariant autoregressive Chain of Thought learning.}\removed{Researchers are still struggling with formalizing what is it that large language models are actually learning (a density model? language generation \cite[e.g.][]{kleinberg}? a representation for other tasks? }
% \natinote{I am fine with removing this paragraph} \abcomment{I think we do need to lay out this paper, because otherwise people might get a bit confused as to the flow and how the results relate to each other}

We begin by presenting a formal setting for time-invariant autoregressive Chain-of-Thought generation and learning using an abstract base class of next-token-generators (\Cref{sec:setting}), studying the basics of learning in this framework for general and abstract base classes (\Cref{sec:samples-computational-complexity-general-F}), and, as an example, for what is perhaps the simplest possible base class, namely linear thresholds (\Cref{sec:linear}). We consider training both with and without explicit Chain-of-Thought supervision. Already here we can start seeing the computational benefits of Chain-of-Thought learning and the statistical benefits of time-invariant autoregressive generation.  We then turn to the more ambitious goal of ``universal'' learning of computable functions and see how far the benefits of time-invariance and autoregressive Chain-of-Thought take us, and how attention arises naturally (\Cref{sec:universal,sec:TM}). 

\section{Time-Invariant Autoregressive Chain-of-Thought Learning}\label{sec:setting}
For a finite set of tokens $\Sigma$, a {\em next token generator} is a mapping $f : \Sigma^* \to \Sigma$. On input $\bx\in\Sigma^*$, autoregressive Chain-of-Thought (CoT) generation will start with a string containing only the input and will repeatedly apply $f$ to the string and append the newly generated token to it (see \Cref{fig:tape}). Formally, for a next-token generator $f$, we define the apply-and-append mapping $\of(\bx): \bx \mapsto \textnormal{append}(\bx, f(\bx)) \in (\Sigma^{\abs{\bx}+1}$),\footnote{For any $\bx \in \Sigma^*$, we use $|\bx|$ to denote its length. We use negative indices to refer to elements of a vector or string from the end, i.e.~$\bv[-1]$ is the last element.  We use inclusive slice indexing, so that $\bv[i:j]$ are the elements from position $i$ to $j$ inclusive.  Dropping $i$ or $j$ means that the slice extends from the beginning or to the end, respectively.}  and apply $\of$ iteratively for $T$ steps to obtain a mapping from $\bx$ to its $T$-step Chain-of-Thought:
\removed{\begin{equation*}
    \CoT{f}{T} (\bx) = \underbrace{\of \circ \of \circ \ldots \circ \of}_{T \text{ times}} (\bx)\in \Sigma^{\abs{\bx}+T}.
\end{equation*}}
\begin{equation}\label{eq:fCoT}
    \CoT{f}{T} (\bx) = \underbrace{\of \circ \of \circ \ldots \circ \of}_{\textrm{\small $T$ times}} \, (\bx)\in \Sigma^{\abs{\bx}+T}.
\end{equation}
We think of all but the last of these $T$ tokens as intermediate ``thinking'', and only the final token after $T$ steps as the ``answer'' and so consider the following end-to-end mapping between an input and the final token in $\CoT{f}{T} (\bx)$:
\begin{equation}
\ete{f}{T}(\bx) = \CoT{f}{T} (\bx)[-1]\,.
\end{equation}
% \removed{. Pictorially, we are going to imagine the input string $\bx \in \Sigma^*$ written on the first $n=|\bx|$ cells of \removed{an infinitely long }the output tape\natinote{Can we reserve 'tape' for the TM and use some other word here?  Not sure I have an idea}.}\removed{ on its right end, whose cells are indexed by $\naturals_{+}$. In each subsequent iterative autoregressive composition, we generate a new token from the current sequence by applying $f$ to string already written on the tape, and then append the new token to the tape. }
% \begin{equation}
%     \of(\bx)[\,i\,] = \begin{cases}
% f(\bx) & , \text{if } i=n+1\\
% \bx[\,i\,] &, \text{if } 1 \leq i \leq  n.
% \end{cases}
% \qquad \text{In view of Figure \Cref{fig:tape}, } \of (\bx) = [\bx,f(\bx)].
% \end{equation}
 %The string with the appended new token will serve as the input for the next token generation. Throughout this work, we will use the words iterative and autoregressive interchangeably. %\underbrace{\of \circ \dots \circ \of}_{t \text{times} }(\bx)$ on the cell with index $|\bx|+t$.
\begin{figure}[t]
    \centering
    \includegraphics[scale=0.8]{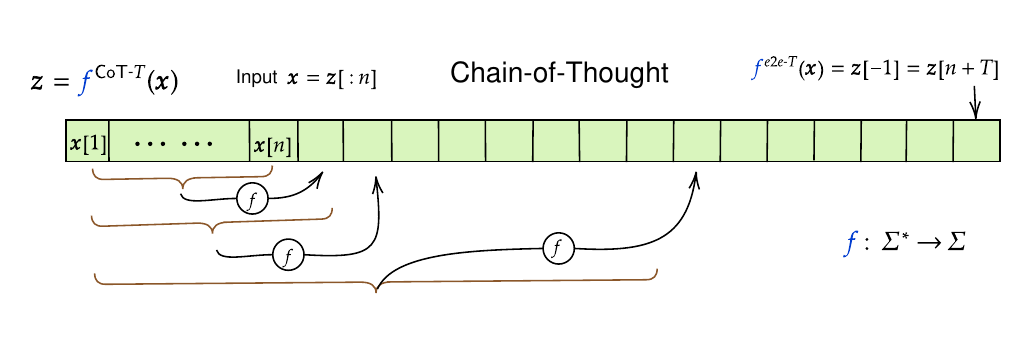}
    \vspace{-1cm}
    \small \caption{\small Illustration of Chain-of-Thought generation and the functions $\CoT{f}{T}$ and $\ete{f}{T}$.}
    \vspace{-0.2cm}
    \label{fig:tape}
\end{figure}
For any base class $\cF = \{ f:\Sigma^*\rightarrow\Sigma \} \subseteq \Sigma^{\Sigma^*}$ of generators, and a generation length $T$, we define the corresponding end-to-end and CoT classes:
\begin{equation}
\ete{\cF}{T}=\{\ete{f}{T}:\Sigma^* \to \Sigma \mid f \in \cF\} \;\;\text{ and }\;\;  \CoT{\cF}{T}=\{\CoT{f}{T}: \Sigma^* \to \Sigma^* \mid f \in \cF \}~.
\end{equation}
For simplicity, throughout the paper we consider a fixed thought generation length (i.e.~number of iterations) $T$.  This number of iterations $T$ will be a crucial parameter for us and we will study how the learning complexity and expressivity depend on $T$.\footnote{ 
 One could also define a variant with variable thought lengths with a special token indicating when to stop.  But if we would anyway be interested in the maximal thought length, we could just as well extend each thought with a padding token to make them all of the same length.} 

\paragraph{Learning.} Our goal is to learn the hypothesis class $\ete{\cF}{T}$.  Throughout we consider a realizable setting, i.e.~we assume there exists some ground truth $f_* \in \cF$, and for an input distribution $\cD$ over $\cX \subseteq \Sigma^*$ (if not explicitly specified, we always take the domain as $\cX=\Sigma^*$), we would like to find a predictor $h:\cX\rightarrow\Sigma$ that minimizes the population error:
\begin{equation*}
 L_{\cD,f_*}^\zo(h) : \removed{\E_{\bx \sim \cD} \left[ \ind\{h(\bx) \neq \ete{f}{T}_*(\bx)\}\right]} = \P_{\bx \sim \cD}\left(h(\bx) \neq \ete{f}{T}_*(\bx) \right).
\end{equation*}
We will consider two learning settings that differ by what data are available during training.  First, we will consider learning based only on observing the final answer $\ete{f}{T}_*(\bx)$:
\begin{definition}[Realizable End-to-End Learnability]\label{def:learnable-wo-CoT} For a base class $\cF \subseteq \Sigma^{\Sigma^*}$, and generation length $T\in \naturals_{+}$, we say that $\ete{\cF}{T}$ is \underline{$\sete$-learnable} over a domain $\cX\subseteq\Sigma^*$, with sample complexity $m(\eps,\delta)$ if there exists a learning rule $A: (\cX \times \Sigma)^* \to \Sigma^{\cX}$ such that for every 
%input 
distribution $\cD$ over $\cX$ and $f_* \in \cF$, for every $0<\eps, \delta<1$, with probability at least $1-\delta$ over $\bx_1, \dots, \bx_m \iid \cD$, with $S_{\sete}=((\bx_1,y_1),\dots,(\bx_m,y_m))$, $y_i=\ete{f_*}{T}(\bx_i)$
%\gal{should be $\ete{f_*}{T}$} 
and  $m=m(\eps,\delta)$, we have
$L_{\cD,f_*}^{\zo}(A(S_{\sete})) \leq \eps$.
\end{definition}
\Cref{def:learnable-wo-CoT} amounts precisely to the standard definition of PAC-learning the hypothesis class $\ete{\cF}{T}$, but we give it explicitly in order to contrast it with learning in situations where the entire Chain-of-Thought is available during training. In contemporary LLM training, such CoT supervision can be found in data sources such scrapped over internet consisting of step-by-step answers to questions in discussion forums and in textbooks \citep{lewkowycz2022solving,li2023textbooks,gao2020pile,achiam2023gpt}, as well as in training data specifically curated \citep{https://doi.org/10.48550/arxiv.2301.11596,chung2024scaling,hendrycksmath2021}.

\begin{definition}[Realizable Chain-of-Thought Learnability]\label{def:learnable-with-CoT} For a base class $\cF \subseteq \Sigma^{\Sigma^*}$ and generation length $T \in \naturals_{+}$, we say that\footnote{We emphasize that $\sCoT$-learnability is a property of the base class $\cF$ and length $T$, not only of the set $\ete{\cF}{T}$.} $\ete{\cF}{T}$ is \underline{$\sCoT$-learnable} over a domain $\cX\subseteq\Sigma^*$, with sample complexity $m(\eps,\delta)$, if there exists a learning rule $A: (\cX \times \Sigma^T)^* \to \Sigma^{\cX}$ such that for every 
%input 
distribution $\cD$ over $\cX$ and $f_* \in \cF$, for every $0<\eps, \delta<1$, with probability at least $1-\delta$ over $\bx_1, \dots, \bx_m \iid \cD$, with $S_{\sCoT}=(\bz_1,\dots,\bz_m)$ where $\bz_i=\CoT{f_*}{T}(\bx_i)$,
%\gal{should be $\CoT{f_*}{T}$}
and $m=m(\eps,\delta)$, we have
$$L_{\cD,f_*}^{\zo}(A(S_{\sCoT})) = \P_{\bx \sim \cD}\left(A(S_{\sCoT})(\bx) \neq \ete{f}{T}_*(\bx) \right) \leq \eps.$$
\end{definition}
Even though the entire Chain-of-Thought is available during training, the learned predictor is still evaluated only on predicting the final answer. We do not explicitly require learning to be {\em proper}, but all our positive results provide a proper learning rule outputting $A(S)\in\ete{\cF}{T}$, and all our negative results preclude also improper learning, i.e.~outputting some $A(S):\cX\rightarrow\Sigma$, even if $A(S)\not\in\ete{\cF}{T}$. We do not observe any gaps between proper and improper learning.

\paragraph{Time-Invariance vs. Time-Dependence.} In our setup we emphasize {\em time invariance}, i.e.~that the same base function $f\in\cF$ is used in every step of autoregressive generation (e.g., a transformer with the same parameters, in a practical setting of interest).  We then need to learn just a single function $f\in\cF$, and our goal is for the learning complexity to be independent of (or at least nearly independent of) $T$.
Time-invariance here can also be viewed as  ``parameter sharing'' \citep{rajaraman2021value} across steps of generation: for a parametric class $\cF=\left\{f_\btheta : \btheta\in\bTheta\right\}$, the parameter $\btheta$ is shared among all time steps and we only need to learn a single $\btheta$.  
This invariance is in contradistinction to a {\em Time-Dependent} notion of autoregressive learning suggested by \citet{malach2023auto}, where a different base predictor $f_t \in \cF$ is used in every step $t$, yielding the Time-Dependent Chain-of-Thought class
\begin{equation}\label{eq:NTI-cot}
%\begin{gathered}
    \NTICoT{\cF}{T}= \left\{\, x\mapsto\of_T \circ \of_{T-1} \circ \dots \circ \of_2 \circ \of_1(\bx) \,\middle|\, f_1,f_2,\ldots,f_T \in \cF \, \right\}~.\removed{ \\
 \NTIete{\cF}{T}=\left\{\, \bx\mapsto g(\bx)[-1] \,\middle|\, g \in \NTICoT{\cF}{T}\,\right\}.}
%\end{gathered}
\end{equation}
\citeauthor{malach2023auto} considered learning the class $\NTICoT{\cF}{T}$, i.e.~with the CoT available during training (as in our \Cref{def:learnable-with-CoT})\footnote{In \citet{malach2023auto}, the learning goal is to be able to generate the entire CoT, not just the final answer.  This is different from our \Cref{def:learnable-with-CoT}, but this is not a substantial difference.}.  One can also consider end-to-end learning (as in our \Cref{def:learnable-wo-CoT}) in such a time-dependent setting, which amounts to PAC-Learning the hypothesis class:
\begin{equation}\label{eq:NTI-e2e}
 \NTIete{\cF}{T}=\left\{\, \bx\mapsto g(\bx)[-1] \,\middle|\, g \in \NTICoT{\cF}{T}\,\right\}.
\end{equation}
Sample complexities obtained by \citeauthor{malach2023auto} all scale linearly with $T$, and this is unavoidable without time-invariance---see also discussion in \Cref{sec:samples-computational-complexity-general-F}. Our goal is to avoid such dependence.  

%\citet{foster2024behavior} made a similar distinction in discussing Behavioral Cloning, where they consider a joint choice of policy for all time steps (parameter sharing), versus independent policies at each time step, and how this is essential for avoiding extra dependencies on the ``horizon'' (equivalent to 
%our generation length)---see also the discussion in \Cref{sec:discussion}. \abcomment{I wonder if this discussion would not be more natural in our related work section?  As it is I feel it interrupts the flow of the preliminaries and is not really necessary for understanding the rest of the paper.}
% \natinote{We need to discuss relationship to BC somewhere, as related work.  Not sure where.  Maybe end of intro, maybe end of paper.  Maybe here?  Also, Adam---feel free to add to this.}\abcomment{I removed this as it is now redundant with the related work paragraph in the intro }

\paragraph{Computational Complexity.}  Although all our learning guarantees are for learning over the entire $\cX=\Sigma^*$, and do not have any dependence on the length of the input, when discussing computational complexity, the length of the input will be important. Likewise, instead of considering a fixed base class $\cF$, we must consider a family $\cF_d$ parametrized by (one or more) size parameters $d$.  We will say that $\ete{\cF_d}{T}$ is $\sete$ or $\sCoT$ learnable in $\runtime(n,d,T,\eps,\delta)$ using some learning algorithm\footnote{Formally, to allow uniform algorithms, we can think of $A$ being implicitly passed $T$ and $d$.} $A$, if over the domain $\cX=\Sigma^{\leq n}$ (i.e.~restricted to prompts of length at most $n$), $A(S)$ runs in time at most $\runtime(n,T,d,\eps,\delta)$ almost surely. For an expression $\kappa(\psi_1,\dots,\psi_k)$ in scalar quantities $(\psi_1,\dots,\psi_k)\in \R^k$, we say that $\kappa$ is in $\poly(\psi_1,\dots,\psi_k)$ if $\kappa$ is uniformly bounded by a polynomial for all $(\psi_1,\dots,\psi_k)\in \R^k$, i.e. there exists a polynomial $p:\R^k \to \R$ such that for every $(\psi_1,\dots, \psi_k)\in \R^k$, we have $\kappa(\psi_1,\dots, \psi_k)\leq p(\psi_1,\dots, \psi_k)$.

\section{Statistical and Computational Complexity for a General \texorpdfstring{$\cF$}{}}\label{sec:samples-computational-complexity-general-F}
%\gal{I think that we only refer to Table 1 from the discussion section (which is 11 pages later. Might help if we add a reference somewhere before that.}
In this section, we discuss the learnability of $\ete{\cF}{T}$ for a hypothesis class $\cF$ in terms of general properties of $\cF$; see \Cref{tab:sample-complexity}. Complete  proofs can be found in \Cref{app:for Section3}.
% \sg{The table entries are a bit confusing, what does sample-based mean? Also, the entry for CoT Ldim the lower bound seems wrong.}
\begin{table}[t]
  \centering
\resizebox{\textwidth}{!}{  \begin{tabular}{|c|c|c|c|}
    \hline
\multicolumn{2}{|c|}{\rule{0pt}{2ex}} &  $\sete$ (Latent CoT) & $\sCoT$ (Available CoT) \\ 
    \hhline{|= =|=|=|}
  \multicolumn{2}{|c|}{\rule{0pt}{2ex} Learning Rule} & \ref{eq:ete-consistent} & \ref{eq:cot-consistent} \\
    \hhline{|= =|=|=|}
    \rule{0pt}{3ex}
   \multirow{3}{*}{\rotatebox{90}{\small Sample Complexity}} & $|\cF|$ bounded & $\log \abs{\cF}$ & $ {\log \abs{\cF}} $ \\
   & & (\Cref{thm:cardinality-based}) & \\
    \cline{2-4}
   \rule{0pt}{3ex}
    & $\VCdim(\cF)$ bounded  & $ \boldsymbol{T}\cdot \VCdim(\cF)$ &  $\VCdim(\cF) \log T$\\
    &   & (\Cref{thm:ete-VC,thm:lb-Omega(T VC)}) & (\Cref{thm:cot-VC})\\
    \cline{2-4}
    \rule{0pt}{3ex}
   & $\ldim(\cF) $ bounded & $\ldim(\cF) \log T $ &  $\textcolor{gray}{\ldim(\cF) \log T}$\\ 
    &  & (\Cref{thm:ete-ldim})& \small (dominated  by \textcolor{black}{$\VCdim(\cF) \log T$})  \\ 
    \hhline{|=|=|=|=|}
    \multicolumn{2}{|c|}{\rule{0pt}{2ex}Computational complexity when} & Not necessarily $\poly(n,T,d)$ & $\poly(n,T,d)$\\
  \multicolumn{2}{|c|}{$\Cons$ for $(\cF_d)$ is tractable} &  (\Cref{sec:linear}) &  (\Cref{cor:tract-Cons-oracle-tract-CoT-lrn}) \\
   \hline
  \end{tabular} }
 \caption{\small A summary of the results in \Cref{sec:samples-computational-complexity-general-F} for time-invariant autoregressive learning for a general base class $\cF$. Each sample complexity row indicates the best possible guarantee, up to $O(\log (1/(\eps\delta))/\eps)$, on the sample complexity of $\sete$ or $\sCoT$ learning under the corresponding assumption on $\cF$.  All indicated sample complexities are tight except for $\log T$ factors, i.e.~for each cell there exists classes $\cF$ matching the indicated upper bound except for the $\log T$ factor (see \Cref{app:sample-complexity-side}).} 
  \label{tab:sample-complexity}
\end{table} 
\subsection{Learnability and Sample Complexity}
% \abcomment{For the table, should we replace `bounded' by $\ldim(\cF) < \infty$ and so on?}
Our first observation is simple, yet as we will see, also powerful: if the cardinality $\abs{\cF}$ of the base class is bounded, then so is the cardinality of the End-to-End and Chain-of-Thought classes:
\begin{equation}\label{eq:cardinality-of-ete-cot}
\abs{\ete{\cF}{T}}, \abs{\CoT{\cF}{T}}\leq\abs{\cF}\, .
\end{equation}
This is already sufficient for obtaining learning guarantees with sample complexity $\log\abs{\cF}$, in particular, using a learning rule requiring ``consistency'' with the final answer:
\begin{tcolorbox}
Given $S_{\sete}=((\bx_1,y_1),\dots,(\bx_m,y_m))$,
    \begin{equation}\label{eq:ete-consistent}
    \text{Return $\ete{\hat{f}}{T}$, for some $\hat{f}\in \cF $ such that }\,  \ete{\hat{f}}{T}(\bx_i) = y_i, \forall (\bx_i,y_i) \in S_{\sete}. \tag{$\Consete$} 
\end{equation}
\end{tcolorbox}
\begin{theorem}\label{thm:cardinality-based}
    For any $\Sigma$, base class $\cF$, and generation length $T$, we have that $\ete{\cF}{T}$ is $\sete$-learnable, and so also $\sCoT$-learnable, with \ref{eq:ete-consistent} and sample complexity
    $$ \mete{T},\mcot{T} \leq \frac{2}{\eps}\left(\log|\cF|+\log \left( 1/\delta \right)\right)\,.$$
\end{theorem}
%\gal{why do we have the factor $2$ here? I don't see it in the corollary from Shai's book that you relied on. Is it instead of writing $\lceil \cdot \rceil$ or $+1$?}
While this is a promising start, it is natural to wonder if analogous results hold for more general function classes satisfying less stringent complexity bounds, such as requiring only bounded VC dimension. Indeed, for classes over $\Sigma=\{0,1\}$ (we will return to general alphabets in \Cref{rem:|Sigma|>2} at the end of the section) we can bound $\VCdim(\ete{\cF}{T})=O(T \cdot \VCdim(\cF))$, which leads to the following sample complexity guarantee for $\sete$-learning. 
% \begin{theorem}\label{thm:ete-VC} There exists a universal constant $c>0$ such that for any base class $\cF\subseteq \{0,1\}^{\{0,1\}^*}$ and generation length $T\in \naturals_{+}$, the class $\ete{\cF}{T}$ is $\sete$-learnable using \ref{eq:ete-consistent} with $\mete{T} \leq c\cdot \eps^{\scriptscriptstyle{-1}}\left(T \cdot \VCdim(\cF) \log\left(\eps^{\scriptscriptstyle{-1}}\right) + \log\left(\delta^{\scriptscriptstyle{-1}} \right) \right)~.$\end{theorem}
\begin{theorem}\label{thm:ete-VC} For any base class $\cF\subseteq \{0,1\}^{\{0,1\}^*}$ and generation length $T\in \naturals_{+}$, the class $\ete{\cF}{T}$ is $\sete$-learnable using \ref{eq:ete-consistent} with $$\mete{T} = O\left(\eps^{\scriptscriptstyle{-1}}\left(T \cdot \VCdim(\cF) \log\left(\eps^{\scriptscriptstyle{-1}}\right) + \log\left(\delta^{\scriptscriptstyle{-1}} \right) \right)\right).$$
\end{theorem}
The linear dependence on $T$ is extremely disappointing, as our goal with time-invariance is to avoid such generation-length dependence.  In fact, we can learn the \emph{time-dependent} $\NTIete{\cF}{T}$, with $O(T \cdot \VCdim(\cF)\log T)$ samples, i.e.~as in \Cref{thm:ete-VC} up to a $\log T$ factor (\Cref{thm:NIT-VC-bound} in \Cref{app:NTI-VCdim}). Unfortunately, even with time-invariance, the linear dependence on $T$ is unavoidable:  
\begin{theorem}\label{thm:lb-Omega(T VC)} For every $D,T\in \naturals_{+}$, there exists a base class $\cF\subseteq \{0,1\}^{\{0,1\}^*}$ with $\VCdim(\cF)=D$, and a distribution $\cD$ over $ \{0,1\}^n$ for $n=\lceil \log (DT)\rceil+1$ such that for any learning rule $A$ there exists $f_* \in \cF$, s.t. with probability at least $0.8$ over $S_{\sete}$ of size $m < \frac{DT}{2}$ (sampled as in \Cref{def:learnable-wo-CoT}), we have $L_{\cD,f_*}^\zo(A(S_{\sete})) \geq \frac{1}{4}$.
\end{theorem}
Now, we see the first benefit of $\sCoT$ training.  With the Chain-of-Thought available, we can define a stronger learning rule: 
\begin{tcolorbox}
    Given $S_{\sCoT}=(\bz_1,\dots,\bz_m)$, and letting $\bx_i=\bz_i[:-(T+1)]$,
\begin{equation}\label{eq:cot-consistent}
\text{Return $\ete{\hat{f}}{T}$, for some $\hat{f}\in \cF$ such that } \bz_i=\CoT{\hat{f}}{T}(\bx_i), \forall \bz_i \in S_{\sCoT} .\tag{$\Conscot$}
\end{equation}
\end{tcolorbox}

As \ref{eq:cot-consistent} is a special case of\footnote{As with many standard learning rules, the rule is not a specific mapping, but a specification that can be implemented by many different mappings. Since any $\sCoT$ consistent $f$ is also $\sete$ consistent, any valid response of \ref{eq:cot-consistent} is also valid for \ref{eq:ete-consistent}, but not vice versa.} \ref{eq:ete-consistent}, we have that \ref{eq:cot-consistent} enjoys the same guarantees as \ref{eq:ete-consistent}, including the cardinality based \Cref{thm:cardinality-based}.  But for VC base classes, \ref{eq:cot-consistent} enjoys a much stronger (nearly independent of $T$) guarantee than in \Cref{thm:ete-VC}:
% \begin{theorem}\label{thm:cot-VC}
%     There exists a universal constant $c>0$ such that for any base class $\cF\subseteq \{0,1\}^{\{0,1\}^*}$
%     %$\cF$ for the alphabet $\Sigma=\{0,1\}$ 
%     and generation length $T\in \naturals_{+}$, the class $\ete{\cF}{T}$ is $\sCoT$-learnable using \ref{eq:cot-consistent} with $\mcot{T} \leq c \cdot \eps^{\scriptscriptstyle{-1}} \left(\VCdim(\cF) \log T \log\left(\eps^{\scriptscriptstyle{-1}}\right) + \log\left(\delta^{\scriptscriptstyle{-1}} \right) \right).$
% \end{theorem}
\begin{theorem}\label{thm:cot-VC}
    For any base class $\cF\subseteq \{0,1\}^{\{0,1\}^*}$
    %$\cF$ for the alphabet $\Sigma=\{0,1\}$ 
    and generation length $T\in \naturals_{+}$, the class $\ete{\cF}{T}$ is $\sCoT$-learnable using \ref{eq:cot-consistent} with $$\mcot{T}  = O\left( \eps^{\scriptscriptstyle{-1}} \left(\VCdim(\cF) \log T \log\left(\eps^{\scriptscriptstyle{-1}}\right) + \log\left(\delta^{\scriptscriptstyle{-1}} \right) \right)\right).$$
\end{theorem}
That is, while the VC dimension of the base class $\cF$ is {\em not} sufficient for ensuring $T$ independent sample complexity for end-to-end learning, having the chain-of-thought available during training {\em does} allow for a (nearly) generation-length independent sample complexity based only on the VC dimension of the base class---see \Cref{tab:sample-complexity}.

Returning to $\sete$ learnability, we ask whether an assumption stronger than bounded VC dimension, but not as strong as finite cardinality, can allow us to obtain $T$-independent (or nearly independent) end-to-end sample complexity.  The answer is yes---we can do so in terms of the familiar Littlestone dimension \citep{littlestone1988learning}, which characterizes online learning \citep{ben2009agnostic}, and has also found applications in other domains \citep[e.g.][]{alon2022private,bun2020equivalence}.
% \begin{theorem}\label{thm:ete-ldim} There exists a universal constant $c>0$ such that for any base class $\cF\subseteq \{0,1\}^{\{0,1\}^*}$ and generation length $T\in \naturals_{+}$, the class $\ete{\cF}{T}$ is $\sete$-learnable using \ref{eq:ete-consistent} with $\mete{T} \leq c \cdot \eps^{\scriptscriptstyle{-1}} \left(\ldim(\cF) \log T \log\left(\eps^{\scriptscriptstyle{-1}}\right) + \log\left(\delta^{\scriptscriptstyle{-1}} \right) \right)$.
% \end{theorem}
\begin{theorem}\label{thm:ete-ldim} For any base class $\cF\subseteq \{0,1\}^{\{0,1\}^*}$ with Littlestone dimension $\ldim(\cF)$, and any generation length $T\in \naturals_{+}$, the class $\ete{\cF}{T}$ is $\sete$-learnable using \ref{eq:ete-consistent} with $$\mete{T} = O\left(\eps^{\scriptscriptstyle{-1}} \left(\ldim(\cF) \log T \log\left(\eps^{\scriptscriptstyle{-1}}\right) + \log\left(\delta^{\scriptscriptstyle{-1}} \right) \right)\right).$$
\end{theorem}
Our \Cref{thm:ete-VC,thm:cot-VC,thm:ete-ldim} follow from simple but elegant covering number arguments (sketched out in \Cref{subsec:proof-sk}), followed by the standard uniform convergence of the empirical and population losses.
\begin{remark}[General finite $\Sigma$]\label{rem:|Sigma|>2} In \Cref{app:for Section3}, \Cref{cor:ete-VC-non-binary,cor:ete-ldim-non-binary,cor:cot-VC-non-binary}, we provide extensions of our \Cref{thm:ete-VC,thm:ete-ldim,thm:cot-VC} for any general finite alphabet $\Sigma$.  The bounds are similar, with the VC dimension replaced with the Natarajan dimension and the Littlestone dimension replaced with the sequential shattering dimension, but there is also an additional factor of $\log \abs{\Sigma}$ in the sample complexities. We also show in \Cref{thm:real} that with the infinite alphabet $\Sigma=\R$, even if the base class $\cF$ has bounded VC-subgraph dimension (a.k.a. Pollard, or pseudo dimension), the end-to-end class $\ete{\cF}{T}$ could have infinite VC-subgraph dimension.
%\gal{The theorem talks about the pseudo dimension. Now, if the loss function is bounded and monotone/unimodal then finite pseudo dimension implies learnability. But for a lower bound on the sample complexity I think that we need a lower bound on the fat-shattering dimension rather than the pseudo dimension (you can ask Nati, he should know everything about it). Anyway, to keep things simple, we can just mention here our result on the pseudo dimension without discussing implications.}
\end{remark}

\subsection{Proof Sketches of Theorems \ref{thm:ete-VC}, \ref{thm:cot-VC} and \ref{thm:ete-ldim}}\label{subsec:proof-sk}
Our learning guarantees are based on bounding the growth function $\Gamma_\cH(m)$ (number of possible behaviors of a function class $\cH$ on a set of $m$ points) and then applying standard concentration bounds. \looseness-1

\paragraph{CoT Learnability in terms of VC Dimension (\Cref{thm:cot-VC}).}
We define the loss class:
\begin{equation}\label{eq:loss-class-main}
    \CoT{\cL}{T}:=\{ \ell_f:\bz \mapsto \ind\{\bz \neq \CoT{f}{T}(\bx)\} \mid f \in \cF\}, \text{ where } \bx=\bz[:-(T+1)] \,.
\end{equation}
The rule \ref{eq:cot-consistent} can be expressed as requiring $\ell_f \in \CoT{\cL}{T}$ whose empirical average is zero.  We want to ensure the population mean of this $\ell_f$ is also small, as this is a direct bound on $L_{\cD,f_*}^{\zo}(f)$.  To ensure this via uniform concentration, we bound the growth function of the loss class $\CoT{\cL}{T}$.  The important observation is that the behaviors of $\CoT{\cL}{T}$ on a set $S_\sCoT=(\bz_1,\dots, \bz_m) \in (\cX \times \Sigma^T)^m$ of size $m$, are determined by the behaviors of the base class $\cF$ on the set of $mT$ prefixes:
\begin{equation}\label{eq:prefixes}
    \prefix(S_{\sCoT})=\left\{\,\bz_i[:-(t+1)]\,, \,\text{ where }  \bz_i \in S_{\sCoT}, t=1,\dots,T\, \right\}.
\end{equation}
We can therefore bound:
\begin{equation}\notag
\Gamma_{\CoT{\cL}{T}}(m) \leq \Gamma_{\cF}(mT) \leq \left( \frac{e\,m\, T}{\scriptstyle{\VCdim(\cF)}}\right)^{\VCdim(\cF)} \hspace{-2mm} \Rightarrow \VCdim(\CoT{\cL}{T}) = O( \VCdim(\cF) \log T)
\end{equation}
where the second  inequality follows from Sauer's Lemma.

\paragraph{e2e Learnability in terms of VC Dimension (\Cref{thm:ete-VC}).} With only the input given, the behavior of $\ete{\cF}{T}$ on $(\bx_1,\ldots,\bx_m)$ depends on the behavior of $\cF$ not only on these inputs, but for each $\bx_i$ also on the possible string after $t<T$ steps of generation.  A crude way of bounding the number of such possible generations is by the total number of possible ways of extending $\bx_i$ with $t<T$ additional tokens from $\{0,1\}$.  There are $\sum_{t<T}2^t \leq 2^T$ such possible extensions, so the behaviors of $\ete{\cF}{T}$ on $m$ points depend on the behaviors of $\cF$ on at most $m 2^T$ points, and we have:
$$\Gamma_{\ete{\cF}{T}}(m) \leq \Gamma_{\cF}( m \cdot 2^T) \leq \left(\frac{e\,m\, 2^T}{\scriptstyle{\VCdim(\cF)}} \right)^{\VCdim(\cF)}\hspace{-3mm}\Rightarrow\VCdim(\ete{\cF}{T})=O(T\, \VCdim(\cF)).$$

\paragraph{e2e Learnability in terms of the Littlestone Dimension (\Cref{thm:ete-ldim}).} Instead of the crude count on the number of behaviors of $\ete{\cF}{T}$ on $(\bx_1,\ldots,\bx_m)$, we consider the ``computation path'' of any $f\in \cF$ on some $\bx_i$.  This can be thought of as a tree, where the possible out-edges at each node correspond to the possible behaviors of $\cF$ on the prefix at the node.  We can construct a complete binary tree of depth $m(T+1)$ such that the covering number of $\cF$ on this tree (in the sense of \cite{rakhlin2015sequential}) upper bounds $\Gamma_{\ete{\cF}{T}}(m)$.  We then use the sequential analogue of Sauer's Lemma \citep[Lemma 5]{rakhlin2015sequential}.  Thus, we compute for large $m$:
% In order to show \Cref{thm:ete-ldim}, we analyze how many different behaviors the class can have on the latent CoT bits. For $m$ input points, each point generates $T$ CoT tokens, so we need a binary tree of depth $m(T+1)$ to capture all possible computation paths - one level for each input point and each of its potential CoT tokens. By applying an equivalent of Sauer's lemma for the Littlestone dimension \citep[Lemma 5]{rakhlin2015sequential}, we can bound:
% To show \Cref{thm:ete-ldim}, the number of behaviors in the next $T$ latent CoT bits, can be also controlled by a binary of depth $m(T+1)$ whose nodes are labelled by all possible prefixes. However, to bound this, we now have to turn to a guarantee based on $\ldim(\cF)$, and an equivalent of Sauer's lemma for it \citep[Lemma 5]{rakhlin2015sequential}. 
$$\Gamma_{\ete{\cF}{T}}(m) \leq \left(\# \text{ behaviors on a binary tree of depth $m(T+1)$} \right) \leq \left(\frac{2 e m (T+1)}{\ldim(\cF)} \right)^{\ldim(\cF)},$$
which allows us to obtain that $\VCdim(\ete{\cF}{T})=O(\ldim(\cF)\log T)$.

\subsection{Computational Complexity}\label{subsec:computational-complexity}
We now investigate the computational complexity of implementing the rules \ref{eq:ete-consistent} and \ref{eq:cot-consistent}. As discussed in \Cref{sec:setting}, while all our sample complexity guarantees are for learning over the entire $\cX=\Sigma^*$, when discussing computational complexity, the length of the input will be important. We consider a hierarchy of base classes $\cF_d$ with some size parameter $d$. If the base class $\cF_d$ itself is computationally hard to learn, we cannot generally hope for the iterated classes to be easier. But what can we say if $\cF_d$ itself is computationally easy to learn?  
Can we then implement \ref{eq:ete-consistent} and \ref{eq:cot-consistent} efficiently?

If we do have the entire Chain-of-Thought data $S_{\sCoT}$, then implementing \ref{eq:cot-consistent} amounts to finding $\hat{f}\in\cF$ that is consistent with the generation of each of the last $T$ tokens in each $\bz_i$:
\begin{equation}\label{eq:consTm}
    \text{Return } \ete{\hat{f}}{T} \text{ for some } {\hat{f} \in\cF}\; \text{such that } \; 
 \forall_{\bz_i \in S_\sCoT} \forall_{t=1,\ldots, T},\; \hat{f}(\bz_i[:-(t+1)])=\bz_i[-t]
\end{equation}
But \eqref{eq:consTm} essentially amounts to solving a consistency problem \ref{eq:cons-f}.
\begin{tcolorbox}
\begin{equation}\label{eq:cons-f}
  \textrm{Given $(\bu_i,v_i)_{i=1,\ldots,\Tilde{m}}$, return some $\hat{f}\in\cF$ such that $ \hat{f}(\bu_i)=v_i\,,$  $ \forall i \in [\Tilde{m}]$} \tag{$\Cons_\cF$}.
\end{equation}
\end{tcolorbox}
\begin{theorem}\label{thm:ConsF-gives-CoT-Cons-implementation}
   Consider any hypothesis class $\cF$ over an alphabet set $\Sigma$. The rule \ref{eq:cot-consistent} on $m$ examples of input length at most $n$ can be implemented by a single call to \ref{eq:cons-f}, with $\Tilde{m}\leq m\cdot T$ samples of input length at most $n+T$, and $O(\Tilde{m})$ additional runtime.
   %\gal{Why "additional runtime"?}
\end{theorem}
As a corollary, we have the following tractability of $\sCoT$-learnability, if the base class has a tractable consistent oracle. 
\begin{corollary}\label{cor:tract-Cons-oracle-tract-CoT-lrn}
For a family $\cF_d$, if $\Cons_{\cF_d}$ can be implemented in time polynomial in its input and the size parameter $d$, and $\VCdim(\cF_d)\leq\poly(d)$, then $\ete{\cF_d}{T}$ is $\sCoT$ learnable in time  $\Poly(n,d,T,\eps^{\scriptscriptstyle{-1}},\log \delta^{\scriptscriptstyle{-1}})$.
\end{corollary}

In the next section, we will see that the situation is much grimmer for $\sete$-learnability.  Even for tractably learnable base classes, where \ref{eq:cons-f} is computationally easy, implementing \ref{eq:ete-consistent}, or in fact $\sete$-learning using any other possible rule, could be computationally hard.  This computational gap is perhaps the biggest advantage of Chain-of-Thought training.  In \Cref{sec:universal,sec:TM}, we will further see how to leverage the computational tractability of $\sCoT$ learning exposed by \Cref{cor:tract-Cons-oracle-tract-CoT-lrn}.

\section{Autoregressive Linear Thresholds}\label{sec:linear}
% \nirmit{I think this section can really be done in 1/2 to 1 page. I see this as providing the leftover table cell from prev section (so hardness), and leaving enough footsteps to develop on in the next section. Lot of it is overkill}
In this section, we study the base class, over the binary alphabet $\Sigma=\{0,1\}$, where generators are $d$-dimensional linear thresholds applied to the last $d$ bits in their input:\removed{ Formally, if $|\bx| < d$ let $\tilde{\bx} \in \{0,1\}^d$ such that $\tilde{\bx} =: (0,\ldots,0, \bx)$, i.e., the vector obtained by padding zeros before $\bx$, and if $|\bx| \geq d$ let $\tilde{\bx} := \bx$. Then we define:}
\begin{equation*}
   \cFlin :=\left\{ f_{\bw,b}=: \ind\left[\sum\nolimits_{i=1}^{d \wedge \abs{\bx}} \bw[\,-i\,] \bx[-i] +b \geq 0 \right]  \,\middle|\, \bw \in \R^d , \; b \in \R\right\}~.
\end{equation*}
% \begin{equation*}
%    \cFlin :=\left\{ f_{\bw,b} := \ind\left[ \langle \bw, \tilde{\bx}[-d:] \rangle +b > 0 \right]  \,\middle|\, \bw \in \R^d , \; b \in \R\right\}~.
% \end{equation*}
\removed{Recall that the base class has to be defined over $\{0,1\}^*$, i.e.~over bit-strings of any length, since it will be used to generate the next bit given any prefix.  This is different from standard linear predictors, where the dimensionality of the predictor matches that of the input.}  
%Here, only the last (at most) $d$ bits are used.\removed{  We denote $f_{\bw}$ if the bias $b=0$.\natinote{I would rather avoid this notation, or alternatively just remove the bias everywhere, or easiest just include the bias in $\bw$}}

From \Cref{thm:ete-VC,thm:cot-VC}, we have that $\ete{\cFlin}{T}$ is $\sete$ and $\sCoT$ learnable with sample complexities $m^\sete \propto O(Td)$ and $m^\sCoT \propto O( d \log T)$ respectively, with the disappointing $O(T)$ scaling for $\sete$ learning.  But in this case, even though $\cFlin$ is specified in terms of real-valued parameters, the discreteness of the input domain actually allows us to bound its cardinality:
\begin{lemma}\label{lem:cardinality-of-halspaces}Over the domain $\{0,1\}^d$, we have\footnote{The lemma follows by applying Sauer's lemma on the entire domain $\{0,1\}^d$ of size $2^d$:  $\abs{\cFlin}=\Gamma_{\cFlin}(2^d)\leq (2e 2^d)^{\VCdim(\cFlin)}=2^{O(d^2)}$. } $|\cFlin|\leq (2 e \cdot 2^d)^{(d+1)}=2^{O(d^2)}\,.$  
\end{lemma}
Plugging in \Cref{lem:cardinality-of-halspaces} into \Cref{thm:cardinality-based} we see that $\sete$ learning is also possible with $T$-independent sample complexity.  Combining the VC-dimension and cardinality-based sample complexities:
\begin{corollary}\label{cor:linear-th-VCbound}
$\ete{\cFlin}{T}$ is $\sete$ and $\sCoT$ learnable using \ref{eq:ete-consistent} and \ref{eq:cot-consistent} with sample complexities
\begin{equation*}
    m^{\sete} = O\left(\frac{ \left(d^2 \wedge d\cdot T \log(\eps^{\scriptscriptstyle{-1}})\right) + \log(\delta^{\scriptscriptstyle{-1}})}{\eps} \right) \quad \text{and} \quad  m^{\sCoT} =O\left(\frac{ \left( d^2 \wedge d \log T \log(\eps^{\scriptscriptstyle{-1}})\right) + \log(\delta^{\scriptscriptstyle{-1}})}{\eps} \right)\,.
\end{equation*}
\end{corollary}
While we do not know if the $d^2$ vs.~$d \log T$ gap in the sample complexity (when $\log T \ll d$) between $\sete$ and $\sCoT$ learning is real, it is clear that there remains a statistical advantage over learning the time-dependent $\NTIete{\cFlin}{T}$  \citep[as in][]{malach2023auto}, which requires $\Omega(d \cdot T)$ samples.
%\natinote{Can we add citation or reference? Gal: it follows from malach...}
The main advantage, however, of $\sCoT$ learning over $\sete$ learning here is computational, rather than statistical. Since the consistency problem $\Cons_{\cFlin}$ amounts to Linear Programming feasibility and is thus easily solvable in polynomial time \citep{dantzig2016linear}, \Cref{cor:tract-Cons-oracle-tract-CoT-lrn} ensures that we can $\sCoT$ learn $\ete{\cFlin}{T}$ with $O( d \log T \wedge d^2)$ samples and implement the rule on inputs of length $n$ in time $\Poly(n,d,T,\eps^{\scriptscriptstyle{-1}},\log(\delta^{\scriptscriptstyle{-1}}))$: 
\begin{corollary}
\label{thm:easiness-with-CoT}
    The class $\ete{\cFlin}{T}$ is $\sCoT$-learnable in time $\poly(d,T,n,1/\eps,\log(1/\delta))$.
\end{corollary}
\begin{table}[!ht]
  \centering
  \begin{tabular}{|c|c|c|}
    \hline
     \rule{0pt}{3ex}
      & $\sete$-learning & $\sCoT$-learning \\ 
    \hline
   \rule{0pt}{3ex}
   Sample Complexity  &  $ d^2 \wedge  d \cdot T \, $ &  $d \log T \wedge d^2$\\
    \hline
    \rule{0pt}{3ex}
    Computational Complexity & Not learnable in $\poly(d,T,n)$, &  Learnable in $\poly(d,T,n)$,\\ 
      & assuming Hardness Assumption \ref{hassum:c-depth-circuit}. &  using LP feasibility.\\ 
    \hline
  \end{tabular}
  \caption{A summary of the results in \Cref{sec:linear} about learnability of time-invariant iterated linear thresholds $\ete{\cFlin}{T}$. Sample complexities follow \Cref{cor:linear-th-VCbound}. We do not know whether these bounds are tight, and it remains open whether the gap between $m^{\sete}$ and $m^{\sCoT}$ is real. In \Cref{thm:easiness-with-CoT} and \Cref{thm:hardness-of-ete-linear}, we show the tractability of $\sCoT$-learning and hardness of $\sete$-learning of $\ete{\cFlin}{T}$ respectively.}
  \label{tab:linear-thresholds}
\end{table}

On the other hand, we show that $\ete{\cFlin}{T}$ is not tractably $\sete$-learnable, i.e.~when Chain-of-Thought is not provided.  Not only is \ref{eq:ete-consistent} hard to implement, but we prove that no learning rule can $\sete$-learn $\ete{\cFlin}{T}$ in polynomial time (even improperly).  To show this, we rely on the hardness of learning constant depth threshold circuits---a well-established hardness assumption in learning theory. See \Cref{app:ete-of-LT-hard} for a formal definition of the class of linear threshold circuits.

\begin{hassumption}[Hardness of Learning $\circuit^0$]\label{hassum:c-depth-circuit} There exists a constant $L>0$ and a polynomial $p(n)$ such that threshold circuits of depth $L$ and size $p(n)$ over $n$ binary inputs are not weakly PAC-learnable in time $\poly(n)$, i.e. to constant accuracy and confidence parameters $\eps<0.5,\delta<1$.  
\end{hassumption}
\Cref{hassum:c-depth-circuit} is implied by core cryptographic assumptions such as hardness of the problems of inverting the RSA encryption function, recognizing quadratic residues, and factoring Blum integers \cite[see][Theorem 6]{kearns1994cryptographic}.  The Decisional Diffie-Hellman (DDH) assumption implies hardness already of learning depth-4 circuits, and so also implies \Cref{hassum:c-depth-circuit} \citep{naor1997number,krause2001pseudorandom}.
An intersection of halfspaces is simply a depth-two linear threshold circuit, and hence hardness of learning the intersection of polynomially many halfspaces also implies \Cref{hassum:c-depth-circuit}. This is in turn implied by hardness of the unique shortest vector problem \citep{klivans2009cryptographic}, hardness of refuting random $K$-SAT formulas \citep{daniely2016complexity}, and the existence of local pseudo-random generators  \citep{daniely2021local}.
%, and so any of these hardness assumptions also imply \Cref{hassum:c-depth-circuit}.}

\begin{theorem}\label{thm:hardness-of-ete-linear} Under \Cref{hassum:c-depth-circuit}, there is no algorithm that $\sete$-learns $\ete{\cF_{\scriptscriptstyle{d,\lin}}}{T}$ over $\{0,1\}^n$ in time $\poly(d,T,n)$ (even to within constants $\eps<0.5,\delta<1$).
\end{theorem}
\Cref{thm:hardness-of-ete-linear} follows from a reduction from circuits to time-invariant iterated linear thresholds\removed{, at the heart which is a result on the expressive power of iterated linear thresholds}: we show that any linear threshold circuit with depth $L$ and size $s$ can be emulated using a time-invariant autoregressive linear threshold of dimension $d=O(s^L)$, up to a fixed poly-time feature map: 
\begin{lemma}\label{lem:LT-expressivity}
     For any input size $n$, circuit size $s$ and depth $L$, there exists a poly-time computable mapping $\phi:\{0,1\}^n \to \{0,1\}^{n'}$, with $n' \leq (s+2)^L (n+1)$, such that for any threshold circuit of size $s$ and depth $L$ computing $C:\{0,1\}^n \to \{0,1\}$, there exists $\bw \in \R^d$, with $d \leq 2(s+2)^L (n+1)$, and $T \leq (s+2)^L(n+1)$, such that $\ete{f_{\bw}}{T}(\phi(\bx)) = C(\bx) \text{  for all  } \bx \in \{0,1\}^n$.
    \end{lemma}
In \Cref{lem:LT-expressivity}, the insistence on time-invariance, i.e.~using the same linear threshold function in every step of generation, makes the construction much trickier.  If instead we allowed a different linear predictor at each step of generation, as in $\NTIete{\cFlin}{T}$ defined in Eqs.~\eqref{eq:NTI-cot}--\eqref{eq:NTI-e2e}, we could directly encode each node in the circuit as a step of generation, with $T=s$ generation steps and dimensionality $d=s+n$.  This was the approach taken by \citet{malach2023auto} when studying time-dependent autoregressive generation.  We leave it as an open question whether it is possible to improve the time-invariant construction of \Cref{lem:LT-expressivity}, and reduce the required dimensionality, and especially the dependence on depth.  In any case, \Cref{lem:LT-expressivity} is sufficient for obtaining the hardness result of \Cref{thm:hardness-of-ete-linear}.\footnote{Improving the construction will only enable basing hardness on slightly weaker assumptions, such as hardness of learning log-depth circuits, or even existence of one-way functions.}

\removed{We leave it as an open question to improve on this expressive power result to express logarithmic, or any depth circuits of polynomial size. This will allow us to establish hardness of $\sete$-learnability under hardness assumptions weaker than Assumption \Cref{hassum:c-depth-circuit}\gal{Like what? Hardness of learning TC$_0$ already follows from the most standard cryptographic assumptions. Understanding the expressive power of iterated linear thresholds further would be nice, but I don't feel that it is important for understanding computational hardness.}. Note that \cite{malach2023auto} already expressed any polynomial size circuits using non time-stationary linear predictors (one for each threshold gate in the circuit generating a Chain-of-Thought). The main challenge in our setup is that we apply a fixed linear predictor ly---it remains open how expressive such time-invariant iterated thresholds are, beyond the circuits of constant depth and polynomial size. }

%%%%%%%%%% Expressivity %%%%%%%%%%%%%%%%%%
%%%%%%%%%% Expressivity %%%%%%%%%%%%%%%%%%
%%%%%%%%%% Expressivity %%%%%%%%%%%%%%%%%%

\section{Expressivity and Universality}\label{sec:universal}

One might possibly view \Cref{lem:LT-expressivity} as a positive result about the expressive power of autoregressive time-invariant linear thresholds, suggesting they are ``universal'' and can express any computable function.   In this sense \Cref{lem:LT-expressivity} is \removed{rather }disappointing, yielding a sample complexity of $\text{size}^{\Omega(\text{depth})}$ for learning circuits. Not only is this exponential in depth, but even for very shallow circuits, e.g.~$L=2$, the sample complexity is much larger than the size of the circuit and of the sample complexity of learning the circuit directly.

Should this motivate us to attempt and improve the construction in \Cref{lem:LT-expressivity}?  If our goal is to express circuits, the time-dependent class $\NTIete{\cFlin}{T}$ is more appropriate, and we will not be able to improve over the sample complexity it provides for learning circuits \citep{malach2023auto,bartlett2003vapnik}.  Circuits are inherently not ``time-invariant'', and their size and number of parameters are always (at least) linear in the input length, and more importantly, in the time of the computation involved. %\abcomment{@Nirmit Can we make the below question sound a bit more formal?}

%{While \Cref{lem:LT-expressivity} is applied above to prove a negative result on the hardness of learning end-to-end function classes, the result itself is positive: iterated time-invariant linear thresholds are capable of expressing arbitrary circuits.  Unfortunately, this expression is extremely inefficient in that it leads to a sample complexity of learning the circuit that is exponential in its depth.  Indeed, even in the shallow case of learning depth 2 circuits, using iterated linear thresholds for learning results in a suboptimal sample complexity.  On the other hand, there is not much to be gained in improving the efficiency of the representation in \Cref{lem:LT-expressivity} as circuits are inherently time-varying, making $\NTIete{\cFlin}{T}$ more appropriate for representation thereof; indeed, in general, one cannot improve on the sample complexity of learning a circuit by expanding the class of interest to $\NTIete{\cFlin}{T}$.  This fact motivates the following question:}
{\em What form of ``universal'' expressive power do we want then?  What would we want the sample complexity to depend on?}
%\abcomment{A few remarks.  First, Is it true that imbedding a circuit in $\NTIete{\cFlin}{T}$ resutls in an optimal sample complexity of learning?  I ported this over from the paragraph that was here but if it is true, we should cite this fact.  Second, I shortened the paragraph considerably because I am a little confused as to the role of this independent section.  Should this not either be placed in some kind of discussion or merged into section 6?  There are not really an formal results we are making and I feel the paper might flow better if we substantially abbreviate this discussion and use it as a more direct jumping off point for section 6.  Becasue of this, I have not touched the rest of the section, but if we decide to keep it as is, I will take another pass here.}
%Learning circuits can be seen as allowing us to learn the class $\TIME(\Time)$ of functions computable in time $\Time$ in $\poly(\Time)$ samples, i.e.~learning computable functions is sample complexity dependent on the runtime of the target. This can also be done with feedforward neural networks, which are indeed ``universal'' in that learning neural nets of size $\Poly(\Time)$ would allow us to learn $\TIME(\Time)$ with $\Poly(\Time)$ samples.  
One type of desirable universality is to learn all computable functions with sample complexity proportional to their runtime. More precisely, letting $\TIME(\Time)$ denote the class of functions computable in time $\Time$, we wish to learn $\TIME(\Time)$ with $\poly(\Time)$ samples. All such functions can be expressed as circuits of size $\tilde{O}(\Time)$ \citep{arora2009computational}, and hence feed-forward neural nets with $E=\poly(\Time)$ edges. 
%\abcomment{Maybe for the sake of consistency, we should refer to dimensionality here as we do below? 
 %Otherwise it seems a bit odd that we are equivocating between these notions without remarking upon the relationship.}. 
 Since, from a sample complexity perspective, neural nets (e.g.~with threshold activations) can be learned from $\tilde{O}(E)$ 
samples \citep{anthony2009neural}, neural nets already provide for ``universal'' learning in this sense: minimizing the loss on a neural net of size $\poly(\Time)$ would allow us to learn $\TIME(\Time)$ using $\poly(\Time)$ samples. The important caveat here, of course, is that learning a neural network is not computationally tractable, and we cannot actually minimize the loss on a neural net in polynomial time.

Time-{\em Dependent} Autoregressive Chain-of-Thought learning with linear thresholds, i.e.~using $\NTIete{\cFlin}{T}$ is also ``universal'' in the same way as feed-forward networks: it allows expressing the class $\TIME(\Time)$ with dimensionality and generation length $d,T=O(\Time)$, and thus learning with $\poly(\Time)$ samples. 
%but with the important advantage, pointed out by \citet{malach2023auto}, that $\sCoT$-learning is computationally tractable. 
The advantage here, as pointed out by \citet{malach2023auto}, is that we can consider Chain-of-Thought training as essentially providing us a peek into the inner working of the computation, and thus allowing {\em computationally tractable} learning.

While the above results on time-dependent autoregressive Chain-of-Thought learning are promising, we can hope for much stronger universality than just learning with number of samples proportional to computation runtime. Consider the much broader class $\prog(\S)$ of programs of length at most $\S$.  
%Program length can be much smaller than runtime.  Indeed, this is a ``time invariant'', as program length is independent of input and computation length.
%This is a much broader class than $\TIME(\S)$---broad class, 
Many very-long-runtime functions have a short program; runtime usually depends also on the input length while program length does not.
%\removed{ The class $\prog(\S)$ can truly be thought of as ``universal'' in that it captures, in some sense, any predictor we even consider using.} 
This is also a ``time-invariant'' class, since the program is independent of the input and computation length. An optimistic perspective from statistical learning theory would suggest that, due to the fact that statistical complexity scales only logarithmically in the cardinality of a hypothesis class, one could, in principle, learn $\prog(\S)$ with only $\log{\abs{\prog(\S)}}=O(\S)$ samples; as in the case of neural networks above, however, the caveat in the previous observation is the computational intractability. We really have no hope of learning $\prog(\S)$ tractably--not only is finding a short program with small error uncomputable, but also functions in $\prog(\S)$ might require unbounded amount of time to even evaluate, let alone learn.  Such functions are anyway useless as predictors, and it would be sensible to exclude them from consideration in the first place. 
%\abcomment{The preceding discussion could be shortened considerably.}\natinote{Yeah, we can probably shorten this.  It's a bit long-winded, especially for COLT.}

This leads us to consider the class $\prog(\S,\Time)$ of programs of length at most $\S$ that run in time at most $\Time$ (on inputs we are considering). Typically the program length $\S$ will be much lower than the runtime $\Time$. Can we learn this class with sample complexity $O(\S)$ (or perhaps $\poly(\S)$, or with a mild dependence on $\Time$ such as $\poly(\S\log \Time)$), and  training time $\poly(\Time)$?
Can we do this using time-invariant autoregressive learning? Can we have a (natural and simple?) base class $\cF$ for every $\S$ and $\Time$ such that:
\begin{enumerate}
    \item We have $\prog(\S,\Time) \subseteq \ete{\cF}{\poly(\Time)}$;
    \item $\ete{\cF}{\poly(\Time)}$ is learnable with $\poly(\S)$ or perhaps $\poly(\S \log \Time)$ samples;
    \item And $\ete{\cF}{\poly(\Time)}$ is $\sCoT$-learnable in time  $\poly(\Time)$?
\end{enumerate} 

\paragraph{High Context and Low Complexity Classes.} 
One reason time-invariant iterated linear thresholds, $\ete{\cFlin}{T}$ are not sufficient for universal expressibility (as defined above) is that they attend to only a limited {\em context window}.  That is, the output $f_{\bw}(\bx)$ of a linear threshold generator $f_{\bw} \in \cFlin$ depends only on the last $d$ tokens $\bx[-d:]$ of its input $\bx$---we refer to this as having a ``context length'' of $d$.  When applying $f_{\bw}$ autoregressively for $T \gg d$ steps to obtain $\bz=\CoT{f_{\bw}}{T}$, we have that $\bz[t:]$ only depends on $\bz[t-d:t-1]$, i.e.~on a finite state of $d$ bits, and we cannot expect to capture computations requiring more than $d$ bits of memory.  Since the sample complexity of learning is also linear in $d$, we see that with iterated linear thresholds the sample complexity of learning a target $h$ must be at least linear in the space complexity of computing $h$, and cannot depend only (or primarily) on the program length. To achieve our desiderata on universal expressibility (as defined above) with sample complexity (nearly) independent of program runtime (and space), we therefore need generator base classes with a long context window (i.e.~where $f(\bx)$ does not depend only on a short suffix of $\bx$) yet small complexity.\footnote{One simple candidate to consider is the class of sparse linear thresholds (studied in \Cref{subsec:sparse-linear-predictors}), which has a large context window of $d$ but its sample complexity may be only logarithmic in the context size. We see that this base class indeed satisfies two of the three desiderata, however, its expressive power remains unclear.}

\section{A Universal Autoregressive Base Class and Natural Emergence of Attention}\label{sec:TM}
% , \texorpdfstring{$\sCoT$}{}-Learnability
%\abcomment{The flow here is a bit confusing.  Maybe sections 5,6,7 should all be one section and we introduce more subsections? SG: Agreed.}
%\gal{This section and the next one discuss some technical details (e.g., defining Turing machines formally, and specifying our learning rule formally). The paper is already very long, and I believe that the main message of this and the next section can be explained at high level without the technical details (with references to the appendix for the details).}
% \abcomment{Maybe the title should be "A Universal Autoregressive...".  Also we kinda switch between autoregressive and iterated while talking about linear thresholds above; should we stick to one or the other?}
In this section, we answer the optimistic question asked in the previous section, providing an explicit generator class satisfying all three desiderata.  To do so, we first have to decide on a model of computation and formalize the notions of program length and runtime. We will work with a Turing machine as our model of computation, which we formalize in \Cref{subsec:TMdef}. In \Cref{subsec:TMcot}, we present our universal base class $\TMAR$ and prove that it satisfies our desiderata from the end of the previous section (with simple input and output transformations). Finally, in \Cref{subsec:attention}, we discuss how attention arises naturally for simulating this class.
\removed{
In this section we give the formal definition of \emph{Runtime-Bounded Turing Machines}~(\Cref{defi:TM}) and the corresponding function class computable by it~(\Cref{def:TM_func}). In \Cref{subsec:universal_class}, we present a universal autoregressive base class capable of simulating runtime-bounded Turing machines with minimal pre- and post-processing. 
}
\vspace{-0.4cm}
\subsection{Turing Machine Computation} \label{subsec:TMdef}
We formalize time-bounded computation using Turing machines, where the number of states of the Turing machine corresponds to program length and the number of computation steps to runtime.
\begin{definition}[Runtime-bounded Turing Machines]\label{defi:TM}
    A runtime-bounded Turing machine operates on an infinitely long ``Tape'' on both of its ends with cells indexed by integers $\Z$. It uses a tape alphabet set $\cA = \{0,1,\nil\}$, of which $\{0,1\}$ are input and working alphabets and $\nil$ is a blank symbol. The Turing machine is then specified by a tuple $\cM=\<\S,\Time,\tau\>$, where $\S \in \naturals_{+}$ denotes the number of internal states of the Turing machine, and $\Time \in \naturals_{+}$ is the number of steps the Turing machine runs for. The transition rule is a mapping of the form
\begin{gather}\label{eq:transition-functions}
    \tau: [\S] \times  \cA \to [\S] \times \{0,1\} \times \{-1,0,+1\}\\ \tau:\textnormal{state},\textnormal{read}\mapsto\textnormal{nextstate},\textnormal{write},\textnormal{move}\,. \notag
    \end{gather}
% We can also define the memory $\M$ as the maximum number of cells accessed by the head (other than the initialization)\zhiyuan{inconsistency here, why exclude the initial token?} on any input $\bsigma$. Clearly, we have $\M \leq \Time$.
% The runtime of the machine is the first time step value when the machine reaches the halting states, i.e. the number of time steps until the machine halts. The memory of the Turing machine is the number of cells on the tape that were accessed by the machine's head at some point during the computation. 
\end{definition}

\paragraph{Computation of Turing Machines.} At the beginning of the computation, the tape is initialized with an input $\bomega \in \{0,1\}^*$ on the cells indexed from $1$ through $|\bomega|$ and the other cells contain $\nil$. The head of the machine is at the position $p_0=|\bomega|+1$ and its state is initialized to $s_0=1 \in [\S]$. During each time step $1 \leq t \leq \Time$, the machine reads the symbol $r_t=\text{Tape}[p_{t-1}]$ at the current head position $p_{t-1}$, and according to the transition rule $(s_{t},a_t,b_t)=\tau(s_{t-1},r_t)$ updates its internal state to $s_{t}$, writes $\text{Tape}[p_{t-1}]\leftarrow a_t$ on the tape, and updates the head location to $p_{t}=p_{t-1}+b_t$\,. Finally, after the completion of $\Time$ many time steps, the machine stops\footnote{
    In a standard textbook definition of a turning machine \citep[e.g.][]{sipser1996introduction}, 
    % there is only one transition function $\tau:  [\S'] \times \cA' \to [\S'] \times \cA' \times \{-1,0,+1\} $, where the machine reads, writes and then moves its head. However, this can be simulated with our Turing machine, up to a constant factor blow up in the state space size and runtime.
    % Another distinction is that a Turing
    there is a halting state, which we do not require. This is because we have a bounded runtime $\Time$. W.l.o.g. one can always consider $\tau$ in a way that the tape does not change once the machine reaches either of the halting states, and the computation has effectively stopped.} and outputs the symbol $a_\Time$ (the symbol it wrote in the final step before moving the head).  We denote the function computed by the machine as $\TMf(\bomega)=a_\Time$.

\begin{definition}[Runtime-Bounded Turning Computable Functions]\label{def:TM_func} For any positive integers $\S \leq \Time$, the class $\TM(\S,\Time)=\left\{ \TMf: \{0,1\}^* \to \{0,1\} \middle| \tau: [\S] \times  \cA \to [\S] \times \{0,1\} \times \{-1,0,+1\}\right\}$ is the set of functions computable by some runtime-bounded Turing Machine $\langle \S,\Time,\tau\rangle$. 
\end{definition}
The number of states of the Turing machine corresponds to program length---roughly speaking, we can think of each state corresponding to a statement in the program flow chart.  More formally, the description length of a Turing machine $\langle \S,\Time,\tau\rangle$ is $O(\S \log \S)$, and by the Complexity-Theoretic Church-Turing (aka Invariance) Thesis, any reasonable model of computation can be simulated by a Turing machine up to polynomial blowup in runtime and additive overhead in description length.  In particular, $\prog(\S,\Time)$ in the RAM model of computation satisfies $\prog(\S,\Time)\subseteq \TM(\Tilde{O}
(\S),\poly(\Time))$. Thus it suffices to address the universality question for the class $\TM(\S,\Time)$.
% We will simply write $\TM(\S,\Time)$, when using $\M=\Time$.
% Clearly, $\TIME(\Time)$ and $\TM(\S,\Time)$ contains functions $\{0,1\}^* \to \{0,1\}$ that only depend on at most $\Time$ many bits of the input. Similarly, the functions in $\TM(\S,\Time,\M)$ depend only on at most $\M$ bits of the input. 

% \paragraph{Uniformity.} The above is a non-uniform view of Turing machines. To achieve the uniform view, it simply suffices to consider the hierarchy of hypothesis classses $\TIME_n(\Time(n)) $ and $\TM_n(\S,\Time(n), \M(n))$, where $\M$ and $\Time$ are allowed to vary with the input length $n=|\bsigma|$, i.e. they are functions $ \M(n)$ and $\Time(n)$. \nirmit{We should say something to compare to standard view of turing machine.} \zhiyuan{The equivalence is subtle here. Even if all length-n instances of some problem can be solved by $TIME(T(n))$, it does not necessarily mean a standard TM can solve it in $T(n)$ time. Our non-uniform version can compute $T(n)$ for free. inconcistency bewteen Time and TM as well. } \abcomment{It is not c lear at this point what this section has to do with CoT learnability, so maybe we should do some rearranging?} 

%%%%%%%%%% TRAM class construction %%%%%%%%%%% 
%%%%%%%%%% TRAM class construction %%%%%%%%%%% 
%%%%%%%%%% TRAM class construction %%%%%%%%%%% 

\subsection{Universal Autoregressive Base Class}\label{subsec:TMcot}

In this subsection, we present a base class of generators $\TMAR$ capable of simulating runtime-bounded Turing machines (with minimal pre- and post-processing maps, $\preprocess$ and $\postprocess$ defined later).  To simulate $\S$-state computation, we will use autoregressive generation over the alphabet
\begin{equation}
\Sigma:= [\S]  \times \cA    \times \{-1,0,+1\}
\end{equation}
where $\abs{\Sigma}=3 \cdot \S \cdot |\cA|=9 \S$.  Here, both the token set $\Sigma$ and our base class $\TMAR\subseteq \Sigma^{\Sigma^*}$ depend on the allowed state space size $\S$ (corresponding to program length), but importantly not the runtime $\Time$. The generator class dependence on $\S$ is unavoidable, since its complexity must grow with $\S$.  Taking the token set $\Sigma$ to depend on $\S$ is merely a convenience---we could equally well use a binary $\Sigma=\{0,1\}$ and define the generator as operating on blocks of $\log(9\S)$ bits, resulting in an additional $O(\log\S)$ factor on input size and generation length.
    
Each token $x=(s,a,b)\in\Sigma$ can be used to encode a step of Turing machine computation: a new state $s\in[\S]$, and a symbol $a\in\cA$ written on the tape, and a move $b\in\{-1,0,+1\}$.  Indeed, for $x=(s,a,b)$ we will denote $x.\state =s,\,x.\symbol=a$, and $x.\move=b$. Our universal base class $\TMAR$ will have next-token generators $f_\tau$ which correspond to each Turing machine in $\TM(\S,\Time)$, i.e.~each possible transition rule $\tau$ as in \eqref{eq:transition-functions}:  
\begin{equation*}
    \TMAR=\left\{ f_{\tau}:\Sigma^* \to \Sigma \,\middle|\, 
    \tau:[\S]\times \cA \to [\S] \times \{0,1\} \times \{-1,0,+1\}\right\}
\end{equation*}
where each $f_\tau$ maps a token sequence $\bz=(\bz[\,i\,].\state,\bz[\,i\,].\symbol,\bz[\,i\,].\move)_{i\in |\bz|}$ encoding the input and the history of computation up to time $t=\abs{\bz}-(\textnormal{input length})-1$, to the next step $z=(s_t,a_t,b_t)=f_\tau(\bz)$ the Turing machine specified by $\tau$ would make.  This is specified explicitly in \Cref{alg:f_tau}.

The input to the Turing machine is specified in $\bz$ through initial tokens that mimic writing the input on the Turing machine tape, and the output can be easily extracted from the final token corresponding to the final step of computation taken by the machine.  More specifically, we describe the following input and output transformations that map a computation input in $\bomega \in \{0,1\}^*$ to a generation input (i.e.~``prompt'') $\bx\in\Sigma^*$, and the last token of the generation output, $\bz[-1]=\ete{f_\tau}{T}(\bx)$, to a computation output.

\begin{algorithm}[t]\label{alg:universal-class}
\caption{The Predictor $f_\tau:\Sigma^* \to \Sigma$ for a transition table $\tau$}\label{alg:f_tau}
\begin{algorithmic}[1]
\Require{An input string $\bz \in \Sigma^*$}
\Ensure{The next token $z \in \Sigma$}
%\State{$(\state,\nread)=\readtape(\bz)$}
\State{ \Return $\tau(\state,\nread)$ where $(\state,\nread)=\readtape(\bz)$ given below.}
\end{algorithmic} 
\rule{\linewidth}{0.4pt}
\textbf{Subroutine: read-tape} \Comment{(independent of $\tau$)}
\begin{algorithmic}[1]
\setcounter{SubAlgLine}{1}
\Require{An input string $\bz \in \Sigma^*$ encoding the history of a Turing machine computation.}
\Ensure{ The current state $(\state)$ and the symbol at the current head-position $(\nread)$.}
\State { For each $i \in [ N], \text{ where } N=|\bz|\,:$\\
\quad\quad \textcolor{blue}{$\pos[\,i\,] = \sum_{j<i}\bz[\,j\,].\move$} \label{step:line2}} \Comment{Head position before a move, where $\bz[\,i\,].\symbol$ got written.} %\refstepcounter{ALG@line}
%\EndFor
\vspace{2pt}
\State $\npos [\,N\,] = \pos[N]+\bz[N].\move$ \Comment{Final head position.}
\State \textcolor{blue}{$\nread[N] = \begin{cases}
    \bz[j^*].\symbol &\text{ if there exists } j^* = \max_{j \leq N} \, j \, \;\text{s.t.}\, \pos[\,j\,] = \npos[N]); \\
     \nil &\text{ if $j^*$ does not exist.}
\end{cases}$} %\refstepcounter{ALG@line}
\label{step:line5}
\State \Return $(\bz[N].\state,\nread[N])$
\end{algorithmic}
\end{algorithm}

\paragraph{Pre and Post Processing.} Given any input $\bomega \in \{0,1\}^*$ with $|\bomega|=n$, the \emph{pre-processing map} $\preprocess$ is defined as $\preprocess(\bomega)\in \Sigma^{n+1}$, where 
$$\preprocess(\bomega)[\,i\,] = \begin{cases} ( 1, \nil,+1) & \text{for } i = 1 \\ (1, \bomega[\,i-1\,],+1) & \text{for } 2 \leq i \leq n+1\,. \end{cases}$$
% $$\preprocess(\bsigma)[\,i\,] = \begin{cases}
% (0, 1, \nil) & \text{for } i = 1 \\ 
% (+1, 1, \bsigma[\,i-1\,]) & \text{for } 2 \leq i \leq n+1
% \end{cases}$$
We need a special token $( 1, \nil,+1)$ to indicate the beginning of the sequence, similar to the [\textsf{BOS}] token (Beginning Of Sentence) in standard practice. \removed{Such construction is also used in previous theoretical works on simulating Turing machines using transformers~\citep{merrill2023expresssive}. }The \emph{post-processing} just amounts to returning the symbol part of the final token, so $\postprocess(x) = x.\symbol$.
% \begin{algorithm}[t]
% \caption{$\Cons_{\TMAR}$}\label{alg:ConsTMAR}
% \begin{algorithmic}[1]
% \Require{Examples $(\bu_i,v_i)_{i=1,\ldots,\tilde{m}}$ where $\bu_i\in\Sigma^*$ and $v_i\in\Sigma$.}
% \Ensure{A generator $f_{\hat\tau}\in\TMAR$ such that $\forall_i f_{\hat\tau(\bu_i)=v_i}$, if such generator exists.}
% \State For {each $i\in [\Tilde{m}]\,:$} \\ 
% \quad $\hat{\tau}(\state_i,\nread_i)=v_i$ where $(\state_i,\nread_i)=\readtape(\bu_i)$ 
%  \Comment{If conflicts, no consistent $f_{\hat\tau}$.}
% \State \Return $f_{\hat{\tau}}$ after completing the other entries of $\hat{\tau}(\cdot,\cdot)$ arbitrarily.
% \removed{
% \State{ Initialize all entries of $\tau$ to be empty}
% \For {each $\bu_i,v_i$}
% \State {$\state_i,\read_i = \readtape(\bz_i)$}
% \If {$\tau(\state_i,\read_i)$ is empty}
% \State {$\tau(\state_i,\read_i)=v_i$}
% \ElsIf {$\tau(\state_i,\read_i)\neq v_i$}
% \State{Abort}\comment{Inconsistent}
% \EndIf
% \Return{$\tau$ (completing empty entries arbitrarily)}}
% \end{algorithmic}
% \end{algorithm}

% commented for space, add back in arxiv version
% We note that only the steps in blue are non-local and require interacting with other tokens. we will return to this and discuss on attention emerge as a natural differential surrogate for these operations in the next subsection. 
\paragraph{Universality of $\TMAR$.} We now argue that this universal autoregressive base class satisfies all three desiderata from \Cref{sec:universal}.  The definition of $f_{\tau}$ as simulating computation by a Turing machine with transition table $\tau$ directly ensures that, for any transition rule $\tau$ and input $\bomega$
\begin{equation}
    \postprocess(\ete{f_\tau}{T}(\preprocess(\bomega))) = g_{\<\S,\Time,\tau\>}(\bomega)\,,
\end{equation}
% \natinote{We need a notation for function computed by a TM}
% \nirmit{NJ: You can define and use $F_{\tau,\Time}$}
and thus
$\TM(\S,\Time) \subseteq \postprocess \circ \ete{\TMAR}{\Time} \circ \preprocess$. Also, as there are only $(6\S)^{3\S}$ possible transition tables, and so $\log \abs{\TMAR}=O(\S\log \S)$. Applying the cardinality-based bound from \Cref{thm:cardinality-based}, we have 
$$\mcot{T} = O\left(\eps^{-1} ( \S \log \S + \log(1/\delta)\right)\,.$$
Moreover, the consistency problem $\Cons_{\TMAR}$ (as in \ref{eq:cons-f}) on $\tilde{m}$ examples of length $\tilde{n}$ can easily be solved in time $O(\tilde{m}\tilde{n})$: for each example $(\bu_i,v_i)$ we can use the subroutine $\readtape(\bu_i)$ from  \Cref{alg:f_tau} (which is fixed and does not depend on $\tau$) to compute the last symbol $\nread_i$ read from the tape, as well as the last state $\state_i$ of the computation specified by $\bu_i$.  The constraint $f_{\tau}(\bu_i)=v_i$ now amounts to a constraint $\tau(\state_i,\nread_i)=v_i$ on the transition function, and so all that is left is to memorize $\tau$---see explicit implementation of $\Cons_{\TMAR}$ below.
\begin{tcolorbox}
%\textbf{The rule $\Cons_{\TMAR}$ and its implementation.}
Given examples $(\bu_i,v_i)_{i=1,\ldots,\tilde{m}}$ where $\bu_i\in\Sigma^*$ and $v_i\in\Sigma$\,:
\begin{equation}\label{eq:universal-cot-consistent}
    \text{Return $f_{\hat{\tau}}$, for some $\hat{\tau}$ with $\S$ states such that } \, f_{\hat\tau}(\bu_i)=v_i,\,\; \forall {i\in [\tilde{m}]}\,. \tag{$\Cons_{\TMAR}$}
\end{equation}
%\rule{\linewidth}{0.4pt}
\textbf{Implementation:}
\begin{algorithmic}[1]
\State For {each $i\in [\Tilde{m}]\,:$} fill $\hat{\tau}(\state_i,\nread_i)=v_i$ where $(\state_i,\nread_i)=\readtape(\bu_i)\,.$ 
 \State If there are conflicts, then there is no consistent $f_{\hat\tau}$.
\State \Return $f_{\hat{\tau}}$ after completing the other entries of $\hat{\tau}(\cdot,\cdot)$ arbitrarily.
% \removed{
% \State{ Initialize all entries of $\tau$ to be empty}
% \For {each $\bu_i,v_i$}
% \State {$\state_i,\read_i = \readtape(\bz_i)$}
% \If {$\tau(\state_i,\read_i)$ is empty}
% \State {$\tau(\state_i,\read_i)=v_i$}
% \ElsIf {$\tau(\state_i,\read_i)\neq v_i$}
% \State{Abort}\comment{Inconsistent}
% \EndIf
% \Return{$\tau$ (completing empty entries arbitrarily)}}
\end{algorithmic}
\end{tcolorbox}
\begin{theorem}[Universality]\label{thm:universal-class} For every $\S \in \naturals_{+}$ and $\Time \geq \S$, we have that $\TM(\S,\Time)$ is contained in $ \ete{\TMAR}{\Time}$ (up to pre and post processing steps), and  $\ete{\TMAR}{\Time}$ is  $\sCoT$-learnable by \ref{eq:cot-consistent} with runtime $\poly(n,\Time,\eps^{\scriptscriptstyle{-1}},\log \delta^{\scriptscriptstyle{-1}})$ and sample complexity $\mcot{\Time} = O\left( \eps^{\scriptscriptstyle{-1}}\left( \S\log \S+\log \delta^{\scriptscriptstyle{-1}}\right)  \right)\,$. 
\end{theorem}
This is a corollary of the observations in the preceding paragraphs and \Cref{thm:ConsF-gives-CoT-Cons-implementation}.
\subsection{Emergence of Attention}\label{subsec:attention}
We now take a closer look at the generator class $\TMAR$, and its generators $f_\tau\in\TMAR$ specified in \Cref{alg:f_tau}, and how these generators can be implemented using ``generic'' function classes.  We intentionally wrote $f_\tau$ as operating on arrays of length $N$. The two operations involving array elements across multiple locations are the position calculation in \Cref{step:line2} of $\readtape$, which involves summing moves over multiple array locations, and the position lookup in \Cref{step:line5} on $\readtape$.  The lookup operation is explicitly an attention-type operation: we look for a location where the write position (the ``key'') matches the read position (the ``query'') and copy its tape symbol (the ``value'').  Meanwhile, the summation operation in \Cref{step:line2} can be implemented as averaging many array locations, as in uniform attention.

To be more concrete, we will show how \Cref{step:line2,step:line5} can be naturally implemented by  Average Hard Attention \citep[e.g.][]{merrill2022saturated,barcelo2023logical,strobl2024formal}:
\begin{definition}[Causal Average Hard Attention]\label{defi:attention}
For any positive integers $N$, $p$, $l$ and any sequences of vectors $\sfq[\,1\,],\ldots,\sfq[N] \in \mathbb{R}^l$ (queries), $\sfk[\,1\,],\ldots,\sfk[N] \in \mathbb{R}^l$ (keys), and $\sfv[\,1\,],\ldots,\sfv[N] \in \mathbb{R}^p$ (values), and a score function $\score: \mathbb{R}^i \to \mathbb{R}^i$ for each $i \in [N]$, the \emph{causal attention function} $\att:  (\mathbb{R}^l)^N \times (\mathbb{R}^l)^N\times (\mathbb{R}^p)^N \mapsto (\mathbb{R}^p)^N $ is defined as:
\begin{equation}\label{eq:attention-equation}
\left[ \att(\{\sfq[\,j\,]\}_{j=1}^N, \{\sfk[\,j\,]\}_{j=1}^N, \{\sfv[\,j\,]\}_{j=1}^N)\right]_i = \sum_{\ell=1}^{i} \score(\{\langle \sfq[\,i\,], \sfk[\,j\,] \rangle\}_{j=1}^i)_\ell  \, \sfv[\,\ell\,]\,\,.
\end{equation}
%\gal{$k$ is used for two different things (both as an index and as the keys)}
%Common choices for the score function include:
% 1. Softmax attention ($\sa$) with temperature $\beta$:
% \begin{equation}
% \score_\sa(\{s_j\}_{j=1}^i)_k = \frac{\exp(s_k/\beta)}{\sum_{j=1}^i \exp(s_j/\beta)}\,.
% \end{equation}
In particular, the average hard attention $\aha(\sfq,\sfk,\sfv)$ uses the following score function:
\begin{equation}\label{eq:scoreaha}
\score_\aha(\{s_j\}_{j=1}^i)_\ell = \begin{cases}
\frac{1}{|\mathcal{M}|} & \text{if } \ell \in \mathcal{M} \\
0 & \text{otherwise}
\end{cases}, \quad \text{where } \mathcal{M} = \argmax_{j \leq i}  s_j\,.
\end{equation}
That is, the output $\aha(\sfq,\sfk,\sfv)[\,i\,]\in\mathbb{R}^p$ at position $i$ is an average of the ``values'' $\sfv[\,\ell\,]$ at positions $\ell \in \arg\max_{j\leq i} \left\langle \sfq[\,i\,],\sfk[\,j\,]\right\rangle $ whose ``keys'' $\sfk[\,j\,]$ have the maximal inner product with the ``query'' $\sfq[\,i\,]$ at the output position $i$ of interest. 
\end{definition}
% % 3. Rightmost-hard attention ($\rha$):
% % \begin{equation}
% % \score_{\rha}(\{s_j\}_{j=1}^i)_k = \begin{cases}
% % 1 & \text{if } k = \max \left( \argmax_{j \in [i]} s_j \right)\\
% % 0 & \text{otherwise}
% \end{cases}
% \end{equation}
% Transformers is a natural function class which is powerful enough to express the above next-token predictor $f_\tau$ (\Cref{alg:f_tau}), yet differentiable. Each transformer layer contains two consecutive parts, connected by residual connections -- the attention layer and MLP layer. The latter are differentiable and can handle all the local computations, while the former could be viewed as a differentiable surrogate of non-local operations involved in simulating next step of Turing machine, or more generally, content-based addressing. 
%\gal{The explanations below are not clear (and some notations are not defined). So either remove this explanation or merge relevant parts from the appendix.}

\paragraph{\textcolor{blue}{Implementing \Cref{step:line2}:}} This is a summation over all previous positions, which is the same as an average over all previous positions multiplied by the number of previous position.  Such an average over all previous positions can be easily implemented by ``uniform`` attention, i.e.~where all keys and queries are zero and values from all positions are averaged. 

There are two minor complications here in using Causal Average Hard Attention as we defined it: as written in \Cref{step:line2}, we need to average over all previous positions, but not the output position, as in {\em strict} causal attention (a variant of \Cref{defi:attention} where the arg-max in \eqref{eq:scoreaha} is over $j<i$ instead of $j \leq i$). We can also implement this using non-strict causal attention (\Cref{defi:attention}) with one additional local step as follows. Define $\npos[\,i\,]=\sum_{j=1}^{i}\bz[\,i\,].\symbol$ instead; i.e. the summation including the index $i$. Once we have computed $\npos$,  a local operation $\pos[\,i\,]:=\npos[\,i\,]-\bz[\,i\,].\move$ suffices to compute $\pos$.   The second minor complication is going between the average and the sum, for which we also need to compute the current index.  This can be done thanks to the special ``beginning of sequence'' token.  Specifically, to compute $\npos$ using a non-strict causal attention, first define $\isfirst[\,i\,]:=\indct[\bz[\,i\,].\symbol = \nil]$, which indicates whether the token is the first. We can use attention to get $\invpos := \aha(\zeros,\zeros,\isfirst)$ using uniform attention, where all keys and query are the same. It is easy to check that $\invpos[\,i\,]= \frac{1}{i}$. Similarly, consider $\scaledpos := \aha(\zeros,\zeros,\bz.\move)$. Using the definition of $\aha$ (\Cref{defi:attention}), it is straightforward to verify that $\scaledpos[\,i\,] = \frac{\sum_{j=1}^{i}\bz[\,j\,].\move}{i} = \frac{\npos[\,i\,]}{i}$. Therefore, $\npos[\,i\,]=i \times \scaledpos[\,i\,]=\scaledpos[\,i\,]/\invpos[\,i\,]$.

%\gal{It is not clear here how to compute $i$. It is explained only in the next paragraph using is-first.}
%\removed{Then we use the first MLP layer to get $\pos = \scaledpos/\invpos$ using division activation and $\npos$, $\npos^{\odot 2}$, and $\pos^{\odot 2}$ using quadratic activation.}
%\paragraph{Implementing line 4:} This just corresponds to $\pos[\,i\,]=\npos[\,i\,]-\bz[\,i\,].\move$, which is a local operation and can be handled an MLP. This computes the head position before taking the move where $\bz[\,i\,].\symbol$ was written.
\paragraph{\textcolor{blue}{Implementing \Cref{step:line5}:}} This lookup is a much more direct application of attention, and indeed is reminiscent of the original motivation for attention as a basic component required for computation \citep{graves2014neural}.  Here we directly want to set the query to the current read position and the keys to the write positions and look for an exact match.  The complication here is that we want {\em the most recent} exact match.  This can be achieved by augmenting the keys with the step counter and thus seeking the highest step-count computation step among those with the exact match.

Specifically, we will get $\nread := \aha(\sfq,\sfk,\sfv)$ with hard-average attention by setting $\sfv[\,i\,] = \bz[\,i\,].\symbol$,
    \begin{align*}
\sfq[\,i\,] = (-\npos[\,i\,]^{2}, \npos[\,i\,], -1, -1), \quad 
\sfk[\,j\,] = \begin{cases}
(0,0,0,\invpos[\,j\,]) & \text{if } j=1;\\
(2, 4\pos[\,j\,], 2\pos[\,j\,]^2,\invpos[\,j\,]) & \text{otherwise}\,.
\end{cases}
\end{align*}
%\gal{for $j=1$ we get $\langle q, k \rangle = -2$, and not $-1$ as required.}
Thus the inner product between queries and products is  
$$\langle \sfq[\,i\,], \sfk[\,j\,]\rangle = \begin{cases}
-1 & \text{if } j=1\,; \\
-2(\npos[\,i\,]-\pos[\,j\,])^2 -\frac{1}{j}& \text{otherwise}.
\end{cases}$$
Therefore, if we look at the last index $\nread[N]$, then whenever $j^* = \max_{j \leq N} \, j\, \;\text{s.t.}\, \npos[N] = \pos[\,j\,]$ exists, indeed we observe that the inner product $\{\<\sfq[N],\sfk[\,j\,]\>\}_{j=1}^{N}$ is maximized uniquely at the index $j^*$, i.e. $j^*=\arg\max_{j \leq i}\<\sfq[N],\sfk[\,j\,]\>$. As a result, the average-hard attention always focuses $j^*$ and we have that $\nread[N]=\sfv[j^*]=\bz[j^*].\symbol$ as desired. When $j^*$ does not exist, it is easy to observe that the inner product $\{\<\sfq[\,N\,],\sfk[\,j\,]\>\}_{j=1}^{N}$ is maximized uniquely at the index $j=1$. Therefore, the attention focuses on the first index and returns $\sfv[\,1\,]=\bz_{\textrm{in}}[\,1\,].\symbol=\nil$, which is precisely what we need, concluding the implementation.

\paragraph{From Attention to Transformers.}  In the preceding paragraphs, we discussed how to implement the ``non-local'' operations on \Cref{step:line2,step:line5} using attention.  All other operations in the implementation of $f_\tau$ are ``local'' in the sense that they operate independently at each location of the arrays\footnote{This includes the operations inside $\readtape$ which operate explicitly at array locations, and the $\tau$-lookup at the end, which can be thought of as operating at location $N$, or in fact at each location independently and in parallel outputting the computation-step-token}.  This directly suggests a decoder-only-transformer architecture \citep{vaswani2017attention} where attention operations are interleaved with multi-layer perceptrons (MLPs) operating independently and in parallel at each location, as a generic way of implementing the local operations.  Most notably, the $\tau$-lookup operation is directly implementable by an MLP with $\tau$ encoded in its weights.  Using Hard Average Attention as in \Cref{defi:attention} results in a discontinuous and non-differentiable model, but a natural alternative is to relax the hard max in the score definition \eqref{eq:scoreaha} to a softmax, resulting in the familiar softmax attention widely used in practice.  Implementing $f_\tau$ with a practical transformer architecture further requires care to numerical precision, and to the specifics of implementing or approximating the local operations with MLPs and possible layer-normalizations, especially the division and multiplication operations involved in calculating the step-counters and positions.  In fact, the  Turing Machine expressivity results of \citet{perez2021attention,wei2022statistically,merrill2023expresssive} essentially do this, following a construction similar to ours, and showing how to implement it with specific transformer architectures.  

The difference here is mostly that of perspective: while recent work on Turing universality of transformers started with the transformer as suggested by \citet{vaswani2017attention} and showed how to coerce it to simulate a Turing machine, we take the reverse perspective.  Our starting point is the desiderata of \Cref{sec:universal}, which directly motivates the generator class $\TMAR$, and we then argue how implementing $f_\tau\in\TMAR$ naturally involves attention and yields a transformer-type architecture.
\paragraph{Avoiding Uniform Attention.}  While the tape lookup operation in \Cref{step:line5} very naturally motivates attention, and also seems unavoidable and fundamental to universal computation, here we discuss how the use of uniform attention to calculate the current Turing Machine head position in \Cref{step:line2} can be avoided.  To do so, instead of writing out the relative move $b_t$ of each step of the Turing computation (as in \Cref{defi:TM}), we can write out the absolute tape position $p_t$. We cannot do so with a single token, since $p_t$ could potentially be as large as the runtime $\Time$, necessitating an alphabet of size $\Sigma = \Omega(\Time)$, something we would like to avoid. But we could easily write down (i.e.~generate) $p_t$ over multiple tokens, even over a binary alphabet $\Sigma=\{0,1\}$, using some prefix unambiguous encoding. The generator $f_\tau$ would then need to collect the encoding of $p_{t-1}$ (as well as the previous state and symbol written), do the arithmetic to calculate the next absolute position $p_t=p_{t-1}+b_t$ (as well other operations in \Cref{alg:f_tau}), figure out where in the middle of generating the description of the next compute state we are at, and generate the next bit in the description.  Implementing this with a transformer-like architecture involves multiple attention layers, mostly with short-range attention to the last $O(\log \Time)$ tokens as well as a single long-term sparse attention to implement the lookup on \Cref{step:line5}, and a transformer of size (and depth) that scales with $\log \Time$.  But it avoids the need for uniform attention.  A similar approach can also simulate RAM-computation, where the non-local operations are reads from memory addresses describable with $\log \Time$ bits. %(or other tokens).

%%%%%%%%%% Discussion %%%%%%%%%%% 
%%%%%%%%%% Discussion %%%%%%%%%%% 
%%%%%%%%%% Discussion %%%%%%%%%%% 

\section{Discussion and Future Directions}\label{sec:discussion}
We investigate training based solely on information about the final answer, as well as training with supervision about the intermediate reasoning steps used to reach that answer.  Training based on explicit Chain-of-Thought examples is studied by \cite{nye2021show,zelikman2022star,chung2024scaling,lewkowycz2022solving}, and is widely used in contemporary training of frontier LLMs, particularly in the fine-tuning step \citep{openai2023gpt4,google2024gemini3,meta2023llama2}.

More recently, there have been attempts to train models using only feedback on the final answer, attempting to learn the reasoning process using reinforcement learning techniques \citep{deepseek2025cot, ouyang2022training}. Our results help highlight and formalize the intractability of such approaches in general, without additional assumptions. So, why have approaches like DeepSeek and others seen empirical success? One possible explanation is that these Reinforce-type methods assume access to a sufficiently good (possibly randomized) next-token predictor $f \in \mathcal{F}$ that already produces correct (or at least useful) CoTs with non-negligible probability with some level of coverage over valid reasoning paths. Such a predictor is typically learned during pretraining on large-scale text data, which, while not always matching the target task’s CoTs exactly, are often similar enough to enable transfer. This type of training is not exactly CoT supervision, but is arguably closer to it than pure end-to-end supervision. Understanding how transfer from such related-but-different next-token prediction tasks enables tractable learning—using only final-answer supervision, and little or no direct CoT data for the task of interest—is an extremely interesting direction for future work.

In summary, we presented a framework of time-invariant autoregressive Chain-of-Thought learning. Using our framework, we discussed how:

\begin{itemize}[noitemsep]
% \vspace{-3mm}
\item Time-invariance can (nearly) avoid the dependence on the number of steps of generation and computation in sample complexity, and allows for universal learning with sample complexity dependent only on program length (and not runtime!).  This is independent of CoT training and holds also for $\sete$ learning.
%\gal{(using the class' cardinality or Ldim)}\nirmit{I don't want to have any technical aspect ``using this or that''. Already enough has been said. This is meant to be high level takeaway. In particualr, in universality result it is already true for ete learning. And the statement reads well as it is as a continuation of the previous statement.}
% \vspace{-2mm}
\item CoT training allows for overcoming computational difficulties and for tractable training even of universal classes. 
 This is independent of time invariance and is true also for time-dependent models.
 % \vspace{-2mm}
\item Attention arises naturally in time-invariant base classes allowing such universal expressivity.
\end{itemize}
\begin{table}[!ht]
    \centering
    \begin{tabular}{c|c|c}
       Learning Goals & Nearly $\sCoT$-length  & Computational  \\
          & Independent Samples & Tractability  \\
         \hline
        Time Invariant $\sete$ & \cmark & \xmark\\
        \hline 
        Time Dependent $\sCoT$ \citep{malach2023auto} & \xmark & \cmark \\
        \hline
        Time Invariant $\sCoT$ & \cmark & \cmark
    \end{tabular}
    %\caption{Caption}
    \label{tab:takeway}
\end{table}
% \paragraph{Relationship to Related Work.} 
We were directly inspired by the work of \cite{malach2023auto}, and address what we view as its major deficiency, namely time dependence.  %Our treatment is also much more comprehensive, establishing generic sample complexity results and proving hardness of $\sete$ learning. \abcomment{I think this latter sentence could be read as a bit overly critical of Eran's paper.  Maybe we should rephrase?}

\paragraph{Relationship to Behavioral Cloning.} Recent work has connected autoregressive generation to \emph{Behavior Cloning} (BC) \citep{pomerleau1988alvinn,bain1995framework,ross2010efficient}, and indeed BC has been suggested as an approach for understanding modern language models \citep{chang2023learning,block2023provable}. Those works view autoregressive generation as a Markov Decision Process (MDP) with a ``state'' at step $t$ being the entire sequence generated thus far, i.e.~$\CoT{f}{t}(\bx)$, actions corresponding to tokens $a\in\Sigma$, and transitions being deterministic mapping a state $\bz$ and an action $a\in\Sigma$ to the next state $\textrm{append}(\bz,a)$.  In that framework, our ``generators'' $f$ become deterministic ``policies'' outputting the action to take (i.e.~token to append) at each state (i.e.~string, or ``context''). Importantly, the reward is only on the final state $\bz=\CoT{f}{t}(\bx)$ and depends only on the correspondence of the last token in $\bz$ to the prefix containing the initial state. BC considers learning a policy maximizing some reward based only on observations from an expert, in our case the ground truth $f_*$, and {\em without observing the ground-truth rewards}.  Most relevant to our paper is the recent work of \citet{foster2024behavior}, which provides general, horizon-free (i.e., independent of $T$) guarantees on the performance of BC.  As in our work, \citet{foster2024behavior} considers the realizable setting, where the expert $f_*$ is in a given policy class, and distinguishes between the time-invariant (called parameter-sharing in that work) and time-varying cases, noting that the former is essential for the aforementioned horizon-free guarantees.  \citet{foster2024behavior} considers only the CoT framework, where `actions' at each time step are available to the learner, and indeed, the implication of \Cref{thm:cardinality-based} for CoT learning is a special case of their result.  That work is focused on the statistical task itself and is not interested in either the end-to-end setting or the computational benefits and drawbacks thereof.

Another innovation in \cite{foster2024behavior} is the consideration of stochasticity in both the environment and the experts. The lack of reward observability (combined with a stochastic expert) necessitates the use of randomized policies—unlike in supervised learning, where deterministic policies are always optimal. In our setting, when mapped to Behavioral Cloning, observing the expert demonstration already reveals the reward, so this consideration is somewhat moot.

%\abdelete{

%As with much of the study of BC, including the relevant recent work of \citet{foster2024behavior}, we assume ``realizability'', in that the demonstrating expert $f_*$ is in our policy class (i.e.~base class $\cF$ of generators).  BC always assumes the expert demonstrations are fully visible, and so corresponds to our notion of Chain-of-Thought learnability;  indeed, our learning guarantee for CoT learning based on the cardinality $\abs{\cF}$ (the implication of \Cref{thm:cardinality-based} to Chain-of-Thought learnability, though not the End-to-End learnability) is a special case of the learning guarantees of \cite{foster2024behavior}.  Specializing to autoregressive generation, we go well beyond cardinality-based guarantees, discuss VC-based guarantees, gaps between End-to-End and Chain-of-Thought learnability, computational hardness.
%We also note that \citet{foster2024behavior} made a similar distinction of time invariance (parameter sharing) versus time-dependent policy classes (independent policies at each time step) in discussing Behavioral Cloning, and how this is essential for avoiding extra dependencies on the ``horizon'' (equivalent to our generation length).}

\paragraph{Further Open Questions.} Our work leaves several technical questions open such as: What is the true sample complexity of $\sete$ and $\sCoT$ learning of iterated linear thresholds, i.e. lower bounds for \Cref{cor:linear-th-VCbound}?  Can the circuit embedding construction of \Cref{lem:LT-expressivity} be improved?  Can the dependencies on $\log T$ in \Cref{tab:sample-complexity} be resolved, and more significantly the dependence on $\log\abs{\Sigma}$ mentioned in \Cref{rem:|Sigma|>2}?  Can $\sete$ learnability be ensured using a notion more relaxed than the Littlestone dimension?  

Additionally, it is interesting to study other base classes under our framework, including sparse linear thresholds, gated generators, simple attention mechanisms and transformers.  What are the sample and computational complexities of $\sete$ and $\sCoT$ learning?  What is the expressiveness of $\ete{\cF}{T}$? Another major challenge is generalizing our framework to the agnostic setting.  In particular, what is the correct way to model misspecification for the $\sCoT$ data?

Our framework also allows us to rigorously study the computational complexity of settings on the spectrum between $\sete$ and $\sCoT$ learning.  What can be done tractably if we have a small amount of full CoT data, and a large amount of end-to-end input-output pairs?  What is the minimum amount of CoT data required for tractable learning if we have access to an unlimited amount of end-to-end data?  How helpful is CoT data from a mixture of different generators $f_i$ all computing the same end-to-end function $\ete{f_i}{T}=\ete{f_j}{T}$ but $\CoT{f_i}{T}\neq\CoT{f_j}{T}$ ?

In this paper, we studied only in-distribution generalization, with a fixed generation length. But insisting on time invariance should allow us to study the much more complex question of length generalization and other forms of out-of-distribution generalization.  

Finally, much of the actual use of large language models stems from their ability to learn not only from solutions to the task of interest (as we do here), but also transfer from modeling unrelated and unlabeled data. One can consider extending our framework to allow studying of such transfer learning, with access to data not solving the desired task but generated by the same generator.
\vspace{-2mm}
\paragraph{Acknowledgments.} NJ and NS would like to thank Gene Li for helpful discussions. AB would like to thank Akshay Krishnamurthy for helpful discussions. Part of this work was done when SG, ZL, and NS were participating in the Special Year on Large Language Models and Transformers, and the Program on Modern Paradigms in Generalization at the Simons Institute for the Theory of Computing at UC Berkeley. NJ, GV, TM, and NS were supported in part by the NSF-Simons Sponsored Collaboration on the Theoretical Foundations of Deep Learning and the NSF TRIPOD Institute for Data, Econometrics, Algorithms, and Learning (IDEAL). GV was supported in part by the Zuckerman STEM Leadership Program, and by research grants from the Center for New Scientists at the Weizmann Institute of Science, and the Shimon and Golde Picker -- Weizmann Annual Grant. SG was supported in part by the OpenAI SuperAlignment grant.

\bibliography{bibliography}
\newpage
\appendix
\section*{Appendix Table of Contents}
\startcontents[sections]
\printcontents[sections]{l}{1}{\setcounter{tocdepth}{2}}

% \section{Related Work}\label{app:related}  In addition to \citet{malach2023auto}, recent work has explored the effects of Chain-of-Thought on autoregressive generation.  Relevant to this, several works \citep{chang2023learning,block2024butterfly} have connected the task of autoregressive generation to \emph{Behavior Cloning} (BC) \citep{pomerleau1988alvinn}, a common approach to learning complex behavior in an interactive setting by imitating an `expert', and leveraged this connection to better understand the difficulty of the former.  In particular, \citet{foster2024behavior} demonstrate that \emph{parameter sharing across time steps} is essential to ensuring that the difficulty of learning an autoregressive model does not increase with the generation length.  The focus of that work is very different, however, in that their setting requires learning arbitrary (bounded) functions of the Chain-of-Thought generation in contradistinction to our focus, where the only relevant quantity is the final output; due to this generality, the BC setting does not admit a natural analogue of end-to-end learning, whose distinction from Chain-of-Thought is the primary focus of this paper.
\section{Technical Preliminary}\label{app:preliminary}

\paragraph{Convention for Runtime of Learning.}  Instead of considering a fixed base class $\cF$, we must consider a family $\cF_d$ parametrized by (one or more) size parameters $d$.  We will say that $\ete{\cF_d}{T}$ is $\sete$ or $\sCoT$ learnable in $\runtime(n,d,T,\eps,\delta)$ using some learning algorithm\footnote{Formally, to allow uniform algorithms, we can think of $A$ being implicitly passed $T$ and $d$.} $A$, if over the domain $\cX=\Sigma^{\leq n}$ (i.e.~restricted to prompts of length at most $n$), $A(S)$ runs in time at most $\runtime(n,T,d,\eps,\delta)$ almost surely. For an expression $\kappa(\psi_1,\dots,\psi_k)$ in scalar quantities $(\psi_1,\dots,\psi_k)\in \R^k$, we say that $\kappa$ is in $\poly(\psi_1,\dots,\psi_k)$ if $\kappa$ is uniformly bounded by a polynomial for all $(\psi_1,\dots,\psi_k)\in \R^k$, i.e. there exists a polynomial $p:\R^k \to \R$ such that for every $(\psi_1,\dots, \psi_k)\in \R^k$, we have $\kappa(\psi_1,\dots, \psi_k)\leq p(\psi_1,\dots, \psi_k)$.

We now provide some preliminary background: (i) the definitions of the complexity measures which we use throughout the paper, (ii) generalization bounds in terms of appropriate complexity measures.

\paragraph{Supervised Learning.} Consider the supervised learning setup under the 0-1 loss. Consider an input domain $\cX$, a finite label set $\cY$, and a hypothesis class $\cH \subseteq \cY^\cX$. For any $h: \cX  \to \cY$, define 
$$\cL_{\cD}(h)=\E_{(\bbx,\bby)\sim \cD } \left[ \ind\{h(\bbx) \neq \bby \} \right],$$
where $\cD$ is an unknown distribution over $\cX \times \cY$. We will say that $\cD$ is realizable by a hypothesis class $\cH\subseteq \cY^{\cX}$ if $\cD_{\bby \mid \bbx}=h_*(\bbx)$ for some $h_* \in \cH$. The distribution $\cD$ is unknown and our goal is to learn from samples $S=((\bbx_1,\bby_1), \dots, (\bbx_m,\bby_m)) \iid \cD^m$. In all our learnability results (in the positive direction) will be via analyzing the performance of Empirical Risk Minimization (ERM) rule; it takes $S \in (\cX \times \cY)^*$ as the input and outputs a predictor from $\cH$ that has the lowest 0-1 error on the training set $S$.
\begin{equation}\label{eq:ERM}
    \mathcal{L}_S(h)=\frac{1}{|S|} \sum_{(\bbx_i,\bby_i) \in S} \ind\{h(\bbx_i) \neq \bby_i\} \quad \text{ and } \quad \ERM_{\cH}(S)=\hat{h}=\arg \min_{h \in \cH} \, \,\mathcal{L}_S(h)\,. \tag{ERM}
\end{equation}
In the realizable case, this reduces to the following rule. \text{Given $S=((\bbx_1,\bby_1),\dots,(\bbx_m,\bby_m))$}
\begin{equation}\label{eq:general-consistent}
  \Cons_{\cH}(S):  \text{Return some } \hat{h} \in \cH \text{ such that } \hat{h}(\bbx_i)=\bby_i, \, \forall (\bbx_i, \bby_i) \in S\,.
\end{equation}

The sample complexity of the learning rule is characterized by some combinatorial dimensions, which we discuss now.  
\subsection{Growth Function, Complexity Measures}
We start by defining an important definition of the growth function.
\begin{definition}[Growth Function]\label{def:growth-function} For any $h:\cX \to \cY$, and $S=(\bbx_1,\dots, \bbx_m) \in \cX^m$, we define $h(S):=(h(\bbx_1),\dots, h(\bbx_m))$ and $\cH(S):=\{h(S): h \in \cH \}$. The growth function $\Gamma_{\cH}: \naturals_{+} \to \naturals_{+}$
$$ \Gamma_{\cH}(m)= \max_{(\bbx_1, \dots, \bbx_m) \in \cX^m} \left| \{(h(\bbx_1), \dots, h(\bbx_m)) : h \in \cH\} \right|= \max_{S \in \cX^m} |\cH(S)|.$$
We will also abuse the notation and may denote $\Gamma_{\cH}(S)=|\cH(S)|$ as the number of behaviors possible on $S$ using hypothesis class $\cH$.
\end{definition}
When $\cY=\{0,1\}$ is binary, we can define the VC dimension of the class, and its relationship with the growth function given by Sauer's lemma.
\begin{definition}[VC Dimension]\label{def:VCdim} When $\cY\in \{0,1\}$, for any $\cH \subseteq \{0,1\}^{\cX}$, we say that $\cH$ shatters a set of points $S \in \cX^m$ iff $\left| \cH(S)\right| = 2^{|S|}.$ The Vapnik-Chervonenkis dimension of the class, denoted by $\VCdim(\cH)$, is the largest integer $D \in \naturals_{+}$ such that $\Gamma_{\cH}(D)=2^D.$ If no such $D$ exists, then we say that $\VCdim(\cH)=\infty$\,.
\end{definition}
\begin{lemma}[Sauer's Lemma]\label{lem:sauer'slemma} For a hypothesis class $\cH \subseteq \{0,1\}^{\cX}$, for every $m \in \naturals_{+}$, we have $ \Gamma_{\cH}(m) \leq (em)^{\VCdim(\cH)}$. Additionally, for $m \geq \VCdim(\cH) \geq 1$:
$$ \Gamma_{\cH}(m) \leq \left( \frac{em}{\VCdim(\cH)
} \right)^{\VCdim(\cH)}.$$  
\end{lemma}
A classic generalization of the VC dimension to non-binary finite outputs is the Natarajan dimension. 
\begin{definition}[Natarajan Dimension]\label{def:natarajan-dimension} Consider any finite $\cY$. We say that a set $S \in \cX^m$ is shattered by $\cH$ if there exist two functions $h_0, h_1 :\cX \to \cY$ such that 
\begin{itemize}
    \item For every $\bbx \in S $, we have $h_0(\bbx) \neq h_1(\bbx)$
    \item For every $U \subseteq S$, there exists $h \in \cH$ such that 
    $$ \forall\, \bbx \in U,\, h(\bbx)=h_0(\bbx) \quad \text{ and }\quad \forall \, \bbx \in S \setminus U,\,  h(\bbx)=h_1(\bbx)\,.$$
\end{itemize}
The Natarajan dimension of $\cH$, denoted by $\Ndim(\cH)$, is the cardinality of the largest $S$ that is shattered by $\cH$. If there is no largest size, then $\Ndim(\cH)=\infty$.
\end{definition}
There is a classic generalization of the Sauer's lemma (\Cref{lem:sauer'slemma}) also for multiclass labels due to Natarajan \cite{natarajan1989learning}. Below is a variant \cite[Lemma 29.4]{shalev2014understanding}.
\begin{lemma}[Natarajan's Lemma]\label{lem:Natarajan's lemma} Recall \Cref{def:growth-function} of the growth function. For any $m \in \naturals_{+}$
$$\Gamma_{\cH}(m) \leq (|\cY|^2\cdot m)^{\Ndim(\cH)}\,.$$
\end{lemma}
We will bound the VC dimension (or Natarajan dimension) of the end-to-end class in terms of the Littlestone dimension (or sequential fat shattering dimension) of the base class to achieve an improved dependence in terms of $T$ (\Cref{thm:ete-ldim}). We define these measures now.
\begin{definition}[$\cX$-labeled Tree]
A $\cX$-labeled tree $\bx$ of depth $d$ is a rooted complete binary tree with nodes labeled by elements of $\cX$. The tree $\bx$ is identified by the sequence $(\bx_1, \ldots, \bx_d)$ of labeling functions $\bx_i: \{0, 1\}^{i-1} \to \cX$ that provide the labels for each node. Here, $\bx_1 \in \cX$ is the label for the root of the tree, and $\bx_i$ for $i > 1$ is the label of the node obtained by following the path of length $i-1$ from the root, with $1$ indicating `right' and $0$ indicating `left'. 
\end{definition}
A path of length $d$ is denoted by the sequence $\epsilon = (\epsilon_1, \ldots, \epsilon_d) \in \{0,1\}^d$. For brevity, we write $\bx_t(\epsilon)$, but it is understood that $\bx_t$ only depends on the prefix $(\epsilon_1, \ldots, \epsilon_{t-1})$ of $\epsilon$.

\begin{definition}[Littlestone Dimension]
A $\cX$-labeled tree $\bx$ of depth $d$ is said to be shattered by a binary function class $\cH \subseteq \{0,1\}^\cX$ if for all $\epsilon \in \{0,1\}^d$, there exists $h \in \cH$ such that for all $i \in [d]$, $h(\bx_i(\epsilon)) = \epsilon_i$. The Littlestone dimension $\ldim(\cH)$ is the largest $d$ such that $\cH$ shatters an $\cX$-labeled tree of depth $d$.
\end{definition}

\begin{definition}[Sequential Fat-Shattering Dimension]\label{def:sfat}
    Given a function class $\cF\subseteq \R^{\cX}$, we say that $\cF$ sequentially shatters at scale $\alpha$ a binary tree $\bx$ of depth $m$ if there exists a real-valued complete binary tree $\bs$ of depth $m$ such that for all $\epsilon \in \left\{ \pm 1 \right\}^m$, there is some $f_\epsilon \in \cF$ such that $\epsilon_i (f_\epsilon(\bx_i(\epsilon)) - \bs_i(\epsilon)) \geq \alpha / 2$.  We let $\sfat_\alpha(\cF)$ be the maximal $m$ such that $\cF$ sequentially shatters at scale $\alpha$ a binary tree of depth $m$.
\end{definition}

% \begin{lemma}[Theorem 7 from \cite{rakhlin2015sequential}]\label{thm:seq}
%     Consider a binary-valued function class $\cH \subseteq \left\{ 0, 1 \right\}^\cX$ with $\ldim(\cH) < \infty$. Then for all $\cX$-labeled trees $\bx$ of depth $d \geq \ldim(\cH) \geq 1$, 
%     % there exist $m = (2 e kn)^d$ trees $\bh_1, \dots, \bh_m$ such that for all $\epsilon \in \left\{ \pm 1 \right\}^n$ and all $h \in \cH$, there exists some $i^\star = i(h, \epsilon)$ such that $\bh_{i^\star}(\epsilon) = h(\bx(\epsilon))$.  In other words, 
%     \[
%     \left|\cH(\bx)\right| \leq  \left( \frac{2 e d }{\ldim(\cH)} \right)^{\ldim(\cH)}.
%     \]
%     where $\cH(\bx) = \left\{ (h(\bx_1(\epsilon)), \dots, h(\bx_d(\epsilon))) \mid \epsilon \in \left\{ 0,1 \right\}^d, h \in \cH \right\}$.  More generally, for $\cH \subseteq \{0,\dots,K\}^{\cX}$ with finite $\sfat_1(\cH)$, it holds for $d \geq \sfat_1(\cH)$ that
%     \begin{align}
%         \left|\cH(\bx)\right| \leq  \left( \frac{e K d }{\sfat_1(\cH)} \right)^{\sfat_1(\cH)}
%     \end{align}
% \end{lemma}
\begin{lemma}[Theorem 7 from \cite{rakhlin2015sequential}]\label{thm:seq}
    Consider a binary-valued function class $\cH \subseteq \left\{ 0, 1 \right\}^\cX$ with $\ldim(\cH) < \infty$. Then for all $\cX$-labeled trees $\bx$ of depth $d \geq \ldim(\cH) \geq 1$, there exists a set of trees $V(\bx)$ of size
    \begin{align}
    |V(\bx)| \le \left( \frac{2 e d }{\ldim(\cH)} \right)^{\ldim(\cH)},
    \end{align}
    such that for all $\epsilon \in \left\{ 0,1 \right\}^d, h \in \cH$, there exists $v \in V$ such that $v(\epsilon) = h(\bx(\epsilon))$.
    % there exist $m = (2 e kn)^d$ trees $\bh_1, \dots, \bh_m$ such that for all $\epsilon \in \left\{ \pm 1 \right\}^n$ and all $h \in \cH$, there exists some $i^\star = i(h, \epsilon)$ such that $\bh_{i^\star}(\epsilon) = h(\bx(\epsilon))$.  In other words, 
    % \[
    % \left|\cH(\bx)\right| \leq  \left( \frac{2 e d }{\ldim(\cH)} \right)^{\ldim(\cH)}.
    % \]
    % where $\cH(\bx) = \left\{ (h(\bx_1(\epsilon)), \dots, h(\bx_d(\epsilon))) \mid \epsilon \in \left\{ 0,1 \right\}^d, h \in \cH \right\}$.  
    
    More generally, for $\cH \subseteq \{0,\dots,K\}^{\cX}$ with finite $\sfat_1(\cH)$, it holds for $d \geq \sfat_1(\cH)$ that 
    \begin{align}
        \left|V(\bx)\right| \leq  \left( \frac{e K d }{\sfat_1(\cH)} \right)^{\sfat_1(\cH)}.
    \end{align}
\end{lemma}
Finally, we end by noting an algebraic inequality that we use to get a bound on our desired dimension from a bound on the growth function. For any $a, b>0$, we have 
\begin{equation}\label{eq:algebraic-identity}
    \ln a \leq ab - \ln b -1\,,
\end{equation}
with equality only if $ab=1$\,. See \cite[Inequality (1.2), Appendix 1]{anthony2009neural}. We have the following corollary, which we will use throughout.
\begin{lemma}\label{lem:technical-corollary}
For any $N \in \naturals_{+},M \in \R_{+}$ with $NM \geq 1$, there exists $m \in \naturals_{+}$ such that
$$N \leq m \leq 3 N \log_2 \left( \frac{2NM}{\ln 2} \right) \text{ and } m > N \log_2 (emM) \,.$$
\end{lemma}
\begin{proof} 
Using Eq.~\eqref{eq:algebraic-identity}, with $a=emM$ and $b=\frac{\ln 2}{2eNM}$, we have for every $m\in \naturals_{+}$,
\begin{align*}
     \ln \left( emM\right) &\leq \left(emM\right) \left( \frac{\ln 2}{2eNM}\right)- \ln \left( \frac{\ln 2}{2eNM}\right) -1\,.\\
    A \log_2 \left( emM\right) &\leq  \frac{m}{2} + N \log_2 \left( \frac{2MN}{\ln 2}\right)\,.
\end{align*}
Therefore, in order to ensure $ m > N \log_2 (emM)$, it suffices to have $ \frac{m}{2} > N \log_2 \left( \frac{2NM}{\ln 2}\right) $, which is equivalent to $m>2 A \log_2 \left( \frac{2NM}{\ln 2}\right)$. Finally, noting that $N \log_2 \left( \frac{2NM}{\ln 2}\right) >1$, so there always exists $m \in \naturals_{+}$ such that $N < 2 N \log_2 \left( \frac{2NM}{\ln 2}\right) < m \leq 3 N \log_2 \left( \frac{2NM}{\ln 2}\right),$ concludes the proof.
\end{proof}

\subsection{Generalization Bounds}\label{sec:generalization-bound}
First of all, for finite hypothesis classes, we have the following classical guarantee.
\begin{proposition}[Corollary 2.3 from \citep{shalev2014understanding}]\label{prop:finite-classes-sample-complexity} Consider any domain $\cX$, a label space $\cY$, and a finite hypothesis class $\cH \subseteq \cY^{\cX}$. For any realizable distribution $\cD$ over $\cX \times \cY$ (i.e. $\cD_{\bby \mid \bbx}=h_*(\bbx)$ for some $h_* \in \cH$) over the draw of $S\sim \cD^m$ with $m=m(\eps,\delta)$, we have that with probability at least $1-\delta$, 
$$\cL_{\cD}(\Cons_{\cH}(S)) \leq \eps\,\quad \text{ where } \quad  m(\eps,\delta)\leq 2\left( \frac{\log|\cH|+\log (1/\delta)}{\eps}\right)\,. $$
\end{proposition}

To go beyond cardinality based bounds, we need to consider dimension based guarantees. When $\cY=\{0,1\}$, $\VCdim$ completely characterizes the learnability via the fundamental theorem of statistical learning. For example, see \cite[Theorems 6.7, 6.8]{shalev2014understanding}.
\begin{proposition}[The Fundamental Theorem of Statistical Learning]\label{prop:fundamental-VC} There is a universal constant $c>0$ such that for any domain $\cX$, a hypothesis class $\cH \subseteq \cY^{\cX}$ with $\cY=\{0,1\}$ and $\VCdim(\cH)<\infty$, and any distribution $\cD$ over $\cX \times \cY$, the following holds. With probability at least $1-\delta$ over $S\sim \cD^m$ with $m=m(\eps,\delta)$,
$$ \cL_\cD(\ERM_{\cH}(S)) \leq \inf_{h \in \cH} \cL_{\cD}(h) + \eps\, \quad \text{ and } \quad m(\eps,\delta) \leq c \left( \frac{\VCdim(\cH)+\log(1/\delta)}{\eps^2} \right)\,.$$
Moreover, if $\cD$ is realizable by $\cH$, we have 
$$ \cL_\cD(\Cons_{\cH}(S)) \leq  \eps\, \quad \text{ with } \quad m(\eps,\delta) \leq c \left( \frac{\VCdim(\cH) \log(1/\eps)+\log(1/\delta)}{\eps} \right)\,.$$
Moreover, for any learning rule $A:(\cX \times \cY)^* \to \cY^{\cX}$, if $m < \frac{\VCdim(\cH)}{2}$, then there exists a realizable distribution $\cD$ such that over the draw of $S\sim \cD^m$, we have 
$$\P(\cL_{\cD}(A(S)) \geq\, 1/4 \,) \geq \, 0.8\,.$$
\end{proposition}
Therefore, the VC dimension completely characterizes the learnability when the labels are $\{0,1\}$. We will also be interested in the upper bounds in more generality, i.e. general finite alphabet set for the tokens. The performance of \eqref{eq:ERM} rule and its sample complexity in terms of the Natarajan dimension is given below. See \cite[Theorem 29.3]{shalev2014understanding}. 
\begin{proposition}[The Fundamental Theorem for Multiclass Labels]\label{prop:fundamental-multi-clas} There is a universal constant $c>0$ such that for any domain $\cX$, a finite $\cY$, a hypothesis class $\cH \subseteq \cY^{\cX}$ with $\Ndim(\cH)<\infty$, and any distribution $\cD$ over $\cX \times \cY$, the following holds. Over the draw of $S\sim \cD^m$ with $m=m(\eps,\delta)$ we have
$$ \cL_\cD(\ERM_{\cH}(S)) \leq \inf_{h \in \cH} \cL_{\cD}(h) + \eps\, \quad \text{ and } \quad m(\eps,\delta) \leq c \left( \frac{\Ndim(\cH) \log |\cY|+\log(1/\delta)}{\eps^2} \right)\,.$$
Moreover, if $\cD$ is realizable by $\cH$, i.e. $\cD_{\bby \mid \bbx}=h_*(\bbx)$ for some $h_* \in \cH$, we have 
$$ \cL_\cD(\Cons_{\cH}(S)) \leq  \eps\, \quad \text{ with } \quad m(\eps,\delta) \leq \frac{c}{\eps} \left( \Ndim(\cH) \log\left(\frac{|\cY| \cdot \Ndim(\cH)}{\eps} \right)+\log(1/\delta) \right)\,.$$
\end{proposition}
For our sample complexity results on learning with CoT, we have a more general label space $\cY$, i.e. the entire CoT output. For this, we will rely on the same uniform convergence argument but for the loss class instead.  
% on the general learning setup and the generalization bounds in terms of Rademacher complexity. 
% \paragraph{Rademacher Complexity.}
% Consider any abstract domain $\cZ$ and a function class $\cF \subseteq \R^{\cZ}$. For any fixed subset $S=(\bbz_1, \dots, \bbz_m) \in \cZ^m$, we define the empirical Rademacher complexity
% \begin{equation}\label{eq:empirical-Rad-complexity}
%     \cR_S(\cF):=\E_{\bzeta \sim \Unif(\{\pm 1\}^m)} \left[ \sup_{f \in \cF} \, \frac{1}{m} \sum_{i=1}^{m} \zeta_i f(\bbz_i) \right]\,.
% \end{equation}
% For any distribution $\cD$ over $\cZ$, the distribution specific Rademacher complexity is given by
% \begin{equation}\label{eq:distribution-Rad-complexxity}
%      \cR_{\cD^m}(\cF):= \E_{S\sim \cD^m} \left[ \cR_S(\cF) \right]
% \end{equation}

\paragraph{Generalization Bound in terms of the Loss Class.} We let $\cZ=\cX \times \cY$ and for any function class $\cH \subseteq \cY^{\cX}$, consider the associated loss class $\cL^{\zo}(\cH) \subseteq \{0,1\}^\cZ$ defined by
$$\cL^{\zo}(\cH)=\{ \ell_h: (\bbx,\bby) \mapsto \ind\{h(\bbx) \neq \bby\} \mid h \in \cH \}\,.$$
One can now consider $\VCdim(\cL^{\zo}(\cH))$ (of the loss class) over the domain $\cZ$. Using the uniform convergence argument for \Cref{prop:fundamental-VC} but now for the loss class $\cL^\zo(\cH)$, we have the following guarantee in the realizable case.
% For any is also well-defined. For any $S=\left((\bbx_1,\bby_1),\dots,(\bbx_m,\bby_m) \right) \in \cZ^m$, we define 
% $$\ell_h(S)=\frac{1}{m}\sum_{i=1}^{m}\ell_h(\bbx_i,\bby_i)=\frac{1}{m}\sum_{i=1}^{m} \ind\{h(\bbx_i) \neq \bby_i \}.$$
% We then have the following generalization bound in the realizable case.
\begin{proposition}[General Guarantee]\label{prop:loss-class-gen-bound} 
There is a universal constant $c>0$ such that for any domain $\cX$, a label space $\cY$, a hypothesis class $\cH \subseteq \cY^{\cX}$ with $\VCdim(\cL^\zo(\cH))<\infty$, and any distribution $\cD$ over $\cX \times \cY$ which is realizable by $\cH$. With probability at least $1-\delta$ over $S\sim \cD^m$ with $m=m(\eps,\delta)$,
$$ \cL_\cD(\Cons_{\cH}(S)) \leq  \eps\, \quad \text{ where } \quad m(\eps,\delta) \leq c \left( \frac{\VCdim(\cL^\zo(\cH)) \log(1/\eps)+\log(1/\delta)}{\eps} \right)\,.$$ 
\end{proposition}

% \begin{proposition}[General Guarantee]\label{prop:loss-class-gen-bound} For every distribution $\cD$ over $\cX \times \cY$, over the draw of $S\sim \cD^m$, with probability at least $1-\delta$, we have 
% $$ \sup_{\ell_h \in \cH^{\zo}} \left| \ell_h(S)-\E_{(\bbx,\bby)\sim \cD}[\ell_h(\bbx,\bby)]\right| \leq 2 \cR_{\cD^m}(\cH^{\zo}) + \sqrt{\frac{\ln(4/\delta)}{2m}}\,.$$
% As a direct corollary, with probability at least $1-\delta$ over the draw of $S\sim \cD^m$,
% $$\cL_{\cD}(\hat{h}_S) \leq \inf_{h \in \cH} \cL_{\cD}(h) + 4  \cR_{\cD^m}(\cH^{\zo}) + 2 \sqrt{\frac{\ln(4/\delta)}{2m}} \,.$$
% Therefore, to obtain the generalization bound, it suffices to obtain the bound on $\cR_{\cD^m}(\cH^{\zo})$. This $\cR_{\cD^m}(\cH^{\zo}) \leq  c \sqrt{\frac{\VCdim(\cH^{\zo})}{m}}$ for some constant $c>0$. As a corollary we have the following guarantee.\\
% There exists a constant $c>0$ such that for every distribution $\cD$ over $\cX \times \cY$, with probability at least $1-\delta$, over $S\sim \cD^m$ with $m=m(\eps,\delta)$, we have
% $$\cL_{\cD}(\hat{h}_S) \leq \inf_{h \in \cH} \cL_{\cD}(h) + \eps \quad \text{ for }\quad m(\eps,\delta) \leq c \left(\frac{\VCdim(\cH^{\zo})+\log(1/\delta)}{\eps^2} \right).$$
 
\section{Proofs from Section \ref{sec:samples-computational-complexity-general-F}}\label{app:for Section3}
In this section, we will prove all our results from \Cref{sec:samples-computational-complexity-general-F} except \Cref{thm:lb-Omega(T VC)} (which will be shown in  \Cref{app:sample-complexity-side}). Along with that, in \Cref{app:sample-complexity-side}, we will also show other complementary lower bounds and results mentioned during the discussion in \Cref{sec:samples-computational-complexity-general-F}). 

\begin{proof}[Proof of \Cref{thm:cardinality-based}] This is a corollary of \Cref{prop:finite-classes-sample-complexity} after noting that $\abs{\ete{\cF}{T}} \leq |\cF|$ and $\mcot{T} \leq \mete{T}$. 
\end{proof}

The proof of sample complexity upper bounds, namely, \Cref{thm:ete-VC,thm:cot-VC,thm:ete-ldim} are in \Cref{app:B.1,app:B.2,app:B.3}. The computational complexity result from \Cref{subsec:computational-complexity} is proven in \Cref{app:B.4}.

\subsection{Proof of Theorem \ref{thm:ete-VC} and its extension to non-binary alphabets}\label{app:B.1}
In order to prove \Cref{thm:ete-VC} and its analog for general finite $\Sigma$, we will need to bound the VC dimension or Natarajan dimension of the end-to-end class. 
\begin{theorem}[VC of $\ete{\cF}{T}$ in terms VC of $\cF$]\label{thm:VCete-VCbase} Consider any finite $\Sigma$, a base function class $\cF \subseteq \Sigma^{\Sigma^*}$ and generation length $T \in \naturals_{+}$.
\begin{itemize}
    \item For binary $\Sigma=\{0,1\}$:
    $$\VCdim(\ete{\cF}{T}) \leq 6 \, T  \cdot \VCdim(\cF).$$
    \item For non-binary finite $\Sigma$:
    $$\Ndim(\ete{\cF}{T}) \leq 9\, T  \cdot \Ndim(\cF) \, \log_2 \left( \frac{2\, \Ndim(\cF) |\Sigma|}{e \ln 2} \right)\,.$$
\end{itemize}
\end{theorem}
\begin{proof}[Proof of \Cref{thm:VCete-VCbase}]
    Consider any $\cF\subseteq \Sigma^{\Sigma^*}$. Our goal is to bound the growth function $\Gamma_{T}:=\Gamma_{\ete{\cF}{T}}$ of the class $\ete{\cF}{T}$ for $T \in \naturals_{+}$. To this end, consider any fixed set $S \in (\Sigma^*)^m$ of size $|S|=m$. Let's consider the set $S'$ of all possible ``partial'' extensions of the examples in $S$.
    $$ S'=\{(\bx,\bu): \bx \in S , \bu \in \Sigma^{t}, 0 \leq t \leq (T-1)\}\,.$$
    Clearly, $|S'|\leq m \cdot |\Sigma|^{T}$. For any $f_1,f_2 \in \cF$ such that $\ete{f_1}{T}(S) \neq \ete{f_2}{T}(S)$, we have that $f_1(S') \neq f_2(S')$. Therefore, $\Gamma_T(S) \leq \Gamma_\cF(S') \leq \Gamma_{\cF}(m\, |\Sigma|^{T})$.\footnote{Note that this step in the proof critically requires time-invariance.} The argument holds for any $S$ with $|S|=m$, and thus,
    $$\Gamma_{T}(m)=\max_{|S|=m} \Gamma_T(S) \leq \Gamma_{\cF}(m\, |\Sigma|^{T})\,.$$
    \begin{itemize}
    \item $\Sigma=\{0,1\}$ (binary alphabets): In this case, using Sauer's lemma (\Cref{lem:sauer'slemma}), for any $m \geq \VCdim(\cF)$,
    $$\Gamma_{T}(m) \leq \left( \frac{em 2^{T}}{\VCdim(\cF)} \right)^{\VCdim(\cF)} \,.$$ 
    Therefore, $\VCdim(\ete{\cF}{T})$ can be bounded by any $m\in \naturals_{+}$ satisfying $2^m >  \Gamma_{T}(m)$ and $ m \geq \VCdim(\cF).$ Using the derived upper bound on $\Gamma_T(m)$, it suffices to choose $m$ such that
    $$m > \VCdim(\cF) \, \log_2 \left( \frac{em \,2^{T}}{\VCdim(\cF)}\right) \quad \text{and}\quad m \geq \VCdim(\cF).$$
    Using \Cref{lem:technical-corollary} with $N= \VCdim(\cF)$ and $M=2^{T}/\VCdim(\cF)$, we directly argue the existence of $m \in \naturals_{+}$ such that
    $$ \VCdim(\ete{\cF}{T}) \leq  m \leq 3 \, \VCdim(\cF) \, \log_2 \left(\frac{2 \cdot 2^{T}}{\ln 2}\right) \leq 6\, T \, \VCdim(\cF)\,.$$
    \item For a finite $\Sigma:$ Using Natarajan's lemma (\Cref{lem:Natarajan's lemma}) for any $m \in \naturals_{+} $,
    $$\Gamma_{\ete{\cF}{T}}(m) \leq \Gamma_{\cF}(m\cdot |\Sigma|^{T}) \leq \left(|\Sigma|^2 \cdot m \cdot |\Sigma|^{T}\right)^{\Ndim(\cF)} \,.$$ 
    Again, $\Ndim(\ete{\cF}{T})$ can be bounded by any $m \in \naturals_{+}$ for which $m>  \Ndim(\cF) \log_2 \left(|\Sigma|^{T+2} \cdot m \right).$ 
    Applying \Cref{lem:technical-corollary} with $N=\Ndim(\cF)$ and $M= |\Sigma|^{T+2}/e$, we obtain that there exists such $m\in \naturals_{+}$ such that 
    \begin{align*}
       \Ndim(\ete{\cF}{T}) & \leq m \leq 3 \Ndim(\cF) \log_2 \left( \frac{2 \Ndim(\cF)|\Sigma|^{T+2}}{e \ln 2} \right)\\
       &\leq 9 \,  T\, \Ndim(\cF) \log_2 \left( \frac{2 \Ndim(\cF)|\Sigma|}{e \ln 2} \right) .
    \end{align*}
\end{itemize}
\end{proof}

Once we have the bound on the VC dimension of $\ete{\cF}{T}$ in terms of $\VCdim(\cF)$, using \Cref{prop:fundamental-VC} we have the following proof.
\begin{proof}[Proof of \Cref{thm:ete-VC}] The proof directly follows by combining \Cref{thm:cot-VC} with \Cref{prop:finite-classes-sample-complexity}, after recalling that $\sete$-learnability of $\ete{\cF}{T}$ (\Cref{def:learnable-wo-CoT}) is just standard supervised learning setup from \Cref{app:preliminary} for $\cH=\ete{\cF}{\lin}$. 
\end{proof}
We also have the following corollary for non-binary but finite $\Sigma$.
\begin{corollary}\label{cor:ete-VC-non-binary}
    There exists a universal constant $c>0$ such that for any base class $\cF$ over a finite $\Sigma$ and generation length $T\in \naturals_{+}$, the class $\ete{\cF}{T}$ is $\sete$-learnable using \ref{eq:ete-consistent} with
    $$\mete{T} \leq \frac{c}{\eps} \left(T \cdot \Ndim(\cF) \log\left(\Ndim(\cF)|\Sigma| \right) \log\left(\frac{1}{\eps}\right) + \log\left(\frac{1}{\delta} \right) \right)~.$$
\end{corollary}
\begin{proof}[Proof of \Cref{cor:ete-VC-non-binary}] This directly follows from substituting the bound on $\Ndim(\ete{\cF}{T})$ from \Cref{thm:VCete-VCbase} in \Cref{prop:fundamental-multi-clas}.
\end{proof}
\subsection{Proof of Theorem \ref{thm:ete-ldim} and its extension for non-binary alphabets}\label{app:B.2} We now prove the guarantee based on the Littlestone dimension to bound $\VCdim(\ete{\cF}{T})$.

\begin{theorem}[VCdim of $\ete{\cF}{T}$ in terms of Ldim of $\cF$]\label{thm:VCete-Ldim(Base)}
    Consider any finite $\Sigma$, a base hypothesis class $\cF \subseteq  \Sigma^{\Sigma^*}$ and generation length $T \in \naturals_{+}$. When $\Sigma=\{0,1\}$, we have
\[
\VCdim(\ete{\cF}{T})  \leq 10 \ldim(\cF) \cdot \log(T).
\]

More generally, for any finite $\Sigma$, it holds that
\[
    \Ndim(\ete{\cF}{T})  \leq 20 \sfat_1(\cF) \cdot \log\left(T |\Sigma|  \right) .
\]
\end{theorem}
\begin{proof}[Proof of \Cref{thm:VCete-Ldim(Base)}]
    \underline{\textit{The case when $\Sigma=\{0,1\}$}}: Let $\bx_1, \dots, \bx_m \in \Sigma^*$ be $m$ points.  Consider the tree $\bx$ of depth $M = m(T + 1)$ where for $t \in [M]$,
    \begin{align*}
        \bx_t(\epsilon) = \begin{cases}
            \bx_i & \text{if }t = (i - 1) (T + 1)  \\
            [\bx_i, \epsilon_{t - (i - 1) (T + 1)}, \dots, \epsilon_t] & \text{otherwise.}
        \end{cases}
    \end{align*}
    For every $f\in\cF$ define the path $\epsilon^{f}\in\{0,1\}^{M}$ such that for all $0\le i<m$ and $1\le s\le T$,
    \[
    f(\bx_{(T+1)i+s}(\epsilon^{f}))= \ete{f}{s}(\bx_i).
    \]
    Let $P_{\cF,\bx}=\left\{\epsilon^{f}:f\in\cF\right\}$ be the set of root-to-leaf paths realized by $\cF$ on $\bx$.
    % For $f \in \cF$, let $\epsilon^f \in \{0,1\}^{M}$ be such that for all $0 \leq i < m$ and all $1 \leq s \leq T$, 
    % \[
    % f(\bx_{(T+1)i + s}(\epsilon^f)) = \ete{f}{s}(\bx_i).\] 
    % By virtue of this mapping, we see that $\left|(f(\bx_1), \dots, f(\bx_m)) : f \in \cF\right| \leq \left|\cF(\bx)\right|$.
    Now by definition,
    \[
    \Gamma_T(m) = \Gamma_{\ete{\cF}{T}}(m) \le \max_{\bx}|(f(\bx_1),\dots,f(\bx_m)) : f\in\cF| \le \max_{\bx}|P_{\cF,\bx}|.
    \]  
    By Lemma \ref{thm:seq}, we know that there exists a set of trees $V(\bx)$ such that
    \[
    |V(\bx)|\le\left(\tfrac{2eM}{\ldim(\cF)}\right)^{\ldim(\cF)},
    \text{ and }
    \forall f\in\cF,\epsilon\in\{0,1\}^{M};~
    \exists v\in V(\bx), v(\epsilon)=f(\bx(\epsilon)).
    \]
    We will show that $|P_{\cF,\bx}| \le |V(\bx)|$ by proving that no two distinct paths in $P_{\cF,\bx}$ can belong to the same tree in $V(\bx)$. Let us prove by contradiction. Consider two distinct paths $\epsilon^{f},\epsilon^{g}\in P_{\cF,\bx}$ and let $t$ be the first index where they differ and consider the tree $v$ that covers both. Up to node $t-1$, both share the same path in $v$ however at node $t$ they assign opposite labels. Since the tree can only assign one value to the node $t$, it cannot cover both paths. Hence, $|P_{\cF,\bx}| \le |V(\bx)| \le \left(\tfrac{2eM}{\ldim(\cF)}\right)^{\ldim(\cF)}$.
    % Now, by definition of $\Gamma_T(m) = \Gamma_{\ete{\cF}{T}}(m)$, it holds by Lemma \ref{thm:seq} that for all $m > \ldim(\cF)$,
    % \begin{align}
    %     \Gamma_t(m) \leq \max_{\bx}\left|\cF(\bx)\right|  \leq (e m (T + 1) / \ldim(\cF))^{\ldim(\cF)} \leq (2e m T / \ldim(\cF))^{\ldim(\cF)}.
    % \end{align} 
    
    To conclude, note that $\Gamma_T(m) < 2^m$ if and only if $m \geq \VCdim(\ete{\cF}{T})$.  Thus, it holds that if $m > \ldim(\cF)$ and 
    \begin{align}
        (2e m T / \ldim(\cF))^{\ldim(\cF)} < 2^m,
    \end{align}
    then $m$ upper bounds $\VCdim(\ete{\cF}{T})$.  Setting $m = 10 \ldim(\cF) \cdot \log(T)$ and observing that such an $m$ satisfies the desired bound concludes the proof.
    
    \textit{\underline{General finite $\Sigma$}}:  WLOG we assume that $0, 1 \not \in \Sigma$. Let $k = \lceil \log_2(|\Sigma|) \rceil$ and let $\bx_1, \dots, \bx_m \in \Sigma^*$.  Consider the binary tree $\bx$ of depth $M = m k (T+ 1)$ where for $t \in [M]$,
    \begin{align}
        \bx_t(\epsilon) = \begin{cases}
            \bx_i & \text{if } t = (i - 1) k (T + 1) \\
            [\bx_i, \epsilon_{t - (i - 1) k (T + 1)}, \dots, \epsilon_t] & \text{otherwise.}
        \end{cases}
    \end{align}
    Choose some injection $\rho: \Sigma \to \left\{ 0, 1 \right\}^k$ and for any $f \in \cF$, let $\epsilon^f$ be such that for all $0 \leq i < m$ and all $1 \leq s \leq T$,
    \begin{align}
        \rho \circ f(\bz_{(T+1)i k + sk}(\epsilon^f)) = \epsilon_{(T+1)i k + sk+1:(T+1)ik + (s+1)k}.
    \end{align}
    Let $P_{\cF,\bx}=\left\{\epsilon^{f}:f\in\cF\right\}$ be the set of root-to-leaf paths realized by $\cF$ on $\bx$. By virtue of this mapping, we see that $|\left\{ (f(\bx_1), \dots, f(\bx_m) ) | f \in \cF \right\}| \leq |P_{\cF,\bx}|$. Since the nodes in the constructed $\bx$ are not in $\Sigma^*$ necessarily, to apply Lemma \ref{thm:seq}, we first extend $\cF$ to $\cF'$ to handle these inputs. For each $f \in \cF$, we construct $f':\Sigma'^* \rightarrow \Sigma'$ where $\Sigma' = \Sigma \cup \{0, 1\}$  such that $f'$ matches $f$ on all inputs of the form $\Sigma^*$ and is 0 everywhere else. Since this is a trivial extension of $\cF$, it does not change the complexity of this class. Now applying \ref{thm:seq} on the tree and $\cF'$, we know that there exists a set of trees $V(\bx)$ such that
    \[
    |V(\bx)|\le\left( \frac{e (|\Sigma| + 2) M }{\sfat_1(\cH)} \right)^{\sfat_1(\cH)},
    \text{ and }
    \forall f'\in\cF',\epsilon\in\{0,1\}^{M};
    \exists v\in V(\bz), v(\epsilon)=f'(\bz(\epsilon)).
    \]
    We can use the same argument as before about two paths not sharing the same tree to show that $|P_{\cF,\bx}| \le |V(\bx)|$. Then we have that for any $m > \sfat_1(\cF)$, it holds that
    \begin{align}
        |\ete{\cF}{T}(\bz)| \leq \left( e (|\Sigma|+2) (T + 1)k m / \sfat_1(\cF) \right)^{\sfat_1(\cF)} \leq \left( 4e |\Sigma| Tk m / \sfat_1(\cF) \right)^{\sfat_1(\cF)}. %\lesssim m^{C\fat_1(\cF)\left( \log(T) + \log(K) \right)}.
    \end{align}
    As above, if $2^m > \Gamma_T(m)$ then it must hold that $m > \Ndim(\ete{\cF}{T})$.  Thus if $m > \sfat_1(\cF)$ and
    \begin{align}
        \sfat_1(\cF) \log \left( \frac{4 e |\Sigma|\log(|\Sigma|) T m}{\sfat_1(\cF)} \right) < m,
    \end{align}
    then $\Ndim(\ete{\cF}{T}) \leq m$.  The result follows.
\end{proof}
Our sample complexity bound in \Cref{thm:ete-ldim} follows directly from this bound.
\begin{proof}[Proof of \Cref{thm:ete-ldim}] Substitute the bound of $\VCdim(\ete{\cF}{T})$ from \Cref{thm:VCete-Ldim(Base)} in \Cref{prop:fundamental-VC}.
\end{proof}
We also have the following corollary of \Cref{thm:VCete-Ldim(Base)} for non-binary alphabets.
\begin{corollary}\label{cor:ete-ldim-non-binary}
    There exists a universal constant $c>0$ such that for any base class $\cF$ over a finite $\Sigma$ and generation length $T\in \naturals_{+}$, the class $\ete{\cF}{T}$ is $\sete$-learnable using \ref{eq:ete-consistent} with
    $$\mete{T} \leq \frac{c}{\eps} \left( \sfat_1(\cF) \log\left(T \, |\Sigma| \right) \log\left(\frac{1}{\eps}\right) + \log\left(\frac{1}{\delta} \right) \right)~.$$
\end{corollary}
\begin{proof}[Proof of \Cref{cor:ete-ldim-non-binary}] This follows from substituting the bound on $\Ndim(\ete{\cF}{T})$ from \Cref{thm:VCete-Ldim(Base)} in \Cref{prop:fundamental-multi-clas}.
\end{proof}
\subsection{Proof of Theorem \ref{thm:cot-VC} and its extension for non-binary alphabets}\label{app:B.3}
We now prove the sample complexity of $\sCoT$-learnability given in \Cref{thm:cot-VC}. For every $\bz \in \Sigma^*$, which corresponds to a chain-of-thought of length $T$ on some input $\bx$, we can decompose $\bz=(\bx,\bu)$ where $\bu=\bz[-T:]$ are the last $T$ tokens. Then $\sCoT$-learnability of $\ete{\cF}{T}$ is equivalent to a supervised learning problem (\Cref{app:preliminary}) with the domain $\cX=\Sigma^*$, the label space $\cY=\Sigma^T$ and $\cH=\CoT{\cF}{T}$. The learning rule $\Cons_{\cH}$ in \eqref{eq:general-consistent} is equivalent to the rule \ref{eq:cot-consistent} in this setup. 

Our goal is to apply \Cref{prop:loss-class-gen-bound} on the corresponding loss class. The loss class $\cL^\zo(\cH)$ discussed in \Cref{sec:generalization-bound} exactly corresponds to the loss class from \eqref{eq:loss-class-main} in \Cref{subsec:proof-sk}, also written below.
\begin{equation}\label{eq:loss-class-app}
    \cL(\CoT{\cF}{T}):=\{ \ell_f:\bz \mapsto \ind\{\bz \neq \CoT{f}{T}(\bx)\} \mid f \in \cF\}, \text{ where } \bx=\bz[:-(T+1)] \,.
\end{equation}
Therefore, our goal is to bound the VC dimension of the loss class. 
\begin{theorem}\label{thm:VC(CoT-loss-class)} Consider any finite $\Sigma$, a base class $\cF \subseteq \Sigma^{\Sigma^*}$, and the associated $\cL(\CoT{\cF}{T})$.
\begin{itemize}
    \item For $\Sigma=\{0,1\}$:
    $$ \VCdim(\cL(\CoT{\cF}{T})) \leq 3 \, \VCdim(\cF) \log_2 \left( \frac{2\,T}{\ln 2} \right) \,.$$
    \item For any non-binary finite $\Sigma$:
    $$ \VCdim(\cL(\CoT{\cF}{T})) \leq  3\, \Ndim(\cF) \log_2 \left( \frac{2\, \Ndim(\cF) |\Sigma|^2 T}{e \ln 2} \right) \,.$$
\end{itemize}
\end{theorem}
\begin{proof}[Proof of \Cref{thm:VC(CoT-loss-class)}] The main idea is to bound the growth function of $\Gamma_{\cL(\CoT{\cF}{T})}(m)$ in terms of the growth function of $\Gamma_\cF(m)$, and then use Sauer's Lemma (or Natarajan Lemma for non-binary $\Sigma$). Towards this, consider any $S=\left((\bx_1,\bu_1),\dots,(\bx_m,\bu_m) \right) \in (\Sigma^*\times \Sigma^T)^m$. From this, we can consider the set $\prefix(S)=\left(\bp_{(i,t)}: 1 \leq i \leq m, 0 \leq t \leq T-1\right)$, of all the prefixes where 
\begin{equation}\label{eq:prefixs-labeled-out}
    \bp_{(i,t)}= [\bx_i, \bu_1[\,1\,],\dots,\bu_1[\,t\,]] \in \Sigma^* \times \Sigma^t\,.
\end{equation}
We then have
\begin{align*}
    \Gamma_{\cL(\CoT{\cF}{T})}(S) &= |\cL(\CoT{\cF}{T})(S) |=\{ (\ell_f(\bx_1,\bu_1),\dots,\ell_f(\bx_m,\bu_m)): f \in \cF \} | \tag{By definition}\\
    & \leq \, |\{ \left( f(\bp_{(1,0)}),\dots,f(\bp_{(m,T-1)})  \right) : f \in \cF \}| \nonumber\\
    & = |\cF(\prefix(S))| \leq \, \Gamma_{\cF}(m T) \,.\tag{As $|\prefix(S)|=mT$}
\end{align*}
Here the only inequality followed from this critical observation: for any two $f,g \in \cF$ such that
$$(\ell_{f}(\bx_1,\bu_1),\dots,\ell_{f}(\bx_m,\bu_m)) \neq (\ell_g(\bx_1,\bu_1),\dots,\ell_g(\bx_m,\bu_m)),$$
there exists $ 1 \leq i^* \leq m,\,  0 \leq t^* \leq T-1$ such that $f(\bp_{(i^*,t^*)}) \neq g(\bp_{(i^*,t^*)})$. Therefore, the number of behaviors of the loss class on $S$, can be bounded by the number of behaviors of the function class $\cF$ on the set of all prefixes, i.e. $|\cL(\CoT{\cF}{T})| \leq |\cF(\prefix(S))|$. As the argument holds for any $|S|=m$, we have
$$\Gamma_{\cL(\CoT{\cF}{T})}(m) \leq \Gamma_{\cF}(m \cdot T)\,.$$ 
We now break into each case and bound $\VCdim(\cL(\CoT{\cF}{T}))$.
\begin{itemize}
    \item Binary $\Sigma=\{0,1\}$: In this case, using Sauer's lemma (Lemma \ref{lem:sauer'slemma}), for any $m \geq \VCdim(\cF)$,
    $$\Gamma_{\cL(\CoT{\cF}{T})}(m) \leq \Gamma_{\cF}(m \cdot T) \leq \left( \frac{emT}{\VCdim(\cF)} \right)^{\VCdim(\cF)} \,.$$ 
    Therefore, $\VCdim(\cL(\CoT{\cF}{T}))$ can be bounded by $m$ such that 
    $$m > \VCdim(\cF) \log_2 \left( \frac{emT}{\VCdim(\cF)}\right) \quad \text{and}\quad m \geq \VCdim(\cF).$$ Using Lemma \ref{lem:technical-corollary} with $N=\VCdim(\cF)$ and $M=T/\VCdim(\cF)$, we directly obtain that there exists such $m \in \naturals_{+}$ such that
    $$ \VCdim(\cL(\CoT{\cF}{T})) \leq m \leq 3 \cdot \VCdim(\cF) \log_2 \left(\frac{2T}{\ln 2}\right)\,.$$
    \item For any non-binary finite $\Sigma:$ Using Natarajan's lemma (\Cref{lem:Natarajan's lemma}) for any $m \in \naturals_{+} $,
    $$\Gamma_{\cL(\CoT{\cF}{T})}(m) \leq \Gamma_{\cF}(m \cdot T) \leq \left(|\Sigma|^2 \cdot m\cdot T \right)^{\Ndim(\cF)} \,.$$ 
    Again, $\VCdim(\cL(\CoT{\cF}{T})) \leq m$ for any $m$ s.t. $m>  \Ndim(\cF) \log_2 \left(|\Sigma|^2 \cdot m \cdot T \right).$
    Applying Lemma \ref{lem:technical-corollary} with $N=\Ndim(\cF)$ and $M= \frac{|\Sigma|^2 T}{e}$, we obtain that there exists such $m\in \naturals_{+}$ such that 
    $$ \VCdim(\cL(\CoT{\cF}{T}))\leq m \leq 3 \Ndim(\cF) \log_2 \left( \frac{2\, \Ndim(\cF) |\Sigma|^2 T}{e \ln 2} \right). $$
\end{itemize}
\end{proof}
This bound on $\VCdim(\cL(\CoT{\cF}{T}))$ implies the desired sample complexity result. 
\begin{proof}[Proof of \Cref{thm:cot-VC}] Recall the discussion of the equivalence $\sCoT$-learnability problem with a supervised learning problem (\Cref{sec:generalization-bound}) with $\cH=\CoT{\cF}{T}$. By plugging the bound from \Cref{thm:VC(CoT-loss-class)} on $\VCdim(\cL(\ete{\cF}{T}))$  in \Cref{prop:loss-class-gen-bound} as $\cL(\CoT{\cF}{T})$ corresponds to $\cL^{\zo}(\cH)$, the theorem follows.
\end{proof}
We also have the following non-binary but finite $\Sigma$ corollary for $\sCoT$-learnability. The proof again follows from combining \Cref{thm:VC(CoT-loss-class)} and \Cref{prop:loss-class-gen-bound}.
\begin{corollary}\label{cor:cot-VC-non-binary} There exists a universal constant $c>0$ such that for any base class $\cF \subseteq \Sigma^{\Sigma^*}$ over a finite $\Sigma$ and generation length $T\in \naturals_{+}$, the class $\ete{\cF}{T}$ is $\sCoT$-learnable using \ref{eq:cot-consistent} with 
    $$\mcot{T} \leq \frac{c}{\eps} \left(\Ndim(\cF) \log\left(T\,|\Sigma|\, \Ndim(\cF)\right) \log\left(\frac{1}{\eps}\right) + \log\left(\frac{1}{\delta} \right) \right)\,.$$
\end{corollary}
\subsection{Computational Complexity}\label{app:B.4}
We finally show that a tractable $\Cons_{\cF}$ oracle implies a computationally tractable $\sCoT$ learnability of $\ete{\cF}{T}$ via a chain-of-thought generated by $\cF$. We first start with the proof of \Cref{thm:ConsF-gives-CoT-Cons-implementation} which reduces $\sCoT$-learnability to a consistency problem on the base class $\cF$.
\begin{proof}[Proof of \Cref{thm:ConsF-gives-CoT-Cons-implementation}] For any CoT dataset $S_{\sCoT}=(\bz_1,\dots,\bz_m)\in (\Sigma^*\times \Sigma^T)^m$, one can first create the dataset of a prefix and the associated next-token pairs for all the examples. Formally, we decompose $\bz_i=(\bx_i,\bu_i)$ again and consider a dataset $\Tilde{S}=((\bp_{(i,t)},y_{i,t}): i \in [m], 0 \leq t \leq (T-1))$ of (prefix,next-token) pairs where
\begin{equation*}\label{eq:TildeS-from-S}
    \bp_{(i,t)}= [\bx_i, \bu_i[\,1\,],\dots,\bu_i[\,t\,]] \in \Sigma^* \times \Sigma^t\,, \quad y_{(i,t)}=\bu_{i}[t+1]\in \Sigma\,, \quad \text{ for } i \in [m], 0 \leq i \leq (T-1)\,.
\end{equation*}
It is easy to observe that if we find a consistent $\hat{f}\in \cF$ on $\Tilde{S}$ (in the sense of \ref{eq:cons-f}), then we also have that this predictor is consistent with the entire chain-of-thought, i.e. $\CoT{\hat{f}}{T}(\bx_i)=\bz_i$ for all $i \in [m]$. If the original inputs $\bx_i$'s are of length at most $n$, then the prefixes $p_{(i,t)}$ are of length at most $n+T$. Moreover, we have $|\Tilde{S}|=\Tilde{m} \leq m \cdot T$ and thus, one such call of \ref{eq:cons-f} on $\Tilde{S}$ suffices. The initial input processing of creating $\Tilde{S}$ from $S$, and also finally returning $\ete{\hat{f}}{T}$ can be done in additional time $O(\Tilde{m})$, concluding the implementation of \ref{eq:cot-consistent}.
\end{proof}

\begin{proof}[Proof of \Cref{cor:tract-Cons-oracle-tract-CoT-lrn}] This is a direct corollary of \Cref{thm:ConsF-gives-CoT-Cons-implementation,thm:cot-VC}. If $\VCdim(\cF_d)=\poly(d)$, then implementing \ref{eq:cot-consistent} on a CoT training set $S_{\sCoT}$ of size $$|S_{\sCoT}|=O(\eps^{\scriptscriptstyle{-1}} (\VCdim(\cF_d)\log T \log \eps^{\scriptscriptstyle{-1}}+\log \delta^{\scriptscriptstyle{-1}}))=\poly(d,T,\eps^{\scriptscriptstyle{-1}},\delta^{\scriptscriptstyle{-1}})$$ 
suffices for $\sCoT$-learnability of $\ete{\cF_d}{T}$ by \Cref{thm:cot-VC}. Moreover, by \Cref{thm:ConsF-gives-CoT-Cons-implementation}, we can implement this by calling $\Cons_{\cF_d}$ on a sample set $\Tilde{S}$ of size at most $|\Tilde{S}|=O(T \cdot |S_{\sCoT}|)=\poly(d,T,\eps^{\scriptscriptstyle{-1}},\log \delta^{\scriptscriptstyle{-1}})$. Finally, for input distributions that are supported on sequences of length at most $n$, the description length of $S_{\sCoT}$ is $\poly(n+T,d,|S_{\sCoT}|)$. Here we are subsuming the $\log |\Sigma|$ dependence to be already captured in the size parameter as it is inherent to the base class $\cF_d$. Thus even the description length of the set $\Tilde{S}$ is $\poly(n+T,d,|\Tilde{S}|)=\poly(n,d,T,\eps^{\scriptscriptstyle{-1}}, \log \delta^{\scriptscriptstyle{-1}})$. As such, as long as $\Cons_{\cF_d}$ is implementable in time polynomial in its input and the size parameter $d$, we have that $\ete{\cF}{T}$ is $\sCoT$-learnable in time $\poly(n,d,T,\eps^{\scriptscriptstyle{-1}}, \log \delta^{\scriptscriptstyle{-1}})$.
\end{proof}

\section{Proofs from Section \ref{sec:linear}}\label{app:proofs-for-linear}
We first show the sample complexity results for $\sete$ and $\sCoT$ learnability.

\begin{proof}[Proof of \Cref{lem:cardinality-of-halspaces}] As noted already, this follows by Sauer's lemma. We know that $\VCdim(\cFlin)$ is at most $(d+1)$. Therefore, the cardinality of $\cFlin$ on the hypercube of $\{0,1\}^d$ of size $2^d$ is:
    $$|\cFlin| \leq \Gamma_{\cFlin}(2^d) \leq (e 2^d)^{\VCdim(\cFlin)} \leq 2^{O(d^2)} \,.$$ 
\end{proof}
\begin{proof}[Proof of \Cref{cor:linear-th-VCbound}]
    The bound simply follows from \Cref{thm:cardinality-based,thm:ete-VC,thm:cot-VC} after substituting $\VCdim(\cFlin)\leq (d+1)$ and $\log |\cFlin|=O(d^2)$ by \Cref{lem:cardinality-of-halspaces}.  
\end{proof}
We now show the computational tractability of the $\sCoT$-learnability of $\ete{\cFlin}{T}$. 
\begin{proof}[Proof of \Cref{thm:easiness-with-CoT}] This directly follows from the facts (i) $\Cons_{\cFlin}(S)$ is implementable in time polynomial in $d$ and the length of the input $S$ (ii) noting that $\VCdim(\cFlin) \leq (d+1)$. This is a corollary of \Cref{cor:tract-Cons-oracle-tract-CoT-lrn}.

More specifically, we are implementing \ref{eq:cot-consistent} by forming the following LP feasibility problem. For a given $S_{\sCoT}$, consider the set of all prefixes and next token pairs mentioned in Eq.~\eqref{eq:TildeS-from-S}:
\begin{equation*}
    \bp_{(i,t)}= [\bx_i, \bu_i[\,1\,],\dots,\bu_i[\,t\,]] \in \Sigma^* \times \Sigma^t\,, \quad y_{(i,t)}=\bu_{i}[t+1]\in \Sigma\,, \quad \text{ for } i \in [m], 0 \leq i \leq (T-1)\,.
\end{equation*}
% Also, consider the set of next tokens $\bby_{(i,t)}:=\bu_i[(t+1)]$ for each of those prefixes from $S_{\sCoT}$. 
Then we have to solve the following linear program in a variable parameter $\bw \in \R^d$ with the following constraints.
\noindent For $1 \leq i \leq m, 0 \leq t \leq (T-1)$
\begin{align}
    \<\bw,\bp_{(i,t)}\> \geq \, & 0, \quad \text{if $\bby_{(i,t)}=1$;} \nonumber\\ 
    \<\bw,\bp_{(i,t)}\> < \, & 0, \quad \text{if $\bby_{(i,t)}=0$.} \tag{LP-Feasibility}\label{eq:linear-p}
\end{align}\label{eq:linear-program}
  Clearly there are $|S_{\sCoT}| \cdot T$ many constraints, and we are implementing \ref{eq:cot-consistent}, and therefore choosing $|S_{\sCoT}|=\poly(n,d,T,1/\eps, \log(1/\delta))$ suffices by \Cref{cor:linear-th-VCbound}. The runtime complexity of solving \eqref{eq:linear-p} with $d$ variables and $|S_{\sCoT}| \times T$ constraints is bounded by $\poly(n,d,T,1/\eps, \log(1/\delta))$; the theorem follows.
\end{proof}
\subsection{End-to-End Learning is Hard for Iterated Linear Thresholds}\label{app:ete-of-LT-hard}
We now return to the computational intractability of $\sete$-learning of $\ete{\cFlin}{T}$; the main technical result of this section. We begin with the description of the class of threshold circuits that we will be using in our reduction.
\paragraph{Bounded Depth and Size Linear Threshold Circuit.} Consider the input $\bx\in \{0,1\}^n$. A linear threshold circuit $C:\{0,1\}^n\to \{0,1\}$ is a computational model represented by a connected directed acyclic graph (DAG), $G(V,E)$ with 
\begin{itemize}
    \item $n$ input nodes $V_{\text{in}}=\{v_{\text{in},1},\dots, v_{\text{in},n}\}$ (the only nodes in $V$ with no incoming edges) that correspond to the input coordinates $(x_1,\dots,x_n)$,
    \item One output node $v_{\text{out}}$ (the only node with no outgoing edges),
    \item A weight function $w:E \to \R$.
\end{itemize}
Each node has a value $\val: V \to \{0,1\}$ associated with it computed recursively as follows:
\begin{equation}\label{eq:value-of-bounded-depth-circuit}
    \val(v_{\text{in},k})=x_k, \text{ for  } k \in [n], \text{ and for any } v \in V \setminus V_{\text{in}}, \quad \val(v)=\th\left(\sum_{(u,v) \in E} w{(u,v)} \val(u) \right)
\end{equation}
The final output of the circuit is $C(\bx)=\val(v_{\text{out}})$. The \textit{depth} of the circuit is the length of the longest path from an input node in $V_{\text{in}}$ to the output node $v_{\text{out}}$. The \textit{size} of the circuit is $|V \setminus V_{\text{in}}|$.

For $n,L,s\in \naturals_{+}$, consider the hypothesis class threshold circuits over $n$ input variables, size $s$, and depth at most $L$. This class is referred in \Cref{hassum:c-depth-circuit}.
We first prove the final hardness results in light of \Cref{lem:LT-expressivity}.
\begin{proof}[Proof of \Cref{thm:hardness-of-ete-linear}]
 Consider the class threshold circuits of depth $L$ and size $p(n)$ from Hardness \Cref{hassum:c-depth-circuit}. By our expressivity result \Cref{lem:LT-expressivity}, every such circuit $C$ can be expressed as $\ete{f_{\bw}}{T}$ with some $\bw \in \R^d$ with $d \leq 2(p(n)+2)^L(n+1)$ and $T \leq (p(n)+2)^L(n+1)$, up to some fixed input transformation that runs in time $O(n')=O\left(2(p(n)+2)^L(n+1) \right)$. All $d,n',T$ are bounded by fixed polynomials in $n$. Therefore, if $\cF_{\scriptscriptstyle{d,\lin}}$ is learnable in time $\poly(n,T,d)$ up to even constants $\eps,\delta>0$, we can also learn the class from \Cref{hassum:c-depth-circuit} in time $\poly(n)$ through this reduction, contradicting \Cref{hassum:c-depth-circuit}.
\end{proof}
%Define 
%$d_i=\sum_{j=1}^{i-1} |\Tilde{\bw}_{j}|$
%The length of $|\Tilde{\bw}_i|=n+j-1+\sum_{k=1}^{j-1}d_k$
We now prove our main expressivity result of this section.
\begin{proof}[Proof of \Cref{lem:LT-expressivity}]
   We begin our proof with some notation. In the graph representation on any circuit $C$ (defined in Section \ref{app:ete-of-LT-hard}), we can partition the internal nodes in $V \setminus V_{\text{in}}$ into different layers $V_{1} \cup V_2 \cup \dots \cup V_L$. The layer number of any internal node $v$ is defined as the length of the longest path from some input node in $V_{\text{in}}$ to $v$. We now number the nodes arbitrarily per layer: the nodes in $V_l$ can be numbered $v^{(l)}_1, \dots, v^{(l)}_s$, for $1 \leq l \leq L$. Without loss of generality, we may assume that $|V_l| =s$ for $1 \leq l \leq L$ as we can always add nodes to the layer with incoming and outgoing edges having weights zero. And, the final output node $v_{\text{out}}=v_s^{(L)} \in V_L$. 
    %\gal{did you mean $v_s^{(L)}$?} 
    Also, w.l.o.g., we consider any node $v_i^{(l)}$ in $V_l$ is connected with all the nodes $V_{\text{in}} \cup V_1 \cup \dots \cup V_{l-1}$, as we can add the edges with weight 0, without affecting the output of the circuit. As a consequence, we can denote all the weights associated with incoming edges as a vector $\bw_{li} \in \R^{p_{l}}$ where $p_{l}=n+(l-1)s$ and $(l,i) \in [L]\times [s]$. The weight coming from $j^\mathrm{th}$ node from the $\ell^{\mathrm{th}}$ layer in $\bw_{li}$ is denoted by $\bw_{li}(\ell,j)\in \R$. Throughout the proof, rather than explicitly defining the coordinate of vectors, we will defined it by listing out its coordinates from left to right. 
    
    Without loss of generality, we will assume that the incoming weight associated to the last node of the previous layer is always 0. Formally, we have $\bw_{1i}(0,n)=0$ and for $l\geq 2$, we have $\bw_{li}(l-1,s)=0$. This can be easily ensured by adding one dummy node per layer in the circuit, and thus, by replacing $n$ with $n+1$ and $s$ with $(s+1)$ in the final bounds on $T,n',d$, which we will do towards the end of the proof. For now, we will proceed assuming this as it simplifies the presentation of our construction. 
    
     We now specify our feature map $\phi:\{0,1\}^n \to \{0,1\}^{n'}$. For any $\bx\in \{0,1\}^n$, the feature $\phi(\bx) \in \{0,1\}^{n'}$ written as a string takes the following form 
    \begin{equation}\label{eq:feature-map}
       \phi(\bx)= (1\underbrace{00\dots \dots00}_{(T-1) \text{ bits}}\underbrace{\bx}_{n \text{ bits}})
    \end{equation}
    where $T$ is the number of end-to-end steps to be taken which will be specified later in terms $(n,s,L)$. Roughly speaking, we want to construct a linear predictor whose chain-of-thoughts on $\phi(\bx)$ when applied iteratively, produces the output of each internal node sequentially (separated by some number of zeros), i.e. 
    $$\phi(\bx),0, \dots, 0,y_{11}, 0\dots 0,y_{12},0,\dots\dots, 0, y_{L1}\,.$$
    The idea is to 
    %concatenate the linear threshold combining 
    combine
    all the weights from the circuit into one linear threshold. However, to ensure that the weights associated with the previous output do not contribute when outputting the other node, we need that the 0s are outputted between any two relevant bits in CoT, so that the weights of the previous output align with 0s exactly. This means that the weight vector of the next layer must be now expanded into a higher dimension to adjust for the padded 0s.

    Formally, we will embed each $\bw_{li}\in \R^{p_l}$ into a vector $\Tilde{\bw}_{li} \in \R^{\Tilde{p}_l}$ by padding extra zeros between the coordinates of $\bw_{li}$ (to be specified how). The new dimension $\Tilde{p}_l \geq p_l$ defined recursively:
    \begin{equation}\label{eq:recursive-seq}
        \Tilde{p}_1=p_1=n , \quad \text{ and for $l>1$, define } \Tilde{p}_l=(s+1)\Tilde{p}_{l-1}\,.
    \end{equation}
    We will create a vector $\bv_0 = \bzero^{n-1}$ and for $l \geq 1$, we will form a vector $\bv_l$, which is a concatenation of the vectors $\Tilde{\bw}_{l1}, \dots, \Tilde{\bw}_{ls} $ in the reverse order:
    \begin{equation}\label{eq:bv-l-description}
        \bv_{l}=[\Tilde{\bw}_{ls},\dots, \Tilde{\bw}_{l1}], \quad \text{ and therefore } |\bv_l|=s \, \Tilde{p}_l\,.
    \end{equation}
Here is a simple claim that specifies the size of each of these linear predictors $\bv_l$.  
    \begin{claim}\label{clm:recursive-sequence}
        For any $l\geq 1$, we have $\Tilde{p}_l = (s+1)^{l-1} n$ and $\sum_{\ell=0}^{l-1} |\bv_\ell|=(s+1)^{l-1} \cdot n-1= \Tilde{p}_{l}-1\,.$ 
    \end{claim}
So roughly speaking, the size of each predictor $\Tilde{\bw}_{li} \in \R^{\Tilde{p}_l}$ in layer $l$ that we embedded our $\bw_{li}\in \R^{p_l}$ into, is one more than the length of the entire concatenation of vectors $\bv_{l-1},\dots, \bv_0$. And the vector $\bv_l$ is a concatenation of $s$ such vectors, giving us $\Tilde{p}_l = O(s^l)$. See \Cref{fig:hardness}.

% \begin{figure}
%     \includegraphics[scale=0.8]{Figures/hardness-proof.pdf}
%     % \label{fig:hardness}
% \end{figure}

\begin{figure}[t]
    \centering
    \fbox{\includegraphics[scale=0.8]{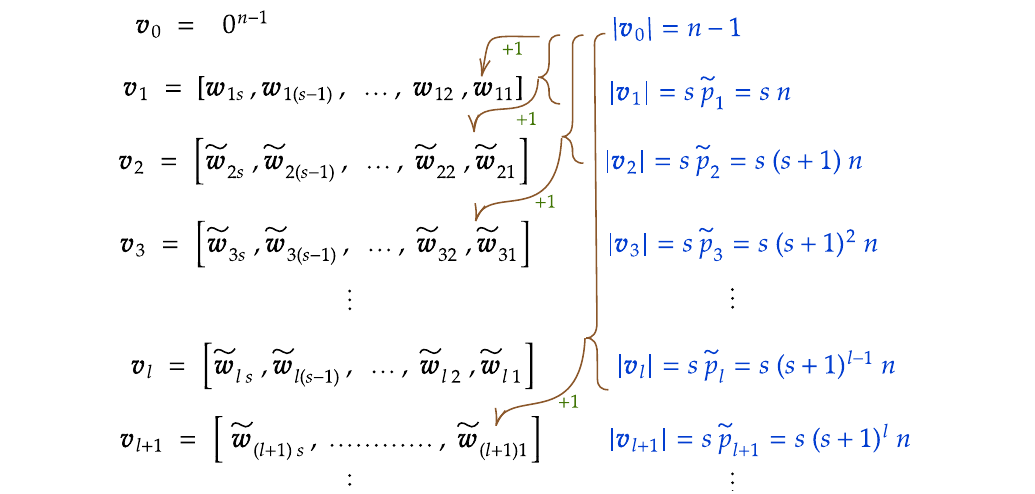}}
   \caption{The figure provides an illustrative summary of the construction so far and the sizes of the predictors, formalized in \Cref{clm:recursive-sequence}.}
     \label{fig:hardness}
\end{figure}

It is now time to specify $T$ and the final linear predictor $\bw$. The number of end-to-end steps $T$ is given by
\begin{equation}\label{eq:bound-on-T}
    T :=\sum_{\ell=1}^{L} |\bv_\ell| = \Tilde{p}_{L+1}-1-|\bv_0|= (s+1)^L \cdot n-n \,,
\end{equation}
where we used Claim \ref{clm:recursive-sequence} in the second equality. Eq.~\eqref{eq:feature-map} gives us
\begin{equation}\label{eq:bound-on-m}
    n'=|\phi(\bx)|=T+n = n+\sum_{\ell=1}^{L} |\bv_\ell|=(s+1)^L \cdot n\,.
\end{equation}
 Having constructed $\bv_0,\bv_1, \dots, \bv_L$, the final predictor is then the concatenation of all these predictors along with another $\bv \in \R^T$ (to be specified later) as follows
 \begin{equation}\label{eq:bw-the-final-predictor}
     \bw=[\bv,\bv_L,\bv_{L-1},\dots,\bv_2,\bv_1, \bv_0].
 \end{equation}
Hence, at the start the linear predictor $\bw$ and $\phi(\bx)$ aligns as follows.
\begin{align}
\bw &= [\text{------} \, \bv \, \text{------},  \underbrace{\bv_L, \bv_{L-1}, \dots, \bv_2, \bv_1}_{T \text{ coordinates}}\,\,,0^{n-1} ] \nonumber\\
\phi(\bx) &= \quad \quad \quad \quad \hspace{2pt} [1 ,  \underbrace{0,0,0,\dots \dots,0,0,0}_{(T-1) \text{ times}}, \underbrace{\text{--}\,\bx \text{---}\,}_{n \text{ bits}}]& \label{eq:intial-alignment}
\end{align}
As $\bv \in \R^T$, we have 
\begin{equation}\label{eq:bound-on-d}
    d:=|\bw|=T+\sum_{\ell=0}^{L} |\bv_\ell|\leq 2 \cdot (s+1)^L \cdot n\,,
\end{equation}
where we used Claim \ref{clm:recursive-sequence} and Eq.~\eqref{eq:bound-on-T}. Overall, our constructions so far satisfies the bounds on $T,n',d$ according to Lemma \ref{lem:LT-expressivity}. Therefore, it remains to show the final construction of $\Tilde{\bw}_{li} \in \R^{\Tilde{p}_l}$ and $\bv \in \R^T$, and that this construction achieves the desired end-to-end computation of the circuit. To this end, we first specify the set of time-steps at which we output a gate in the circuit. Let $t_{li}$ for $l \in [L]$ and $i\in [s]$ be the time when we are going to be writing $y_{li}$ (i.e. output of the $i^\mathrm{th}$ gate from the layer ${l}$). We define these time-steps recursively as 
\begin{align}
    t_{1i}=n \cdot i, \text{ for $i\in [s]$ } \quad \text{and}\quad t_{li}=t_{(l-1)s} + i \cdot \Tilde{p}_l \;\;\text{ for $l>1, i\in [s]$ }\,. \label{eq:t_lis}
%     \cI_1&:=\{(1+n),(1+2n),\dots,(1+ns)=:t_1\}\,;\\
% \cI_2&:=\{(t_1+\Tilde{p}_2),(t_1+2\Tilde{p}_2),\dots,(t_1+s \cdot \Tilde{p}_2)=:t_2 \}\,.
\end{align}
Let $\cI:=\{t_{li}: l \in [L], i\in [s]\}$ be the collection of these indices. It is straight-forward to verify that $t_{Ls}=\sum_{\ell=1}^{L} s \cdot  \Tilde{p}_\ell=\sum_{\ell=1}^{L} |\bv_\ell|=T\,$ by Eq. \eqref{eq:bound-on-T}. Therefore, the final output of the circuit $y_{Ls}$ will be computed at the $T^\mathrm{th}$ step (to be shown).

We would want that the iterated predictor on the time steps $t \in [T]\setminus \cI$ when applied iteratively outputs zero. By the construction of $\cI$, this corresponds to outputting exactly $\Tilde{p}_l-1$ zeros before outputting $y_{li}$ for any $l \in [L], i \in [s]$. So the desired chain-of-thought is of the following form
\begin{equation}\label{eq:desired-CoT}
    \phi(\bx), \underbrace{0 \dots 0}_{(n-1) \text{ bits }},y_{11}, \dots \dots, y_{l(i-1)},\underbrace{0 \dots 0}_{\Tilde{p}_l-1 \text{ bits }},y_{li}, \dots \dots \dots \dots,y_{L(s-1)}, \underbrace{0 \dots 0}_{\Tilde{p}_L-1 \text{ bits }} y_{Ls}\,.
\end{equation}
%\gal{The $y_{(L-1)(s-1)}$ above should be $y_{L(s-1)}$, right?}\nirmit{yes}
This is exactly where the $1$ at the beginning of the feature map $\phi$ in Eq.~\eqref{eq:feature-map} plays the role of ``positional encoding'' together with the vector $\bv\in \R^T$. The main idea is to create a vector $\bv \in \{-B,0\}^T$ where $-B$ is a ``very large'' negative value to be specified such that for every $t \in [T]\setminus \cI$, the $1$ in $\phi(\bx)$ aligns with $-B$ in $\bv$, which on thresholding gives us the output 0. Formally,
\begin{equation}\label{eq:v[-t]}
   \bv[-t]:= \begin{cases}
        -B \, &, \text{if } t\in [T]\setminus\cI\,;\\
        \,\,0\, &, \text{if } t \in \cI \,.\\
    \end{cases}
\end{equation}
Here $B=1+\sum_{l=1}^{L}\sum_{i=1}^{s} \norm{\bw_{li}}_1$ is a real number greater than the total $\ell_1$ norm of the weights in the circuit. We have the following claim due to this construction. 
%\gal{We shouldn't state this claim before defining the $\tilde{\bw}$'s because $\bw$ is not defined yet, and the claim depends on the construction of the $\tilde{\bw}$'s (namely, on the fact that they only contain weights from the circuit).}\nirmit{I agree, and I would change the claim as written. For this proof, I really don't care about the order, but rather breaking it so that it is followable.}

% \begin{claim}\label{clm:0-output-in-CoT}
%     For every $t \in [T]\setminus \cI$, we have $\ete{f_{\bw}}{t}(\phi(\bx))=0\,.$
%  \end{claim}
 \begin{claim}\label{clm:0-output-in-CoT}
    Consider any $\bw$ of the form Eq.~\eqref{eq:bw-the-final-predictor}, where $\Tilde{\bw}_{li}$ are just constructed by padding zeros in $\bw_{li}$ for $l\in [L], i \in [s]$, i.e., without increasing the total $\ell_1$ norm of the weights. Then 
    $$ \text{ for every } t \in [T]\setminus \cI\,, \text{ we have }\; \ete{f_{\bw}}{t}(\phi(\bx))=0\,.$$
 \end{claim}
Finally, it is time to specify the constructions of $\Tilde{\bw}_{li} \in \R^{\Tilde{p}_{l}}$. 
%\gal{should be $\R^{\Tilde{p}_{l}}$?} 
Note that the goal of $\Tilde{\bw}_{li}$ is to compute $y_{li}$ in our CoT. For this, it needs to compute the threshold of the inner product $\<\bw_{li}, [\bx, y_{11},y_{12},\dots,y_{(l-1)(s-1)} y_{(l-1)s}]\>$, where $\bw_{li} \in \R^{p_l}$ is the weight vector from the circuit
 $$\bw_{li}=[\bw_{li}(0,1),\bw_{li}(0,2),\dots,\bw_{li}(l-1,s-1),\bw_{li}(l-1,s)],$$
where we denote the scalar weight coming from the $j^{\mathrm{th}}$ node in the layer $\ell \leq l-1$ by $\bw_{li}(\ell,j)$. However, note that by Claim \ref{clm:0-output-in-CoT} and Eq.~\eqref{eq:desired-CoT}, we have to now pad zeros so that the $\bw_{li}$ aligns with the desired output gates. Formally, construct $\Tilde{\bw}_{li}$ by padding $\Tilde{p}_{\ell}-1$ zeros before $\bw_{li}(\ell,j)$ for any $i \in [s]$. 
\begin{equation}\label{eq:tilde-bw-view}
   \Tilde{\bw}_{li}=[ \bw_{li}[(0,1)\rightarrow(0,n)], \underbrace{0 \dots 0}_{(n-1) \text{ bits }},\bw_{li}(1,1), \dots ,\bw_{li}(\ell,i-1),\underbrace{0 \dots 0}_{\Tilde{p}_{\ell}-1 \text{ bits }},\bw_{li}(\ell,i),\dots \dots, \bw_{li}(l-1,s) ]
 \end{equation}
Clearly, we are only padding zeros ensuring that the total $\ell_1$ norm $\sum_{l=0}^{L}\norm{\bv_l}_1 = \sum_{l=1}^{L} \sum_{i=1}^{s} \norm{\bw_{l i}}_1$ is the total $\ell_1$ norm of the circuit as needed. Formally, the following claim finalizes that this construction outputs $y_{li}$ after $t_{li}$ end-to-end steps for all $t_{li} \in \cI$.
\begin{claim}\label{clm:Tilde{bw}-size-and-the-final-prediction}
   For every $l \in [L], i \in [s]$, we indeed have $|\Tilde{\bw}_{li}| = \Tilde{p}_l$, and $\ete{f_\bw}{t_{li}}(\phi(\bx))=y_{li}$ for $t_{li} \in \cI$\,.
\end{claim}
Invoking this claim at $t_{Ls}=T$ directly gives us that $\ete{f_\bw}{T}(\phi(\bx))=y_{Ls} = C(\bx)$ is the output of the circuit. Finally, noting that bounds on $T,d,n'$ by Eqs.~\eqref{eq:bound-on-T},\eqref{eq:bound-on-d}, and \eqref{eq:bound-on-m} even after replacing $(n,s)$ with $(n+1,s+1)$ is according to the lemma statement. This removes the posed restriction on the circuit that the weight coming from the last gate of the previous layer is zero, concluding the proof.
\end{proof}

We now return to the deferred proofs of the claims in order.
    \begin{proof}[Proof of \Cref{clm:recursive-sequence}]
    Solving the recurrence in Eq.\eqref{eq:recursive-seq} immediately gives us $\Tilde{p}_l=(s+1)^{l-1} \cdot n$. The second part will be shown inductively. Observe that $\sum_{\ell=0}^{0} |\bv_\ell|=|\bv_0|=n-1=\Tilde{p}_1-1$. For any $l>1$, inductively
    \begin{equation*}
        \sum_{\ell=0}^{l-1}|\bv_\ell|=|\bv_{l-1}|+\sum_{\ell=0}^{l-2} |\bv_\ell|=s\, \Tilde{p}_{l-1} + \Tilde{p}_{l-1}-1= (s+1)\Tilde{p}_{l-1}-1=\Tilde{p}_l-1\,.
    \end{equation*}
    \end{proof}

\begin{proof}[Proof of \Cref{clm:0-output-in-CoT}] We first note that during the output of the first $T$ predictions, our linear predictor has only shifted by at most $T-1$ positions to the right. At the beginning, the right-most coordinate $\bv[-1]$ is aligned with $1$ to the left-end of $\phi(\bx)$ according to the construction Eq.~\eqref{eq:intial-alignment}. As there are $T-1$ additional $0$'s padded before $\bx$, the weight $\bv[-1]$ never gets aligned with any coordinates in $\bx$ (or the new ones to be added to its right) in the first $T$ predictions. 

In summary, letting $\bz^{t}=\CoT{f_{\bw}}{t}(\phi(\bx))$, for any $1 \leq t \leq T-1$, we can say  
\begin{equation}\label{eq:inner-product-except-1}
    \<\bw[-n-t \, :], \bz^t[-n-t\,:]\> \leq \sum_{l=1}^{L}\sum_{i=1}^s \norm{\bw_{ls}}_1\,
\end{equation}
where we used the fact that only weights from $[\bv_1,\dots,\bv_L]$ align with $[\bx,\bz^t[-t :]]$ during the first $T$ iterative steps. And these vectors are only constructed from $\{\bw_{li}: l \in [L], i \in [s]\}$ by padding zeros, and $\bz^t$ has binary coordinates. Hence, the inner-product is bounded by the total $\ell_1$ norm of the circuit.

Finally, during the $t^\mathrm{th}$ iterative prediction for any $ t \in [T]\setminus\cI$, we have that the coordinate of $\bv$ that aligns with $1$ to the left in $\phi(\bx)$ is $\bv[-t]$; this is because $\bv[-1]$ is aligned at the start and the predictor moved $t-1$ steps to the right. Therefore, using Eq.~\eqref{eq:inner-product-except-1}, for any $t \in [T] \setminus\cI$, 
$$\<\bw,\bz^{t-1}[-d:]\> \leq \bv[-t]+ \sum_{l=1}^L \sum_{i=1}^s \norm{\bw_{li}}_1 \leq - 1 \,,$$
where, in the last inequality, we used $\bv[-t]=-B$ from Eq.~\eqref{eq:v[-t]} and the value of $B$. Thresholding immediately gives us $\ete{f_{\bw}}{t}(\phi(\bx))=\th(\<\bw,\bz^{t-1}[-d:]\> )=0$, as desired.
\end{proof}
\begin{proof}[Proof of \Cref{clm:Tilde{bw}-size-and-the-final-prediction}] We first start by inductively (on $l$) verifying that $\Tilde{\bw}_{li}=\Tilde{p}_l$ for $l \in [L], i \in [s]$. Clearly, by our description before \eqref{eq:tilde-bw-view}, we don't insert any zeros in $\bw_{1i}$ for $1 \leq i \leq s$, and therefore $|\Tilde{\bw}_{1i}|=p_1=n=\Tilde{p}_1$. Now for any $l>1$, on the top of weights from the first $(l-2)$ layers, there are also $s$ number of weights associated with $(l-1)^\mathrm{th}$ layer in $\bw_{li}$. Before all of these weights we have padded $\Tilde{p}_{l-1}-1$ zeros. Therefore, using this and the induction hypothesis
$$|\Tilde{\bw}_{li}|=\Tilde{p}_{l-1}+s (\Tilde{p}_{l-1}-1)+s=(s+1) \cdot \Tilde{p}_{l-1}=\Tilde{p}_{l}\,.$$
Finally, we return to the most important claim that for any $t_{li}\in \cI$, we have $\ete{f_{\bw}}{t_{li}}(\phi(\bx))=y_{li}$. First of all, by our construction of the vector $\bv$ in Eq.~\eqref{eq:v[-t]}, for any of these time values, the coordinate of $\bv$ that aligns with the left-most $1$ in $\phi(\bx)$ has the value simply zero. This allows us to focus on the coordinates starting from $\bx$ and to its right from newly added CoT. We will show this by induction on the indices from $\cI$ in their increasing order.

The base case is $t_{11}=n$. That means $n-1$ steps have already elapsed. By Claim \ref{clm:0-output-in-CoT}, we only output $0$ in them. The linear predictor also moved $n-1$ positions to the right, and $\bv_0=0^{n-1}$ aligns with the last $n-1$ tokens of CoT. The tokens before that are $\bx$ which align with $\bv_1[-n:]=\Tilde{\bw}_{11}=\bw_{11}$. Therefore, the output
$$\ete{f_\bw}{t_{11}}(\phi(\bx))=\th(\<\bw_{11},\bx\>)=y_{11}\,.$$
We now consider any time step $t \in \cI$ for $t=t_{li}$. Let $\bz:=\CoT{f_{\bw}}{(t-1)}(\phi(\bx))$ be the current CoT. We break into two cases.
\begin{enumerate}
    \item \emph{Case 1: $t=t_{li}$ with $i=1$ for some $l>1$.} 
    In this case, note that $t_{li}=t_{(l-1)s}+\Tilde{p}_l$. And thus, the CoT $\bz$ has $\Tilde{p}_{l}-1$ zeros at the end. Also, by Claim \ref{clm:recursive-sequence}, these zeros exactly align with $\bw[-(\Tilde{p}_l-1):]=[\bv_{l-1},\dots, \bv_0]$, i.e. the weights associated with the previous layers. And also note that $(n+t_{(l-1)s})=\Tilde{p}_l$, and thus these many coordinates prior to that are $\bv_l[-\Tilde{p}_{l} :] = \Tilde{\bw}_{l1}$
    %\gal{should be $\Tilde{p}_{l}$ and $\Tilde{\bw}_{l1}$?} 
    using Eqs.~\eqref{eq:bw-the-final-predictor} and \eqref{eq:bv-l-description}. Finally, using the induction hypothesis this $\Tilde{\bw}_{li}$ exactly aligns with the coordinates of $\bz$ that has the following form
    $$\bx, \underbrace{0 \dots 0}_{(n-1) \text{ bits }},y_{11}, \underbrace{0 \dots 0}_{(n-1) \text{ bits }},y_{12},\dots \dots, y_{\ell(i-1)},\underbrace{0 \dots 0}_{\Tilde{p}_\ell-1 \text{ bits}},y_{\ell i}, \dots \dots \dots \dots,y_{(l-1)(s-1)}, \underbrace{0 \dots 0}_{\Tilde{p}_{l-1}-1 \text{ bits }} y_{(l-1)s}\,.$$
    %\gal{$y_{ls}$ should be $y_{(l-1)s}$?} \nirmit{The number of zeros before should also be $\Tilde{p}_{l-1}-1$, right? previously was $\Tilde{p}_{l}-1$ }
Based on the way, we constructed $\Tilde{\bw}_{li}$, we directly conclude that
$$\ete{f_{\bw}}{t}=\th(\<\bw_{li},[\bx,y_{11},y_{12},\dots,y_{(l-1)(s-1)},y_{(l-1)s}]\>)=y_{li}\,.$$
%\gal{also here $y_{ls}$ should be $y_{(l-1)s}$?}
\item \emph{Case 2: $t=t_{li}$ with $i>1$ for some $l \in [L]$.} The argument is very similar to the previous case. If we only isolate the CoT after the time $t_{(l-1)s}$, then by the induction hypothesis and Claim \ref{clm:0-output-in-CoT}, it has the following form:
$$ \underbrace{0 \dots 0}_{\Tilde{p}_l-1 \text{ bits }},y_{l1}, \underbrace{0 \dots 0}_{\Tilde{p}_{l}-1 \text{ bits }},y_{l2},\dots \dots, y_{l(i-1)},\underbrace{0 \dots 0}_{\Tilde{p}_\ell-1 \text{ bits }}\,.$$
Again $[\bv_{l-1},\dots, \bv_0]$ align with the last $\Tilde{p}_\ell-1$ zeros and does not contribute. The $(i-1)\Tilde{p}_l$ many weights before that in $\bw$ are $\bv_{l}[-(i-1)\Tilde{p}_l :]=[\Tilde{\bw}_{l(i-1)},\dots, \Tilde{\bw}_{l1}]$ by Eq~\eqref{eq:bv-l-description}. Note that the only variables with potentially non-zero values are at the indices which are multiples of $\Tilde{p}_l$. These correspond to the last coordinates of each $\Tilde{\bw}_{l(i-1)},\dots, \Tilde{\bw}_{l1}$ respectively; the value of all these coordinates is zero by our assumption that the last weight coming from the previous layer is always 0. The CoT before this (from $\bx$ and new $t_{(l-1)s}$ CoT tokens) is exactly of the form we described in Case 1. Moreover, this now aligns with the $\Tilde{p}_l$ many coordinates prior to $\Tilde{\bw}_{l(i-1)}$, which by Eq.~\eqref{eq:bv-l-description} is exactly $\Tilde{\bw}_{li}$. Again noticing the form $\Tilde{\bw}_{li}$, we conclude
$$\ete{f_{\bw}}{t}=\th(\<\bw_{li},[\bx,y_{11},y_{12},\dots,y_{(l-1)(s-1)},y_{(l-1)s}]\>)=y_{li}\,.$$
%\gal{also here $y_{ls}$ should be $y_{(l-1)s}$?}
\end{enumerate}

\end{proof}

\section{Proof of Theorem \ref{thm:universal-class} \& Miscellaneous Discussions}\label{app:TM-attention}
\subsection{Context and Iterated Sparse Linear Thresholds}\label{subsec:sparse-linear-predictors}
Here, we discuss how the class of sparse linear thresholds (at least with bounded norm) already satisfies two of the three desiderata in \Cref{sec:universal} simultaneously. However, it remains open whether it satisfies the first desideratum of the expressive power. Let
\begin{equation}\label{eq:sparse-lin-predictor}
    \cFlinsp=\{f_{\bw,b} \in \cFlin : \norm{\bw}_0 \leq k\}\,.
\end{equation}
As mentioned, it remains open whether $\TM(\S,\Time) \subseteq \ete{\cFlin}{T}$ with $d,T=\poly(\Time)$ and $k \lll d$\,. However, the computational tractability is easy to see. The important thing is that it has the property of ``high-context'' and ``low-sample-complexity''. %Since $\VCdim(\ete{\cFlinsp}{T})=O(k \log d)$. Using \Cref{thm:cot-VC}, we have that $\ete{\cFlinsp}{d}$ is $\sete$-learnable with  $O(k\log^2 d)$ samples. One can also show a better bound of $O(k^2+k\log d)$. 
Indeed, the $\sete$ sample complexity is $O(k^2+k\log d)$.

We now prove that indeed our universal base class satisfies all three desiderata.
\subsection{Proof of Theorem \ref{thm:universal-class}}
We break the entire proof into three parts. 

\paragraph{Expressivity.} We first have that $\TM(\S,\Time) \subseteq \post \circ \ete{\TMAR}{\Time} \circ \pre$. (i.e. $\ete{\TMAR}{\Time}$ expresses $\TM(\S,\Time)$ up to pre and post-processing steps). This is because our autoregressive function essentially is designed to hard code each transition $(s_t,a_t,b_t)$, i.e. for any possible transition table $\tau$ and input $\bomega \in \{0,1\}^*$, we have 
$$\ete{f_{\tau}}{t}(\preprocess(\bomega))=(s_t,a_t,b_t)\,.$$
This can be shown by induction on $t$. For the base case, consider $t=1$. The head is initially at the $p_0=|\bomega|+1$, and $\text{Tape}[p_0]=\nil$, and the internal state is initialized to $s_0=1$. Therefore, we have $(s_1,a_1,b_1)=\tau(1,\nil)$. We must show that $\ete{f_{\tau}}{(t=1)}(\preprocess(\bomega))=f_{\tau}(\preprocess(\bomega))=(s_1,a_1,b_1)$ to verify the base case. Let $\bx=\preprocess(\bomega)$. Indeed, from the pre-processing step we have $\bx[\,i\,].\move=+1$ for all $i \in |\bx|$, and thus $\pos[\,i\,]=i-1$ for any $i\in |\bx|$, and there will be no position $i$ for which $\pos[\,i\,]=|\bx|$. Also, the last index in $\bx$ has $\bx[-1].\state=1$ by construction in the pre-processing step. Therefore, $(1,\nil)=\readtape(\bx)$. Clearly, by \Cref{alg:f_tau}, we output $(s_1,a_1,b_1)=\tau(1,\nil)$.

For the inductive case, let's assume that the claim holds for all steps $i$ in $1 \leq i \leq t$ for some $t$. Then we must show that 
$$\ete{f_{\tau}}{(t+1)}(\bx)=f_{\tau} \circ \CoT{f_{\tau}}{t}(\bx)=(s_{t+1},a_{t+1},b_{t+1})\,.$$
Let $\bz=\CoT{f_{\tau}}{t}(\bx)$. By induction, we know that for any $1\leq i  \leq t$, we have $p_{i-1}=\sum_{j<i}\bz[\,j\,].\move=\pos[\,i\,]$ is the head position of the machine just before the time-step $i$ and this is where the symbol $\bz[\,i\,].\symbol$ got written. By the way a Turing machine operates, the current head position after the end of $t$ steps is $p_t=p_{t-1}+b_t=\pos[-1]+\bz[-1].\move=\npos[|\bz|]$ (by the induction hypothesis). So our goal is to retrieve the symbol $r_{t+1}=\text{Tape}[\,p_t\,]$ currently present at this location, i.e. the symbol that was written most recently at this position on the tape. This exactly corresponds to finding the largest $j^*$ such that $\pos[j^*]=\npos[|\bz|]$ and retrieve $\bz[j^*].\symbol$, and if $j^*$ does not exist then at no point in the past, the head of the machine has been at the position so the tape must contain $\nil$. Our $\readtape(\bz)$ operation is exactly implementing this step (\Cref{step:line5}). So we indeed correctly find the symbol $r_{t+1}$ that was read by the Turing machine in order to make the next $(t+1)^\mathrm{th}$ move using our read-tape operation. Also, the read-tape operation outputs the state at the final location $\bz[-1].\state$ which is $s_t$ by the induction hypothesis. Therefore, $\readtape(\bz)=(s_{t},r_{t+1})$ is satisfied. Then we indeed have
$$\ete{f_{\tau}}{(t+1)}(\bx)=f_\tau(\bz)=\tau(\readtape(\bz))=\tau(s_t,r_{t+1})=(s_{t+1},a_{t+1},b_{t+1})\,,$$
where the second inequality follows from the definition of $f_{\tau}$ (\Cref{alg:f_tau}), the next one follows from the discussion above, and the last one follows directly from the way a Turing machine works.
% machine's next move is $\tau(s_t,r_{t+1})$, it is straightforward to   that Our read-tape operation e , and this is where 
Therefore, at the end-to-end step with $t=\Time$, we have $\ete{f_{\tau}}{\Time}(\preprocess(\bomega))=(s_\Time,b_\Time,a_\Time)\,.$
Applying post-processing step returns $a_{\Time}=g_{\<\S,\Time,\tau\>}(\bomega)$, which is the desired answer returned to the Turing machine with the transition rule $\tau$. This establishes $\TM(\S,\Time) \subseteq \postprocess \circ \ete{\TMAR}{\Time}\circ \preprocess$.

\paragraph{Sample Complexity.} By \Cref{thm:cardinality-based}, we have that $\ete{\TMAR}{\Time}$ is $\sCoT$-learnable by the rule \ref{eq:cot-consistent} on a sample set $S_{\sCoT}$ of size 
$$\mcot{\Time}=O(\eps^{\scriptscriptstyle{-1}} \left( \S \log \S +\log \delta^{\scriptscriptstyle{-1}}\right) )\,,$$
where we used the facts that $\log \abs{\TMAR} \leq \log (6\S)^{3\S}=O(\S \log \S)$ and \ref{eq:cot-consistent} is also \ref{eq:ete-consistent}.
\paragraph{Tractable $\sCoT$-learnability.} Finally, we need to show that \ref{eq:cot-consistent} for the class $\TMAR$ on $S_{\sCoT}$ of size $|S_{\sCoT}|=O(\eps^{\scriptscriptstyle{-1}} \left( \S \log \S +\log \delta^{\scriptscriptstyle{-1}}\right) )\,$  containing inputs of length at most $n$, is implementable in time $\poly(n,\Time, \eps^{\scriptscriptstyle{-1}},\log \delta^{\scriptscriptstyle{-1}})$. Due to \Cref{thm:ConsF-gives-CoT-Cons-implementation}, it suffices to show that the base class consistency \ref{eq:universal-cot-consistent} is solvable in time polynomial in its input. This is clearly true for the implementation of \ref{eq:universal-cot-consistent} we provide, concluding the proof of this theorem.

\section{Complementary Results from Section \ref{sec:samples-computational-complexity-general-F}}\label{app:sample-complexity-side}
In this section, we provide all our complementary results including lower bounds during the discussion in \Cref{sec:samples-computational-complexity-general-F}. We start with the proof of \Cref{thm:lb-Omega(T VC)} that establishes that a linear dependence on the generation length $T$ in the sample complexity is necessary for a guarantee based on $\VCdim$ of the base class. 
\subsection{Proof of \texorpdfstring{\Cref{thm:lb-Omega(T VC)}}{}}
The proof will follow by using \Cref{prop:fundamental-VC} once we have a construction of a family of base classes with $\VCdim(\ete{\cF}{T})=\Omega(T \cdot \VCdim(\cF))$. Formally, we show the following.
\begin{theorem}\label{thm:lb-VCete-VCbase} For $\Sigma=\{0,1\}$, for any $D,T \in \naturals_{+}$, there exists $\cF \subseteq \Sigma^{\Sigma^*}$ such that 
$$\VCdim(\cF)=D \quad \text{ but } \quad \VCdim(\ete{\cF}{T}) = T \cdot D, \, \text{ over domain $\cX=\Sigma^n$ with } n=\lceil \log_2 (DT) \rceil+1.$$
\end{theorem}
We will return to the proof of this theorem---for now it is immediate to see the following proof. 
\begin{proof}[Proof of \Cref{thm:lb-Omega(T VC)}] Use the negative direction of the Fundamental Theorem of Statistical Learning (\Cref{prop:fundamental-VC}), and recall that \Cref{def:learnable-wo-CoT} is simply a supervised learning for $\ete{\cF}{T}$ class.     
\end{proof}
Let us now show the construction of such base classes.
\begin{proof}[Proof of \Cref{thm:lb-VCete-VCbase}] The claim trivially holds for $T=1$. Thus, below we consider $T \geq 2$. \\
\textbf{Description of the Class.}
For any target $D,T \in \naturals_{+}$, we aim to construct a hypothesis class $\cF$ from $\{0,1\}^*$ to $\{0,1\}$. Consider the following set of $M:=D \cdot T$ strings of length $n= \lceil \log_2  M \rceil+1$ points. 
$$S=(\bx_1,\dots, \bx_M),\;\quad\text{where }  \bx_i= 1 \underbrace{(\text{bit representation of $i$})}_{\lceil \log_2  M \rceil \, \text{ bits}}.$$
Our base class $\cF$ has functions $f_{\vbb}$ that will correspond to sign-patterns on $M$ points:
$$ \cF=\{f_{\vbb} : \vbb \in \{0,1\}^{M}\}.$$
We will define each hypothesis in such a way that $\VCdim(\cF)=D$, but when we apply each $f_{\vbb}$ auto-regressively $T$ times on the points in $S$, we realize the sign-pattern $\vbb$ at the end, i.e. $\ete{f}{T}_{\vbb}(S)=\vbb$.
\begin{equation}\label{eq:description-of-f_vbb}
    f_{\vbb}(\bx)=\begin{cases}
    b_{jD+k},& \text{if } \bx=0^*\underbrace{\bx_i}_{n \,\text{bits}}\underbrace{b_{k}b_{D+k} \dots b_{(j-1)D+k}}_{j \, \text{bits}}, \text{ for } ((i-1)\text{ mod }D \equiv k-1), \text{ and } 0 \leq j \leq (T-2),\\
    b_i, & \text{if } \bx= 0^*\underbrace{\bx_i}_{n \, \text{bits}} \underbrace{b_k \, b_{D+k}\dots b_{(T-2)D+k}}_{(T-1) \text{ bits}}, \text{ for some } ((i-1) \text{ mod } D \equiv k-1),\\
    0, & \text{otherwise}.
\end{cases}
\end{equation}
In simple words, the hypothesis $f_{\vbb}: \Sigma^* \to  \Sigma$ is defined in such a way that on the point $\bx_i$ such that $(i-1) \text{ mod } D = (k-1)$ for $k \in [D]$ outputs $b_k,b_{D+k},b_{2D+k},\dots, b_{(T-2)D+k}$ in the next $(T-1)$ steps of auto-regression before outputting $b_i$ at the $T^{\mathrm{th}}$ step.

It is first important to note that each $f_{\vbb}:\Sigma^* \to \Sigma$ is well-defined. Because, for any input $\bx \in \Sigma^*$, we can first drop the leading 0s since $\bx_i$ always starts with $1$. Therefore, we can uniquely decode if the input point's prefix is of the form $0^* \bx_i$. Once that is determined, we can find the value of $k=((i-1)\text{ mod }D)+1$ and check whether the remaining tokens to be followed (if there are any) are of the form $b_k, b_{D+k}, \dots$, to decide the output on $\bx$. If $\bx$ does not match the description, then it simply outputs 0.\\
\textbf{Showing end-to-end $\VCdim$ is large.} It is straightforward to note that for the set of points $S=(\bx_1,\dots,\bx_M)$, at the end of $T$ steps, by definition, we will realize all sign-patterns of these $M$ points. In particular, for any $\bx_i \in [M]$, we have $f_{\vbb}^{\sete(T)}(\bx_i)=b_i$. We achieve
$$  \left|\{ (\ete{f}{T}_{\vbb}(\bx_1), \dots, \ete{f}{T}_{\vbb}(\bx_M)): f_{\vbb} \in \cF \} \right|= |\{\vbb: \vbb\in \{0,1\}^M\}|=2^M,$$
shattering all points in $S$, giving us $\VCdim(\ete{\cF}{T}) \geq M$.
Also, we know that $\VCdim(\ete{\cF}{T}) \leq \log_2 |\cF| \leq M.$ Combining, we obtain
$$\VCdim(\ete{\cF}{T})=M = T \cdot D.$$
\textbf{Showing that Base Class has small $\VCdim$.} We first start by showing that $\VCdim(\cF) \geq D$. Fix the set $S'=(\bx_1,\dots, \bx_D) \subset S$, of the first $D$ points in $S$. It is immediate to see that $k=i$ for any $i \in [D]$, and therefore, applying the first case in \eqref{eq:description-of-f_vbb}
$$ \text{for any } i \in [D], \text{ we have } f_{\vbb}(\bx_i)=b_k=b_i.$$
The main challenge is to show that our base class $\cF$ has $\VCdim(\cF) \leq D$. To this end, suppose for the purpose of contradiction, that $\VCdim(\cF) \geq D+1$. Then there exists a set $\Tilde{S} \in (\Sigma^*)^{D+1}$ such that all $2^{D+1}$ possible labelings can be realized by some $f_{\vbb} \in \cF$. Then it must be that any point $\Tilde{\bx} \in \Tilde{S}$ is of the form
$$ \bx\, = \, 0^* \bx_i \Tilde{b}_1 \dots \Tilde{b}_{\ell}, \text{ for some } 1 \leq \ell \leq (T-1), \text{ and } \Tilde{b}_1, \dots, \Tilde{b}_{\ell} \in \{0,1\}.$$
Because otherwise the functions in $\cF$ just outputs $0$ on that point. Then by pigeon-hole argument, there exists $i_1, i_2 \in M$ such that 
$$(i_1-1)\,\text{mod}\, D = (i_2-1) \,\text{mod} \,D =(k-1), \text{ for some } k \in [D],$$
and there are two points $\Tilde{\bx}_1, \Tilde{\bx}_2 \in \Tilde{S}$ such that
\begin{align*}
    \Tilde{\bx}_1 \,&= \, 0^* \bx_{i_1} \Tilde{b}_{k} \Tilde{b}_{D+k}\dots \Tilde{b}_{(\ell_1-1)D+k} \\
    \Tilde{\bx}_2 \,&= \, 0^* \bx_{i_2} \Tilde{b}_{k} \Tilde{b}_{D+k}\dots \Tilde{b}_{(\ell_1-1)D+k} \dots \Tilde{b}_{(\ell_2-1)D+k}, \text{ for some } \Tilde{\bb} \in \{0,1\}^M.
\end{align*}
where w.l.o.g., we let $\ell_2 \geq \ell_1$. Moreover, we can continue to shatter $\{\Tilde{\bx}_1, \Tilde{\bx}_2\}$ using $\cF$. We break into the cases to show that there is a contradiction. 
\begin{enumerate}
    \item $\ell_1<\ell_2\leq (T-1)$: The main intuition is that, any predictor, that outputs 1 on the point $\Tilde{\bx}_2$, is also required to output $\Tilde{b}_{\ell_1 D+k}$ on $\Tilde{\bx}_1$, due to the first case \eqref{eq:description-of-f_vbb}. Formally, consider the target labels $(\Tilde{y}_1, \Tilde{y}_2)=(\neg \Tilde{b}_{\ell_1\cdot D+k}, 1)$. Then we claim that, for any predictor $f_{\vbb} \in \cF$
    $$\Tilde{y}_1 \neq f_{\vbb}(\Tilde{\bx}_1) \;\; \text{ or }\;\; \Tilde{y}_2 \neq f_{\vbb}(\Tilde{\bx}_2).$$
    This is because in order for any $f_{\vbb}(\Tilde\bx_2)=1$, we must have that $\vbb$ agrees with $(\Tilde{b}_{k}, \Tilde{b}_{D+k},\dots,\Tilde{b}_{(\ell_2-1)D+k})$, on the respective indices; otherwise $f_{\vbb}(\Tilde\bx_2)=0$, by the third case of \eqref{eq:description-of-f_vbb}. So 
    $$ (b_{k},b_{D+k},\dots,b_{(\ell_1-1)D+k}, b_{\ell_1 D+k}, \dots, b_{(\ell_2-1)D+k} )=(\Tilde{b}_{k}, \Tilde{b}_{D+k},\dots,\Tilde{b}_{(\ell_1-1)D+k}, \Tilde b_{\ell_1 D+k}, \dots, \Tilde b_{(\ell_2-1)D+k}).$$
    But this also subsumes that $(b_{k},b_{D+k},\dots,b_{(\ell_1-1)D+k})=(\Tilde{b}_{k}, \Tilde{b}_{D+k},\dots,\Tilde{b}_{(\ell_1-1)D+k})$. According to first case of \eqref{eq:description-of-f_vbb}, we must have $f_{\vbb}(\Tilde{\bx}_1)=b_{\ell_1 \cdot D + k}=\Tilde{b}_{\ell_1 \cdot D+k}$, giving us that we cannot realize the behavior $(\Tilde{y}_1, \Tilde{y}_2)=(\neg \Tilde{b}_{\ell_1\cdot D+k}, 1)$, contradicting that we can shatter $\{\Tilde{\bx}_1, \Tilde{\bx}_2\}$.
    \item $\ell_1=\ell_2 < (T-1)$: In this case, the main intuition is that, any predictor $f_{\vbb} \in \cF$, either agrees with the bit pattern of $\Tilde{\bb}$ on the respective bits and outputs the same bit $b_{\ell_1 D+k}=b_{\ell_2D+k}$ or otherwise, it outputs $0$s on both points. So, effectively, we cannot have different outputs on the two points. Formalizing this, consider the target labels $(\Tilde{y}_1, \Tilde{y}_2)=(0, 1)$. Then for any predictor $f_{\vbb} \in \cF$
    $$\Tilde{y}_1 \neq f_{\vbb}(\Tilde{\bx}_1) \;\; \text{ or }\;\; \Tilde{y}_2 \neq f_{\vbb}(\Tilde{\bx}_2).$$
    This is because in order for any $f_{\vbb}(\Tilde\bx_2)=1$, we must have that $\vbb$ is such that
    $$ (b_{k},b_{D+k},\dots,b_{(\ell_2-1)D+k})=(\Tilde{b}_{k}, \Tilde{b}_{D+k},\dots,\Tilde{b}_{(\ell_2-1)D+k})\; \text{ and } \; 1=b_{\ell_2 D+k}=b_{\ell_1 D+k}.$$
    But this immediately also implies that, according to first case of \eqref{eq:description-of-f_vbb}, we must have $f_{\vbb}(\Tilde{\bx}_1)=b_{\ell_1 \cdot D + k}=1$, contradicting that we can shatter $\{\Tilde{\bx}_1, \Tilde{\bx}_2\}$.
    \item $\ell_2=\ell_1=(T-1)$: In this case, in order to shatter $\{\Tilde{\bx}_1, \Tilde{\bx}_2\}$, we must have $i_1 \neq i_2$ because otherwise any $f_{\vbb} \in \cF$ will have the identical outputs on both the points. W.l.o.g. let $i_1 < i_2$. Then we have $i_1=(j_1-1) \cdot D+k$ for some $1 \leq j_1 \leq T-1$.
    Then we shall show that for the target labels $(\Tilde{y}_1, \Tilde{y}_2)=(\neg \Tilde{b}_{(j_1-1) \cdot D+k}, 1)$, for any predictor $f_{\vbb} \in \cF$
    $$\Tilde{y}_1 \neq f_{\vbb}(\Tilde{\bx}_1) \;\; \text{ or }\;\; \Tilde{y}_2 \neq f_{\vbb}(\Tilde{\bx}_2).$$
    This holds true because in order for any $f_{\vbb}(\Tilde\bx_1)=\neg \Tilde{b}_{(j_1-1) \cdot D+k}$, we must have $\vbb$ such that 
    $$(b_{k}, b_{D+k},\dots,b_{(T-2)D+k}) \neq (\Tilde{b}_{k}, \Tilde{b}_{D+k},\dots,\Tilde{b}_{(T-2)D+k});$$
    otherwise, by the second case of \eqref{eq:description-of-f_vbb}, the output $f_{\vbb}(\Tilde\bx_1)=b_{i_1}=b_{(j_1-1) \cdot D+k}=\Tilde{b}_{(j_1-1) \cdot D+k}$. But if this is the case, by the third case of  \eqref{eq:description-of-f_vbb}, we must also have $f_{\vbb}(\Tilde{\bx}_2)=0$. Therefore, we cannot realize $(\Tilde{y}_1, \Tilde{y}_2)=(\neg \Tilde{b}_{(j_1-1) \cdot D+k}, 1)$, achieving a contradiction. 
\end{enumerate}
\end{proof}
\subsection{Sample Complexity of Learning Time-Dependent End-to-End Class.}\label{app:NTI-VCdim}
We now discuss that even for the end-to-end classes that have time-dependent compositions have roughly the same complexity, i.e. up to $\log T$ factor, as the worst-case lower bound from \Cref{thm:lb-Omega(T VC)}. We show that $\VCdim(\NTIete{\cF}{T})=O(T \cdot \VCdim(\cF) \log T)$, essentially the same as $\VCdim(\ete{\cF}{T})=O(T \cdot \VCdim(\cF)$) up to a $O(\log T)$ factor.
\begin{theorem}[Time Dependent Class]\label{thm:NIT-VC-bound} Consider any finite $\Sigma$, a base function class $\cF \subseteq \Sigma^{\Sigma^*}$ and generation length $T \in \naturals_{+}$.
\begin{itemize}
    \item For binary $\Sigma=\{0,1\}$:
    $$\VCdim(\NTIete{\cF}{T}) \leq 3 \, T  \cdot \VCdim(\cF) \log\left( \frac{2\,T}{\ln 2}\right).$$
    \item For non-binary finite $\Sigma$:
    $$\Ndim(\NTIete{\cF}{T}) \leq 3\, T  \cdot \Ndim(\cF) \, \log_2 \left( \frac{2\, \Ndim(\cF) |\Sigma|^2}{e \ln 2} \right)\,.$$
\end{itemize}
\end{theorem}

\begin{proof}[Proof of \Cref{thm:VCete-VCbase}] Let $\Gamma_{T}:=\Gamma_{\NTIete{\cF}{T}}$ be the growth function of the class $\NTIete{\cF}{T}$ for $T \in \naturals_{+}$. Then by definition $\Gamma_1=\Gamma_{\cF}$. We will try to bound $\Gamma_T$ recursively. Let us first define
   $$\cF^T:=\underbrace{\cF \times \dots \times \cF}_{T \text{\, times}}.$$  
   Consider any fixed set $S \in (\Sigma^*)^m$ of size $|S|=m$. For any $T \in \naturals_{+}$, divide $\cF^T$ into equivalence classes according to  the equivalence relation $\equiv_T$ such that 
   $$ f:=(f_1,\dots,f_T)\equiv_T g:=(g_1,\dots,g_T), \text{ iff } \CoT{f}{T}(\bx)=\CoT{g}{T}(\bx), \text{ for all } \bx \in S\,.$$
  In words, a pair of time-dependent functions $f,g \in \cF^T$ are equivalent if for the next $T$ steps, for any point in $S$, they both have identical chain-of-thoughts on that point. Let $\kappa_T(S)$ be the number of equivalence classes. By definition of growth function,
\begin{equation}\label{eq:growth-function<=equiv-classes}
    \Gamma_T(m) \leq \max_{S \in (\Sigma^*)^m} \kappa_T(S)\,.
\end{equation}
Therefore, it suffices to bound the latter. The main observation is that each equivalence class with relation $\equiv_{T}$ may split into a partition of at most $\Gamma_{\cF}(m)$ many equivalence classes with respect to the relation $\equiv_{T+1}$, because we have at most $\Gamma_{\cF}(m)$ many behaviors on any fixed set of $m$ points our hypothesis class. We obtain
 \begin{equation}\label{eq:recursion-before-sauer}
    \kappa_{T+1}(S) \leq \Gamma_{\cF}(m) \cdot \kappa_{T}(S) \quad \text{ and } \quad \kappa_1(S) \leq \Gamma_{\cF}(m).
 \end{equation}
 Solving this recursion yields 
 \begin{equation}\label{eq:solved-recursion}
  \kappa_T(S) \leq \left(\Gamma_{\cF}(m) \right)^T  
 \end{equation}
This argument holds for any $S \in   (\Sigma^*)^m$. Combining the inequalities \eqref{eq:growth-function<=equiv-classes} and \eqref{eq:solved-recursion}, we have
$$ \Gamma_T(m) \leq \left(\Gamma_{\cF}(m) \right)^T \,.$$
We now break into cases:
\begin{itemize}
    \item $\Sigma=\{0,1\}$ (binary alphabets): In this case, using Sauer's lemma (Lemma \ref{lem:sauer'slemma}), for any $m \geq \VCdim(\cF)$,
    $$\Gamma_{T}(m) \leq \Gamma_{\cF}(m)^T \leq \left( \frac{em}{\VCdim(\cF)} \right)^{T \cdot \VCdim(\cF)} \,.$$ 
    Therefore, $\VCdim(\NTIete{\cF}{T})$ can be bounded by any $m\in \naturals_{+}$ satisfying $2^m >  \Gamma_{T}(m)$ and $ m \geq \VCdim(\cF).$ Using the derived upper bound on $\Gamma_T(m)$, it suffices to choose $m$ such that
    $$m > T \cdot \VCdim(\cF) \, \log_2 \left( \frac{em}{\VCdim(\cF)}\right) \quad \text{and}\quad m \geq \VCdim(\cF).$$
    Using Lemma \ref{lem:technical-corollary} with $N=T  \cdot \VCdim(\cF)$ and $M=1/\VCdim(\cF)$, we directly argue the existence of $m \in \naturals_{+}$ such that
    $$ \VCdim(\NTIete{\cF}{T}) \leq  m \leq 3 \, T \cdot \VCdim(\cF) \, \log_2 \left(\frac{2T}{\ln 2}\right)\,.$$
    \item For a finite $\Sigma:$ Using Natarajan's lemma (Lemma \ref{lem:Natarajan's lemma}) for any $m \in \naturals_{+} $,
    $$\Gamma_{\NTIete{\cF}{T}}(m) \leq \Gamma_{\cF}(m)^T \leq \left(|\Sigma|^2 \cdot m\right)^{T \cdot \Ndim(\cF)} \,.$$ 
    Again, $\Ndim(\NTIete{\cF}{T})$ can be bounded by $m \in \naturals_{+}$ for which $m>  T \cdot \Ndim(\cF) \log_2 \left(|\Sigma|^2 \cdot m \right).$ 
    Applying Lemma \ref{lem:technical-corollary} with $N=T \cdot \Ndim(\cF)$ and $M= |\Sigma|^2/e$, we obtain that there exists such $m\in \naturals_{+}$ such that 
    $$ \Ndim(\NTIete{\cF}{T}) \leq m \leq 3 \, T  \cdot \Ndim(\cF) \log_2 \left( \frac{2\, \Ndim(\cF) |\Sigma|^2}{e \ln 2} \right). $$
\end{itemize}
\end{proof}

\subsection{Lower Bound for Theorem \ref{thm:ete-ldim} in terms of Littlestone Dimension}
We now show that $\Omega(\ldim(\cF))$ samples are necessary in the worst case. Our lower bound is non-trivial in the sense that we show this for a family of classes that always have $\VCdim(\cF)=1$. First, note that by the negative direction of the Fundamental Theorem of Statistical Learning (\Cref{prop:fundamental-VC}), it suffices to show that the VC dimension of the end-to-end class, $\VCdim(\ete{\cF}{T})$, is lower bounded by $\ldim(\cF)$, to establish the lower bound in the sample complexity of end-to-end learning of $\ete{\cF}{T}$. The following theorem establishes this for a non-trivial family of classes.
\begin{theorem}[Lower Bound for Littlestone Dimension]\label{thm:lb-ldim}
    For $\Sigma=\{0,1\}$, for any $D \in \naturals_{+}$, there exists $\cF \subseteq \Sigma^{\Sigma^*}$ such that 
    $$\ldim(\cF)=D, \VCdim(\cF)=1\, \text{over domain }\Sigma^* \,,$$
    and 
    $$\VCdim(\ete{\cF}{T}) = D \text{ for any } T > D \text{ even over domain $\Sigma^n$ for } n=\lceil\log D\rceil+1~.$$
\end{theorem}
\begin{proof}[Theorem \ref{thm:lb-ldim}]
    \noindent\textbf{Description of the Class.} For any target $D\in \naturals_{+}$, consider a set of $D$ points of length $n= \lceil \log_2 D \rceil + 1$ defined as:
   $$S=\{\bx_i: i \in [D]\},\;\quad\text{where }  \bx_i= 1\underbrace{(\text{bit representation of $i$})}_{\lceil \log_2  D \rceil \, \text{ bits}}.$$
    We define the function class $\cF$ consisting of functions $f_{\vbb}$ for each $\vbb \in \{0,1\}^D$, defined as:
    $$f_{\vbb}(\bx) = \begin{cases}
        b_{j+1}, & \text{if } \bx = 0^* \underbrace{\bx_i}_{n\text{ bits}}\underbrace{b_1 b_{2} ... b_{j}}_{j\text{ bits}} \text{ for some } i \in [D], 0 \le j < D\\
        b_i, & \text{if } \bx = 0^* \underbrace{\bx_i}_{n\text{ bits}} b_1 b_{2} ... b_{D} b_i^* \text{ for some } i \in [D]\\
        0, & \text{otherwise}
    \end{cases}$$
    In simple words, the function $f_{\vbb}$ is defined in such a way that on the point $\bx_i$, it outputs $b_1, \ldots, b_D$ in the next $D$ steps of auto-regression and then outputs $b_i$ for the rest of the steps.\\
    \textbf{Showing Base Class has Littlestone Dimension $D$.}
    We first show that $\ldim(\cF) \geq D$. Let $\bz$ be a complete binary tree of depth $D$. For any binary string $\epsilon \in \{0,1\}^j$ of length $j < D$, let $\bz(\epsilon)$ denote the example at the node reached by following path $\epsilon$ from the root. We label the tree as follows:
    $$\bz(\epsilon) = \bx_1 \epsilon[1] \epsilon[2] \ldots \epsilon[j]$$
    where $\epsilon[i]$ denotes the $i$th bit of $\epsilon$. Note that we use $\bx_1$ for simplicity; the same argument works with any $\bx_i$.
    
    We now show this tree is shattered by $\cF$. For any path $\epsilon \in \{0,1\}^D$, consider $f_{\vbb}$ with $\vbb = \epsilon$. By construction of $f_{\vbb}$:
    \begin{itemize}
        \item At root: $f_{\vbb}(\bx_1) = b_1 = \epsilon[1]$
        \item For any prefix $\epsilon$ of length $0 \le j < D$:
            $$f_{\vbb}(\bz(\epsilon)) = f_{\vbb}(\bx_1 \epsilon[1] ... \epsilon[j]) = b_{j+1} = \epsilon[j+1]$$ where the second equality follows from the first case of the function definition
    \end{itemize}
    Thus $f_{\vbb}$ realizes the labeling $\epsilon$ on this tree. Since this holds for any $\epsilon \in \{0,1\}^D$, we have $\ldim(\cF) \geq D$.
    
    For the upper bound, note that $\ldim(\cF) \leq \log_2|\cF|$ for any function class. Since our construction has $|\cF| = 2^D$ functions (one for each $\vbb \in \{0,1\}^D$), we immediately get $\ldim(\cF) \leq D$. Combined with the lower bound, this establishes $\ldim(\cF) = D$.\\
    \textbf{Showing Base Class has VC Dimension 1.}
    First, we show $\VCdim(\cF) \geq 1$ by exhibiting a point that can be shattered. Consider the point $\bx_1$. We can achieve both labelings:
    \begin{itemize}
        \item For $0$: Choose $f_{\vbb}$ with $b_1=0$. Then $f_{\vbb}(\bx_1)=b_1=0$
        \item For $1$: Choose $f_{\vbb}$ with $b_1=1$. Then $f_{\vbb}(\bx_1)=b_1=1$
    \end{itemize}
    
    To show $\VCdim(\cF) \le 1$, consider any two distinct points $\by_1 = \bx_i p$ and $\by_2 = \bx_j q$ for $i,j \in [D]$ and $p,q \in \{0,1\}^*$. Assume without loss of generality that $|p| \le |q|$. Suppose for contradiction that we can realize both labelings (1,1) and (1,0)/(0,1). Now consider the following cases:

    \begin{itemize}
        \item $p = q$. If $i=j$ then both points are the same, so they can not be shattered. So we can assume $i\ne j$.
        \begin{itemize}
            \item If $|p| < D$ then for any $\vbb$ that has $p$ as the prefix, both $\by_1$ and $\by_2$ evaluate to $\vbb[|p|+1]$, and for any $\vbb$ that does not have $p$ as a prefix, they evaluate to 0. So they cannot achieve (0,1) or (1,0)  labeling and thus cannot be shattered.
            \item If $|p| \ge D$, in order to get output (1,1), we would need some $\vbb$ to match $p$ on its first $D$ bits, have $\vbb[i] = \vbb[j] = 1$ and the rest of the bits of both $p$ and $q$ be 1. Now for any other $\vbb$, on this input, $\vbb$ would not match $p$ on its first $D$ bits, and so $f_{\vbb}(\by_1)=0$ and $f_{\vbb}(\by_2)=0$. So we cannot achieve $(1,0)$ or $(0,1)$ labeling and thus cannot be shattered.
        \end{itemize}
        \item $p$ is a strict prefix of $q$. 
        \begin{itemize}
            \item If $r \le D$, then to get labeling $\by_1$ as 1 forces $\vbb$ to match $p$ for its first $r$ bits, and have $\vbb[r] = 1$. Now if $q[r]$ is 0, then $f_{\vbb}(\by_2)=0$, so we cannot achieve (1,1). So $q[r]=1$ must be 1. Now if we want $\by_1$ to labeled to 0, then we would need $\vbb$ to not match $p$. But $p$ is a prefix of $q$ so this qould imply that $f_\vbb(\by_2)=0$ implying (0,1) is not achievable. 
            \item If $r > D$, then $p$ and $q$ share the first $D$ bits, So for any $\vbb$, either both $\by_1$ and $\by_2$ are labeled 0, or $\by_1$ is labeled $p[i]$ and $\by_2$ is labeled $p[j]$. So they cannot be shattered.
        \end{itemize}
        \item $p \ne q$ and they differ fon some position $r \le |p|$:
        \begin{itemize}
            \item If $r \le D$ and $p[r]\neq q[r]$ then to label $\by_1$ with 1, the function $f_{\vbb}$ must match $p$ on the first $r$ bits, forcing $\vbb[r] = p[r]$. Because $q[r]\neq p[r]$, we get $f_{\vbb}(\by_2)=0$. This implies that we cannot assign both $\by_1,\by_2$ the label 1, and these points cannot be shattered.
            \item If $r > D$, then $p$ and $q$ share all the first $D$ bits but differ after bit $D$. By definition of $f_{\vbb}$, to get output 1 on $\by_1$, we would need $p = \vbb 1^*$ and $p[i] = 1$. But since $p[r]\neq q[r]$ beyond bit $D$, for such $\vbb$, $f_{\vbb}(\by_2)=0$. Hence we cannot assign both $\by_1,\by_2$ the label 1, and these points cannot be shattered.
        \end{itemize}
    \end{itemize}
    \textbf{Showing End-to-End Class has VC Dimension $D$.}
    We first show that $\VCdim(\ete{\cF}{T}) \geq D$. Consider the set $S = \{\bx_1,\dots,\bx_D\}$. For any desired labeling $\epsilon \in \{0,1\}^D$, we construct $f_{\vbb}$ with $\vbb = \epsilon$. After $T > D$ steps, the end-to-end function on input $\bx_i$ outputs $b_i$ for all $i \in [D]$. Thus we can realize any labeling $\epsilon$ on $S$.
    
    To show $\VCdim(\ete{\cF}{T}) \leq D$, suppose for contradiction that we could shatter $D+1$ points $\{\by_1,\dots,\by_{D+1}\}$. Consider the behavior of any $f_{\vbb} \in \cF$ on these points: for each point $\by_j$, after $T> D$ steps: if $\by_j$ starts with some $\bx_i$, follows the path $b_1,\dots,b_D, b_i^*$ and the output is $b_i$, otherwise (if it deviates or doesn't start with any $\bx_i$), the output is $0$.

    In order to achieve the labeling with all $D+1$ points being 1, we would need all $D+1$ points to start with some $\bx_i$ and follow the path $b_1,\dots,b_D, b_i^*$ for some $\vbb$. Since we only have $D$ possible choices for $\bx_i$, by pigeonhole, two of the points will need to share $\bx_i$ for some $i$ and have the form $\bx_1b_1\dots b_Db_i^*$. However, for these two points, all other functions in the class would output 0 since they would not match the prefix $b_1\dots b_D$, so we cannot achieve the labeling (1, 0) or (0, 1) on these two points. Therefore we cannot shatter $D+1$ points, so $\VCdim(\ete{\cF}{T}) \leq D$.
\end{proof}

\subsection{End-to-End Sample Complexity may Collapse}
We now show that learning the end-to-end class sometimes might be simpler than learning the base class (or even trivial). We show a family of base classes with a high VC dimension, but the end-to-end learning complexity collapses after one step of iterative composition.
\begin{theorem}\label{thm:collapse}For any natural number $D\in \naturals_{+}$, there exists a hypothesis class $\cF$ such that $\VCdim(\cF)=D$ over domain $\cX=\{0,1\}^n$ for $n=\log \lceil D\rceil+1$ but $\VCdim(\ete{\cF}{T})=0,$ even with $T=2$.
As a corollary, $\ete{\cF}{T}$ is learnable without any samples. However, for every learning rule $A$, there exists a distribution $\cD$ over $\{0,1\}^n$ realizable by $\cF$ such that for any set $S\sim \cD^m$ with $m<D/2$, we have
$$\P\left( \cL_\cD(A(S)) \geq 1/4 \right) \geq 1/7\,.$$
\end{theorem}

\begin{proof}[Proof of \Cref{thm:collapse}]
    Consider the set of $D$ points as follows:
    $$S=( \bx_1,\dots, \bx_D), \text{ where } \bx_i=1 \underbrace{(\text{bit representation of }i)}_{\lceil \log_2 D \rceil \text{ bits }}1\, .$$
\textbf{Description of the class.} For each bit pattern $\vbb \in \{0,1\}^D$, we have one hypothesis $f_{\vbb} \in \cF$ which is defined as follows
\begin{equation}\label{eq:description-of-f_vbb-VC-collapse}
    f_{\vbb}(\bx)=\begin{cases}
    b_i,& \text{if } \bx=0^*\underbrace{\bx_i},\\
    0, & \text{otherwise.} 
\end{cases}
\end{equation}
Again, note that this class is well-defined: for any input $\bx \in \Sigma^*$, we drop the leading 0s since $\bx_i$ always starts with $1$ and then uniquely decodes whether $\bx$ is of the form $0^*\bx_i$. \\
\textbf{VCdim after at step $1$ is large.} It is immediate to see that $|\cF|=2^D$ and $\VCdim(\cF) \leq \log |\cF| \leq D.$ Also, for the set of points $S$, by definition for any $\vbb \in \{0,1\}^D$, we have that 
$$|\{ (f_{\vbb}(\bx_i), \dots, f_{\vbb}(\bx_D)) : f_{\vbb} \in \cF \}|= |\{(b_1,\dots, b_D): \vbb \in \{0,1\}^D\}|=2^D.$$
Therefore, $\VCdim(\cF) \geq D$. Combining, we obtain $\VCdim(\cF)=D$.\\
\textbf{VCdim collapses at step 2.} We now want to show that $\VCdim(\ete{\cF}{T})=0$ for $T=2$ (fix for the rest of the proof). Suppose for the purpose of contradiction, assume that there is a point $\bx \in \Sigma^*$ such that $\{ \ete{f_{\vbb}}{T}: \vbb \in \cF\}=\{0,1\}$, i.e. we can get both behaviors at two steps.
\begin{enumerate}
    \item If $\bx=0^*\bx_i$ for some $i \in [D]$: In this case, it is easy to see that for any $\vbb \in \cF$
    $$\ete{f}{T}(\bx)=\ete{f}{T}(0^*\bx_i)=f_{\vbb}(0^*\bx_i b_i)=0,$$
    contradicting that we can get both possible outputs.
    \item If $\bx \neq 0^*\bx_i$ for any $i \in [D]$: In this case, we first note that $[\bx,0]$ is also not of the form $0^*\bx_i$, as $\bx_i$ always has the right most symbol $1$. Using this
    $$\ete{f_{\vbb}}{T}(\bx)=f_{\vbb}(\bx 0)=0,$$
    again reaching a contradiction. 
\end{enumerate}
\end{proof}

\subsection{Real Valued Alphabets}\label{app:real-alphabets}
In this section, we formalize the claim mentioned in \Cref{rem:|Sigma|>2}. 
%Let us start by the definition of pseudo dimension.  
\begin{definition}[Pseudo Dimension \citep{pollard1989asymptotics}] For any domain $\cX$ and a real valued function class $\cH \subseteq \R^{\cX}$. Then pseudo-dimension of $\cH$, denoted by $\Pdim(\cH)$ is the largest $D\in \naturals$ such that there exists a set of points $\bbx_1,\dots, \bbx_D \in \cX$ and real target thresholds $\theta_1,\dots, \theta_D \in \R$ where
$$ \left|\{(\sign( h(\bbx_1)-\theta_1), \dots,  \sign( h(\bbx_D)-\theta_D)): h \in \cH\} \right| = 2^D.$$
Moreover, if there is no such $D$, then we say that $\Pdim(\cF)=\infty$. 
\end{definition}
Under some assumptions on the loss function, pseudo-dimension 
%tightly characterizes 
guarantees
learnability. Below is the construction of $\cF \subseteq {\R}^{\R^*}$, which has bounded pseudo-dimension. However, even one more iterative composition leads to an infinite pseudo-dimension of the end-to-end class.
\begin{theorem}\label{thm:real}
    Let $\Sigma=\R$. There exists a class $\cF \subseteq \R^{\R^*}$ such that $\Pdim(\cF)=1$, but $\Pdim(\ete{\cF}{T})=\infty$, even for $T=2$.
\end{theorem}
\begin{proof}[Proof of \Cref{thm:real}]
    We will consider a hypothesis class $\cF$ defined as follows. Consider the following set of points $$S=\{\bx_n : n \in \naturals_{+}\} \text{ where } \bx_n= \underbrace{1 \dots 1}_{n \text{ times }} 0=1^n 0.$$
  Therefore, the cardinality of $S$ is countably infinite. Now the hypotheses class is defined such that there exists $f_{\bb} \in \cF$ for each sequence $\bb \in \{0,1\}^{\infty}$, whose output is defined as follows.
  \begin{equation}\label{eq:fbb-description-real-tokens}
      f_{\bb}(\bx)=\begin{cases}
         n+0.\bb:=n+0.b_1b_2 \dots , \quad  & \text{if } \bx= 0^*
         \bx_n\\
          b_n \quad &  \text{if } \bx=[0^*\bx_n (n+0.\bb)]\\
          0,\quad & \text{otherwise}.
      \end{cases}
  \end{equation}
Throughout the proof, $0.\bb \in \R$ is a real number representation for any sequence $\bb \in \{0,1\}^\infty$.
\paragraph{End-to-End pseudo-dimension is infinite.} Fix the set $S=\{\bx_n: n \in \naturals_{+}\}$ as defined previously. We will fix the thresholds to $z_1=z_2=\dots=1/2$. We know that for any bit pattern $\bb \in \{0,1\}^{\naturals_+}$, we have
$$\sign(f_\bb^{\sete{(2)}}(\bx_n)-z_n)=\sign(b_n-1/2)=\begin{cases}
    +1, \text{if } b_n=1;\\
    -1, \text{if } b_n=0.
\end{cases}.$$
This means we can realize all sign patterns on the points in $S$, given by $\{-1,+1\}^{\naturals_{+}}$, using hypothesis from $\ete{\cF}{2}$ and the thresholds $z_1=z_2=\dots=1/2$. This establishes that $\Pdim(\ete{\cF}{2})=\infty$.
\paragraph{Pseudo-dimension after one step is small.} It is easy to see that $\Pdim(\cF)\geq 1$. Simply choose $\bx_1$ and $z_1=1.05$. Then for $\bb^{(1)}=0^\infty=(000\dots)$ and $\bb^{(2)}=10^\infty=(1000\dots)$,
$$ \sign(f_{\bb^{(1)}}(\bx_1)-z_1)=\sign(1+0.0-1.05)=-1 \; \text{ but } \sign(f_{\bb^{(2)}}(\bx_1)-z_1)=\sign(1+0.1-1.05)=+1.$$
The main challenge is to show that $\Pdim(\cF) \leq 1$. Suppose for the purpose of contradiction, $\Pdim(\cF) \geq 2$ and there exists points $\{\Tilde{\bx}_1, \Tilde{\bx}_2\}$ and thresholds $z_1, z_2\in \R$ such that 
$$\left| \left\{ (\sign(f_{\bb}(\Tilde{\bx}_1)-z_1),\sign(f_{\bb}(\Tilde{\bx}_2)-z_2)) : f_{\bb} \in \cF \right\} \right| =4.$$
We first note that this enforces that the points $\{\Tilde{\bx}_1, \Tilde{\bx}_2\}$ must be such that they match the description of either of the first two cases in \eqref{eq:fbb-description-real-tokens}. Otherwise on the point that does not match the description, each predictor $f_{\bb} \in \cF$ outputs 0, and irrespective of what $(z_1,z_2)$ is chosen, we can never realize all four sign patterns. We now consider the following two cases and show a contradiction in each.
\begin{enumerate}
    \item At least one point from $\{\Tilde{\bx_1}, \Tilde{\bx}_2\}$ matches description of the second case: W.l.o.g. assume that $\Tilde{\bx}_1$ is that point, i.e. $\Tilde{\bx}_1= [0^* \Tilde{\bx}_{n_1}(n_1+0.\Tilde{\bb})]$ for some $n_1 \in \naturals_{+}$ and $\Tilde{\bb} \in \{0,1\}^\infty$. This implies that 
    $$ f_{\bb}(\Tilde{\bx}_1) = \begin{cases}
        \Tilde{b}_{n_1}, & \text{ if } \bb=\Tilde{\bb},\\
        0, & \text{ if } \bb \neq \Tilde{\bb}.
    \end{cases}$$
So there are only two possible real outputs on $\Tilde{\bx}_1$, therefore, in order to shatter the real threshold $0<z_1 \leq n_1+0.\Tilde{\bb}\,$. However, there is only one predictor $f_{\Tilde{\bb}} \in \cF$ whose output $f_{\Tilde{\bb}}(\Tilde{\bx}_1) \geq z_1$. But then this completely also determines the output on $\Tilde{\bx}_2$, and irrespective of what $z_2$ is chosen, we cannot realize at least one of the sign patterns from $\{(+1,-1),(+1,+1)\}$, arriving at a contradiction.
\item Both $\{\Tilde{\bx_1}, \Tilde{\bx}_2\}$ matches the description of the first case: In this case, it is important to note that we must have
 $$ \Tilde{\bx}_1=0^*\bx_{n_1} \quad \Tilde{\bx}_2=0^*\bx_{n_2}, \text{for }\, n_1 \neq n_2 \in \naturals_{+}.$$
Otherwise both points are identical for the purpose of prediction for any $f_{\bb} \in \cF$, and we cannot realize all four sign patterns anyway. Finally, in order to shatter these two points, there exist thresholds $z_1,z_2 \in \R$ such that 
  $$\{(\sign(f_{\bb}(x_{n_1})-z_1),\sign(f_\bb(x_{n_2})-z_2)) : f_\bb \in \cF \}=\{(s_1,s_2): s_1 ,s_2 \in \{-1,+1\}\}.$$
Using the output of $f_{\bb}$ from \eqref{eq:fbb-description-real-tokens}, this is equivalent to 
$$\{(\sign(n_1-z_1+0.\bb),\sign(n_2-z_2+0.\bb)) : \bb \in \{0,1\}^\infty \}=\{(s_1,s_2): s_1 ,s_2 \in \{-1,+1\}\}.$$
But we know that $(n_1-z_1), (n_2-z_2) \in \R$ and w.l.o.g. assume $(n_1-z_1) \geq (n_2-z_2)$. This immediately implies there is no real number $0.\bb \in \R$ such that 
$$\sign(n_1-z_1+0.\bb)=-1 \text{ but } \sign(n_2-z_2+0.\bb)=+1,$$
concluding the proof. 
\end{enumerate}
\end{proof}

\end{document}